\documentclass[11pt]{article}

\usepackage{subfigure}
\usepackage{float}
\usepackage[all]{xy}

\usepackage{amsmath,amssymb,amsthm}
\usepackage{latexsym}
\usepackage[dvipdfmx]{graphicx}
\usepackage[ruled]{algorithm2e}
\usepackage{algorithmic} 
\usepackage{authblk}
\usepackage{color}
\usepackage{relsize}
\usepackage{url}
\usepackage[colorlinks,linkcolor=blue,anchorcolor=green,citecolor=green]{hyperref}
\usepackage{caption} 
\usepackage{tikz}
\usepackage{yhmath}
\usepackage{mathrsfs,booktabs}

\usepackage[T1]{fontenc}
\usepackage[utf8]{inputenc}
\usepackage[font=small,labelfont=bf,tableposition=top]{caption}
\usepackage{booktabs}
\usepackage{threeparttable} 
\usepackage{multirow}

\makeatletter

\newcommand{\Rmnum}[1]{\expandafter\@slowromancap\romannumeral #1@}
\makeatother

\newtheorem{theorem}{Theorem}

\newtheorem{lemma}{Lemma}

\newtheorem{proposition}{Proposition}
\newtheorem{definition}{Definition}

\newtheorem{assumption}{Assumption}
\usepackage{color}

\usepackage{tikz,pgfplots}
\pgfplotsset{compat=newest}
\pgfplotsset{plot coordinates/math parser=false,trim axis left}
\newlength\figureheight
\newlength\figurewidth

\newcommand{\eins}{\boldsymbol{1}}

\newcommand{\argmax}{\operatornamewithlimits{arg \, max}}

\pgfplotsset{
standard/.style={
	axis x line=middle,
	axis y line=middle,
	enlarge x limits=0.15,
	enlarge y limits=0.15,
	every axis x label/.style={at={(current axis.right of origin)},anchor=north west},
	every axis y label/.style={at={(current axis.above origin)},anchor=north east}
}
}

\author[]{Hanyuan Hang}

\date{\today}

\affil[]{Department of Applied Mathematics \\ 
	University of Twente, The Netherlands \\
	{\tt h.hang@utwente.nl}
}

\begin{document}

\title{Bagged $k$-Distance for Mode-Based Clustering 
	Using the Probability of Localized Level Sets}



\allowdisplaybreaks

\maketitle


\begin{abstract}In this paper, we propose an ensemble learning algorithm named \textit{bagged $k$-distance for mode-based clustering} (\textit{BDMBC}) by putting forward a new measurement called the \textit{probability of localized level sets} (\textit{PLLS}), which enables us to find all clusters for varying densities with a global threshold.
	On the theoretical side, we show that with a properly chosen number of nearest neighbors $k_D$ in the bagged $k$-distance, the sub-sample size $s$, the bagging rounds $B$, and the number of nearest neighbors $k_L$ for the localized level sets, BDMBC can achieve optimal convergence rates for mode estimation. 
	It turns out that with a relatively small $B$, the sub-sample size $s$ can be much smaller than the number of training data $n$ at each bagging round, and the number of nearest neighbors $k_D$ can be reduced simultaneously.
	Moreover, we establish optimal convergence results for the level set estimation of the PLLS in terms of Hausdorff distance, which reveals that BDMBC can find localized level sets for varying densities and thus enjoys local adaptivity.
	On the practical side, we conduct numerical experiments to empirically verify the effectiveness of BDMBC for mode estimation and level set estimation, which demonstrates the promising accuracy and efficiency of our proposed algorithm. 
\end{abstract}


\section{Introduction}

In the field of \textit{density-based clustering}, the common assumption that all clusters have similar levels of densities is shared by many algorithms. 
In detail, those algorithms employ a global threshold for densities to define the high-density regions and categorize them as clusters.
Due to the algorithmic simplicity, such paradigm, also named as \textit{single-level} density-based clustering, attracts lots of attention in the early stage of clustering researches \cite{ester1996density,rehioui2016denclue,hinneburg2007denclue,idrissi2015improvement,jang2019dbscanpp}.
However, with the rapid development of information technology, the assumption is hard to hold as the number of clusters and the size of data keeps growing.
It has also been proved in experiments that the well-performed single-level clustering algorithms are incompetent in encountering datasets that have varying densities for different clusters \cite{zhu2021cdf,mcinnes2017hdbscan}.
Consequently, a more general setting for density-based clustering called \textit{multi-level density} clustering comes into vogue \cite{zhu2021cdf,mistry2021aedbscan,chazal2013persistence} and is applied in various subjects including computer vision \cite{pla2019finding,cesario2021towards,hu2022gmc_fm}, medicine and biometrics \cite{liu2015segmentation,ranjbarzadeh2020automated}, etc.

To solve the multi-level density clustering problem, a primary idea is to expand the solutions proposed in single-level clustering problems. 
Some researchers therefore hold the opinion that increasing the number of thresholds can help seek clusters with different densities, and propose the paradigm called \textit{hierarchical density-based clustering}.
The term \textit{hierarchical} means that the algorithm follows either an agglomerative (bottom-up) or a divisive (top-down) order to move the threshold, estimates the clusters with each threshold, and finally grows a clustering tree based on the clustering results. And by carefully selecting the nodes in the clustering tree, the hierarchical methods can obtain promising results in multi-level density situations \cite{mcinnes2017hdbscan,malzer2020hybrid,amini2016mudi}. 
For example, \cite{mcinnes2017hdbscan} takes the advantage of DBSCAN and proposes an automatic framework called HDBSCAN to decide which thresholds are better according to some pre-determined informational metric; 
In addition, \cite{malzer2020hybrid} also studied how to choose the optimal clustering results from a clustering tree.
Nevertheless, the hierarchical methods are criticized for the heavy computational cost of growing the clustering tree. 
And it is still an unsolved problem for automatically determining the clusters from a cluster tree still remains.

Therefore, to improve the computation efficiency and directly obtain the clustering result, another part of the research aims at finding a suitable transforming for the current density measure to balance the levels of densities for all clusters into a similar level \cite{cheng1995mean,jang2021meanshift++,zhu2021cdf,zhu2016density,mitra2011kddclus}. 
To be specific, such algorithms care little about the absolute density value of samples. Instead, they attach more importance to the relative information of samples in a local area.
For example, \cite{cheng1995mean} proposes the mean shift method to find the density hill of clusters by iteratively searching the center of mass from several randomly chosen initial points.
And in \cite{zhu2021cdf}, the estimated probability density function is transformed to a new measure called density ratio to help balance the density measure.
Since they result in seeking the \textit{bump} or \textit{hill} in the distribution, they are also called mode-based clustering algorithms.

Although mode-based methods largely increase the computational efficiency of multi-level density clustering problems, they still suffer from two inevitable shortcomings. 
Firstly, many mode-based algorithms require the estimation of the probability density function, e.g.~\cite{zhu2021cdf}. 
Nevertheless, density estimation problems suffer from the curse of dimensionality, which means estimating a satisfactory density function will be much harder and require more training samples when the dimension of the input variables is high.
Hence, it is hard to perform the mode-based clustering algorithm on high-dimensional datasets. 
Secondly, the computational efficiency of the mode-based algorithms may not be as satisfactory as expected. On the one hand, mode-based algorithms require much training time when the sample size is large. On the other hand, the procedure of searching an optimal combination of parameters can be even more tiresome.

Under such background, in this paper, we propose an ensemble learning algorithm called
\textit{bagged $k$-distance for mode-based clustering} (\textit{BDMBC})  to solve the multi-level density clustering problems. 
To be specific, we first introduce a new measurement called \textit{probability of localized level sets} (\textit{PLLS}) to deal with the multi-level density  problems. 
PLLS represents the local rank of the density which makes it possible to employ a global threshold to recognize the multi-level density clusters.
Secondly, to resist the curse of dimensionality in density estimation, we introduce the \textit{$k$-distance} as the density function which is then plugged into the localized level set estimation. As a distance-based measure, $k$-distance has strong resistance to the curse of dimensionality, and hence enables BDMBC to deal with high-dimensional data sets.
Last but not least, we further employ the bagging technique to enhance the computational efficiency in calculating the $k$-distance. In particular, when dealing with large-scale datasets, the bagging technique can accelerate the algorithm with a small sampling ratio and thus uses a much smaller training dataset in each bagging iteration. Since the size of the training dataset in each iteration is largely decreased by sub-sampling, the searching grid for  sample-size-based hyper-parameters can also be simplified, preventing the practitioners from tedious hyper-parameter tuning.

The theoretical and experimental contributions of this paper are summarized as follows:

\textit{(i)}  
From the theoretical perspective, we first conduct a learning theory analysis of the bagged $k$-distance by introducing the \textit{hypothetical density estimation}.
Under the H\"{o}lder smoothness of the density function, with properly chosen $k$, we establish optimal convergence rates of the hypothetical density estimation in terms of the $L_\infty$-norm. It is worth pointing out that our finite sample results demonstrate the explicit relationship among bagging rounds $B$, the number of nearest neighbors $k$, and the sub-sample size $s$. 

Then we propose a novel mode estimation built from the probability of a localized level set. Based on the convergence rates of the hypothetical density estimation, we show that under mild assumptions on modes, we obtain optimal convergence rates for mode estimation with properly chosen parameters. We show that the bagging technique helps to reduce the subsample size and the number of neighbors simultaneously for mode estimation and thus increases the computational efficiency.

Moreover, under mild assumptions on the density function, we establish convergence results of level set estimation for the probability of localized level set in terms of Hausdorff distance. Compared to previous works on level set estimation in clustering that focus on a single threshold, our results reveal level sets for varying densities. This reveals the local adaptivity of our BDMBC in multi-level density clustering.

\textit{(ii)}  
From the experimental perspective, we conduct numerical experiments to illustrate the properties of our proposed BDMBC.
Firstly, we verify our theoretical results about mode estimation by conducting the experiment of mode estimation on synthetic datasets. We demonstrate that our BDMBC can detect all modes successfully and thus can cluster all mode-based clusters.
Secondly, we verify our theoretical results about level-set estimation by conducting numerical comparisons with other competing methods. We show the promising accuracy and efficiency of our proposed algorithm compared with other density-based, cluster-tree-based, and mode-based methods.
Thirdly, we conduct parameter analysis on our proposed BDMBC, and empirically demonstrate that with a relatively small subsample ratio, bagging can significantly narrow the search-grid of parameters. Moreover, we compare the bagging and non-bagging version of the BDMBC on large-scale synthetic datasets and verify that bagging can largely shorten the computation time without sacrificing accuracy.

The remainder of this paper is organized as follows.
Section \ref{sec::Methodology} is a warm-up section for the introduction of some notations and the new measurement, the \textit{probability of localized level sets} (\textit{PLLS}). Then we propose \textit{the bagged $k$-distance for mode-based clustering} (\textit{BDMBC}) in Section \ref{sec::Methodology}. 
In Section \ref{sec::results}, we first present our main results on the convergence rates for mode estimation and level set estimation. Then we provide some comments and discussions concerning the main results in this section. 
In Section \ref{sec::Error}, we conduct the error analysis for the bagged $k$-distance and calculate its computational complexity. 
Section \ref{sec::Experiment} presents experimental results on both real and synthetic data. We also conduct scalability experiments to show the computational efficiency of our algorithm in this section. In Section \ref{sec::proofs}, we demonstrate the details of proofs. 
Finally, we summarize our work in Section \ref{sec::conclusions}.

\section{Methodology}\label{sec::Methodology}

In this section, we briefly recall some necessary notations and algorithms as preliminaries in Section \ref{sec::preliminaries}.
Then, to avoid the drawback of classical density-based clustering methods, we propose a new measurement, the \textit{probability of localized level sets} (\textit{PLLS}) in Section \ref{sec::densitycomparison}, introduce the bagged $k$-distance in Section \ref{sec::bagdistance}, and construct the corresponding density-based clustering algorithm called \textit{bagged $k$-distance for mode-based clustering} (\textit{BDMBC}) in Section \ref{sec::dcbc}.

\subsection{Preliminaries} \label{sec::preliminaries}

First, we introduce some basic notations that will be frequently used in this paper. 
We use the notation $a \vee b := \max \{ a, b \}$ and $a \wedge b := \min \{ a, b \}$. 
For any $x \in \mathbb{R}$, let $\lfloor x \rfloor$ denote the largest integer less than or equal to $x$ and $\lceil x \rceil$ the smallest integer greater than or equal to $x$. 
Recall that for $1 \leq p < \infty$, the $\ell_p$-norm is defined as $\|x\|_p := ( x_1^p + \cdots + x_d^p )^{1/p}$, and the $\ell_{\infty}$-norm is defined as $\|x\|_{\infty} := \max_{i = 1, \ldots, d} |x_i|$.  
Let $(\Omega, \mathcal{A}, \mu)$ be a probability space.
We denote $L_p(\mu)$ as  the space of (equivalence classes of) measurable functions $g : \Omega \to \mathbb{R}$ with finite $L_p$-norm $\|g	\|_p$. 
For any $x \in \mathbb{R}^d$ and $r > 0$, denote $B(x, r) := \{ x' \in \mathbb{R}^d : \|x' - x\|_2 \leq r \}$ as the closed ball centered at $x$ with radius $r$. 
For a set $A \subset \mathbb{R}^d$, the cardinality of $A$ is denoted by $\#(A)$ and the indicator function on $A$ is denoted by $\eins_A$ or $\eins \{ A \}$.

In the sequel, the notations $a_n \lesssim b_n$ and $a_n=\mathcal{O}(b_n)$ denote that there exists some positive constant $c \in (0, 1)$, such that $a_n \leq c b_n$ and $a_n \gtrsim b_n$ denotes that there exists some positive constant $c \in (0, 1)$, such that $a_n \geq c^{-1} b_n$. 
Moreover, the notation $a_n \asymp b_n$ means that there hold $a_n\lesssim b_n$ and $b_n\lesssim a_n$ simultaneously.
Let $P$ be a probability distribution on $\mathbb{R}^d$ with the underlying density $f$ which has a compact support $\mathcal{X} \subset [-R, R]^d$ for some $R > 0$. 
Suppose that the data $D_n = (X_1, \ldots, X_n) \in \mathcal{X}^n$ is drawn from $P$ in an i.i.d. fashion.
With a slight abuse of notation, in this paper, $c, c', C$ will be used interchangeably for positive constants while their values may vary across different lemmas, propositions, theorems, and corollaries.

\subsection{Probability of Localized Level Sets} \label{sec::densitycomparison}

One of the main drawbacks of density-based clustering based on density estimation is that it can not find all clusters with varying densities using a global threshold, see, e.g., \cite{chacon2015population,zhu2016density,chacon2020modal}.
Here we give a simple univariate example of this phenomenon in Figure \ref{fig::comparative::truedensity}.
For the univariate trimodal density, there are three different clusters that are visually identifiable, yet none of the level sets of the density has three connected components.
In fact, if the level is chosen too low, the two clusters of high densities will be merged into a single cluster. 
If the density level is chosen too high, the other cluster exhibiting a lower density will be lost.
Clearly, in such case, the clusters derived from a single density level cannot completely describe the inherent clustering structure of the data set.

\begin{figure}[!h]
	\centering
	\centerline{\includegraphics[height=0.32\columnwidth]{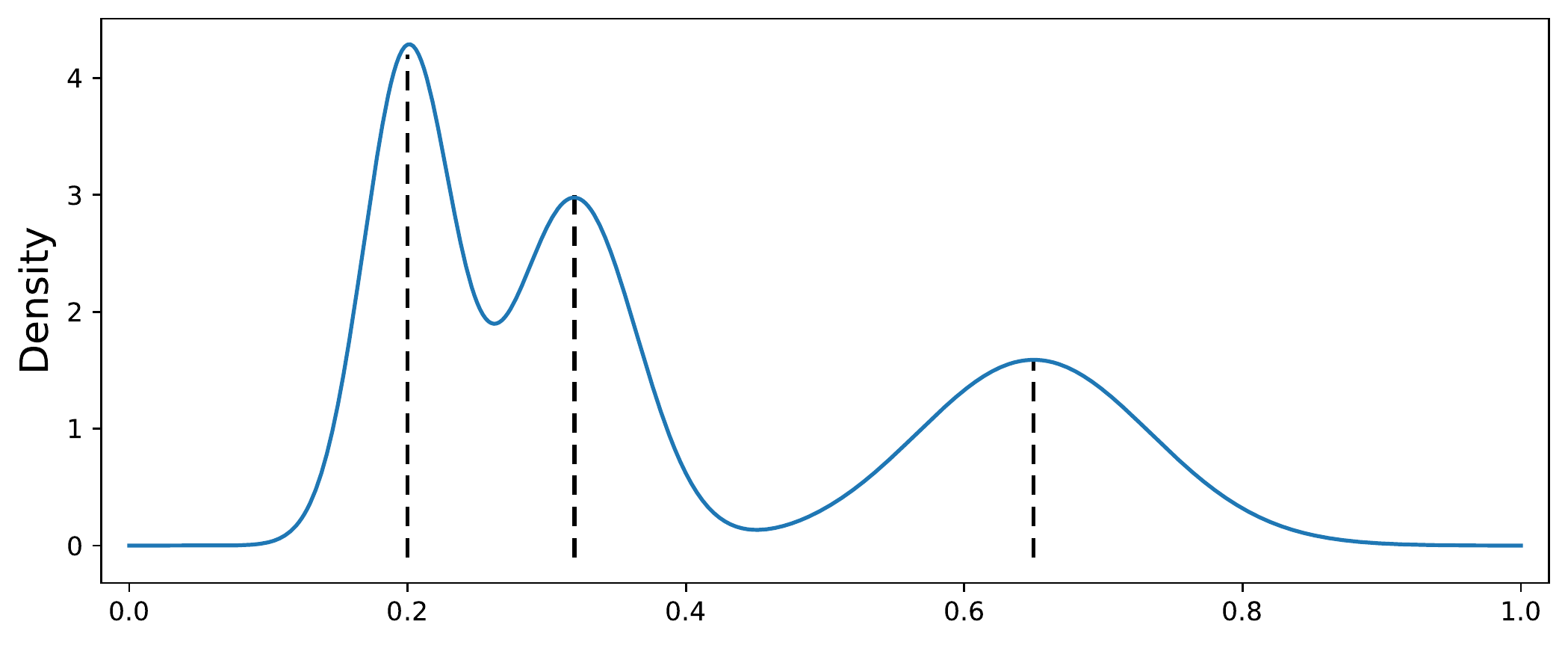}}
	\caption{Univariate trimodal density for which it is not possible to capture its whole cluster structure using a global threshold.}
	\label{fig::comparative::truedensity}
	\vskip -0.05in
\end{figure}

To deal with this issue, we propose a local measurement named the \textit{probability of localized level sets} (\textit{PLLS}) to implement the density-based clustering.

\begin{definition}[Probability of Localized Level Sets] \label{def::plls}
	Let $x \in \mathcal{X}$ and $\eta(x) > 0$ be the local radius parameter.
	Given the true density function $f : \mathcal{X} \to \mathbb{R}$, the \textit{probability of localized level sets} (\textit{PLLS}) is defined by
	\begin{align}\label{equ::frx}
		p_{\eta}(x)
		= P \bigl( f(y) \leq f(x) | y \in B(x,\eta(x)) \bigr)
		= \frac{P \bigl( f(y) \leq f(x), y \in B(x,\eta(x)) \bigr)}{P(y \in B(x,\eta(x)))}.
	\end{align}
\end{definition}

Note that the PLLS is the conditional probability of the event where the density of the instance is larger than that of its neighborhood.
To explain the advantages of the PLLS over the original probability density function for clustering, we point out two critical observations from \eqref{equ::frx}. 
On the one hand, if $x$ is a mode of $f$, then $f(y) \leq f(x)$ for all $y \in B(x, \eta(x))$. 
This yields that $p_{\eta}(x) = 1$. 
On the other hand, if $f(x)$ is a local minimum of the density, i.e., if $f(y) \geq f(x)$ for all $y \in B(x, \eta(x))$, then we have $p_{\eta}(x) = 0$.
Therefore, the PLLS figures out the relative positions to the modes of the density $f$ unlike the probability density function.
As a result, we can deal with the variation in density across different clusters and thus allow for a single threshold to identify all the modes and the corresponding clusters simultaneously.

\subsection{Bagged $k$-Distance}\label{sec::bagdistance}

In this section, we introduce the bagged $k$-distance, which represents the density implicitly, for the construction of mode-based clustering.
For any $x \in \mathbb{R}^d$ and any subset $D \subset D_n$, we denote $X_{(k)}(x):=X_{(k)}(x;D)$ as the $k$-th nearest neighbor of $x$ in $D$. Then we denote $R_k(x;D)$ as the distance between $x$ and $X_{(k)}(x;D)$, termed as the $k$-distance of $x$ in $D$. Specifically, we let $R_k(x) := R_k(x;D_n)$.

We first recall $k$-nearest neighbor ($k$-NN) for density estimation. 
To be specific, denote $\mu(B(x,r))$ as the area (described in the Lebesgue measure) of the ball $B(x,r)$. Then the $k$-NN density estimator \cite[Definition 3.1]{biau2015lectures} is defined by
\begin{align}\label{eq::fD}
	f_{k}(x) = \frac{k/n}{\mu(B(x,R_k(x)))} = \frac{k/n}{V_d R_k(x)^d},
\end{align}
where $V_d := \pi^{d/2} / \Gamma(d/2+1)$ is the volume of the unit ball.

However, in practice, a major problem is the numerical issues when computing $k$-NN density estimation for high-dimensional data where samples in a finite dataset can distribute quite sparsely. As a consequence, the target density can be extremely small in areas of the input space. In this case, $R_k(x)^d$ in \eqref{eq::fD} can be extremely large when $d$ is big, leading to arithmetic overflow in the process of computing. Therefore, density estimation is problematic to be derived directly for high-dimensional data.
On the other hand, for large-scale datasets, the computational burden of searching for $k$-nearest neighbors can be heavy.
To deal with these two problems, in this work, we adopt the bagging technique to reduce the number of nearest neighbors to search, and investigate a bagged variant of $k$-distance, called \textit{bagged $k$-distance}. To be specific, let $\{ D_b\}_{b=1}^B$ be a series of subsets uniformly subsampled from $D_n$ without replacement.
We define the bagged $k$-distance as 
\begin{align}\label{eq::bd}
	R_k^B(x) =\frac{1}{B}\sum_{b=1}^BR_k(x; D_b).
\end{align}

For the following theoretical analysis, we show that the bagged $k$-distance can be used to construct a \textit{hypothetical density estimator},
\begin{align}\label{equ::fbk}
	f_B(x) := \frac{\bigl( \sum_{i=1}^n p_i(i/n)^{1/d} \bigr)^d}{V_d R_k^B(x)^d}
\end{align}
with the weights 
\begin{align}
	p_i
	& := P(X_i \text{ is the } k\text{-th nearest neighbor of } x \text{ in } D_b)
	\nonumber\\
	& =
	\begin{cases}
		\binom{i-1}{k-1}\binom{n-i}{s-k}/\binom{n}{s}, & \text{ if } k\leq i\leq n-s+k\\
		0, & \text{ if } i\leq k \text{ or } i>n-s+k,
	\end{cases}
	\label{equ::def}
\end{align}
where $s$ denotes the subsample size of bagging.

The terminology \textit{hypothetical} derives from the observation that it is difficult to compute the $p_i$'s in practice due to the complicated calculations of combinations for the large sample size $n$. That is, rather than for practical use, the hypothetical density estimator is only for understanding the bagged $k$-distance and thus the theoretical analysis. 

The above definition acts as a bridge between bagged $k$-distance and hypothetical density estimator \eqref{equ::fbk}, where $f_B(x)$ is proportional to $R_k^B(x)^{-d}$.
We show that $f_B(x)$ has the same relative magnitude as $R_k^B(x)^{-1}$, that is, for a given $x$, larger bagged $k$-distance $R_k^B(x)$ indicates smaller hypothetical density estimation $f_B(x)$.
We delay the discussions that $f_B(x)$ is indeed a desired estimator of the underlying density function $f$ to Proposition \ref{thm::main2} in Section \ref{sec::RatesPseudoDE}.

\begin{figure}[!h]
	\centering
	\vskip -0.05in
	\begin{minipage}{0.43\columnwidth}
		\centering
		\centerline{\includegraphics[height=0.75\columnwidth]{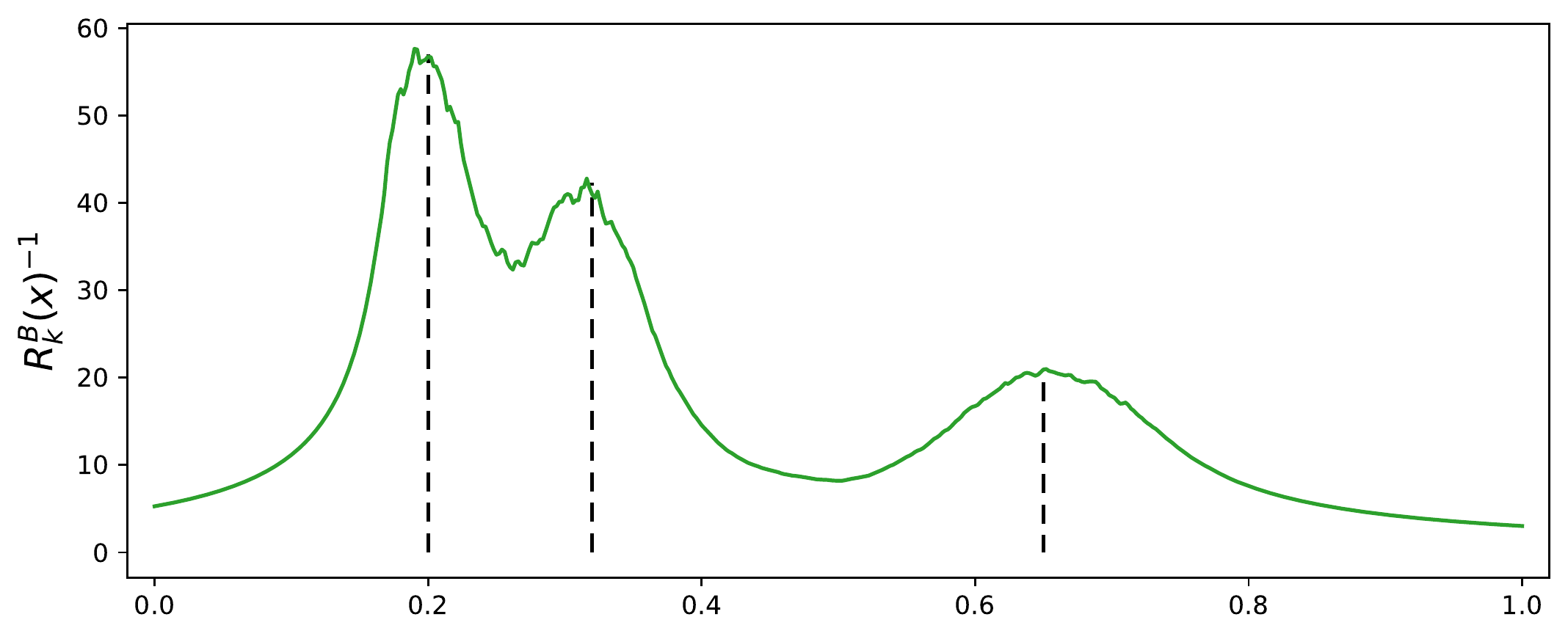}}
	\end{minipage}
	\vskip 0.0in
	\caption{Examples of the inverse of bagged $k$-distance $R_k^B(x)^{-1}$ with $B=25$, $s=0.9n$, $k_D=300$. Three vertical dash lines present modes of these three Gaussian distributions, showing that our method can find out all modes with varying densities at a time and make clustering much easier.}
	\label{fig::comparative::bknndensity}
	\vskip -0.05in
\end{figure}

In Figure \ref{fig::comparative::bknndensity}, we empirically illustrate the relationship between the inverse of bagged $k$-distance $R_k^B(x)^{-1}$ and the true density $f(x)$.
Here, we use a one-dimensional synthetic dataset named {\tt 3Mix} containing three Gaussian distributions $\mathcal{N}(0.20, 0.001)$, $\mathcal{N}(0.32, 0.002)$, and $\mathcal{N}(0.65, 0.007)$ with equal mixture component weights.
Figure \ref{fig::comparative::bknndensity} illustrates a one-dimensional example of the bagged $k$-distance.
The true density which is a mixture of three Gaussian distributions is provided in Figure \ref{fig::comparative::truedensity}.
With $d=1$, we show that the inverse of bagged $k$-distance $R_k^B(x)^{-1}$ in Figure \ref{fig::comparative::bknndensity} is proportional to the underlying density $f(x)$ in Figure \ref{fig::comparative::truedensity}.

\subsection{Bagged $k$-Distance for Mode-Based Clustering} \label{sec::dcbc}

The reformulation in \eqref{equ::frx} inspires us to empirically estimate both the numerator and denominator term of $p_{\eta}(x)$ respectively. 
More specifically, let $k_L\in \mathbb{N}$ be the number of nearest neighbors for localized level sets and $\eta(x):=R_{k_L}(x)$, then $p_{\eta}(x)$ can be estimated by
\begin{align}\label{equ::empiricalpkl}
	\widehat{p}_{k_L}(x)
	= \frac{\sum_{i=1}^n \eins \{ X_i \in B(x, R_{k_L}(x)), f(X_i) \leq f(x) \}}{\sum_{i=1}^n \eins \{ X_i \in B(x, R_{k_L}(x))\}}
	= \frac{1}{k_L} \sum_{i=1}^{k_L} \eins \{ f(X_{(i)}(x)) \leq f(x) \}.
\end{align}

To derive a computationally efficient estimator for PLLS, we use the bagged $k$-distance in Section \ref{sec::bagdistance} to define the \textit{empirical PLLS} by
\begin{align}\label{getah}
	\widehat{p}_{k_L}^B(x)
	:= \frac{1}{k_L} \sum_{i=1}^{k_L} \eins \bigl\{ R_{k_D}^B(X_{(i)}(x)) \geq R_{k_D}^B(x) \bigr\}.
\end{align}
In fact, $\widehat{p}_{k_L}^B(x)$ denotes the proportion of instances in the $k_L$ nearest neighbors whose density estimates are smaller than that of $x$.

With respect to the bagged $k$-distance plotted in Figure \ref{fig::comparative::bknndensity}, here we plot the empirical PLLS in Figure \ref{fig::comparative::bknnlocaldensity} which shows that PLLS pushes all density peaks towards $1$ and forces all density valleys towards $0$. This enlarges the difference between peaks and valleys, and therefore it is easier to use a global threshold to separate high-density regions and low-density regions. 
Moreover, note that three vertical dash lines in Figure \ref{fig::comparative::bknnlocaldensity} present modes of these three Gaussian distributions, the figures show that the PLLS based on bagged $k$-distance can find out all modes with varied densities at a time. 
Specifically, in Figure \ref{fig::comparative::truedensity}, there are three density peaks (modes) in $x=0.2$, $x=0.32$ and $x=0.65$, respectively. 
Two density valleys are located nearby $x=0.25$ and $x=0.50$. 
By introducing the PLLS, we see from Figure \ref{fig::comparative::bknnlocaldensity} that, the values of density peaks are close to $1$, and the values of density valleys are close to $0$. 
The score difference between the density peak in $x=0.32$ and the density valley in $x=0.25$ is significantly enlarged. 
The score difference between the density peak in $x=0.65$ and the density valley in $x=0.50$ is also significantly enlarged. 

\begin{figure}[!h]
	\centering
	\vskip -0.1in
	\begin{minipage}{0.428\columnwidth}
		\centering
		\centerline{\includegraphics[height=0.75\linewidth]{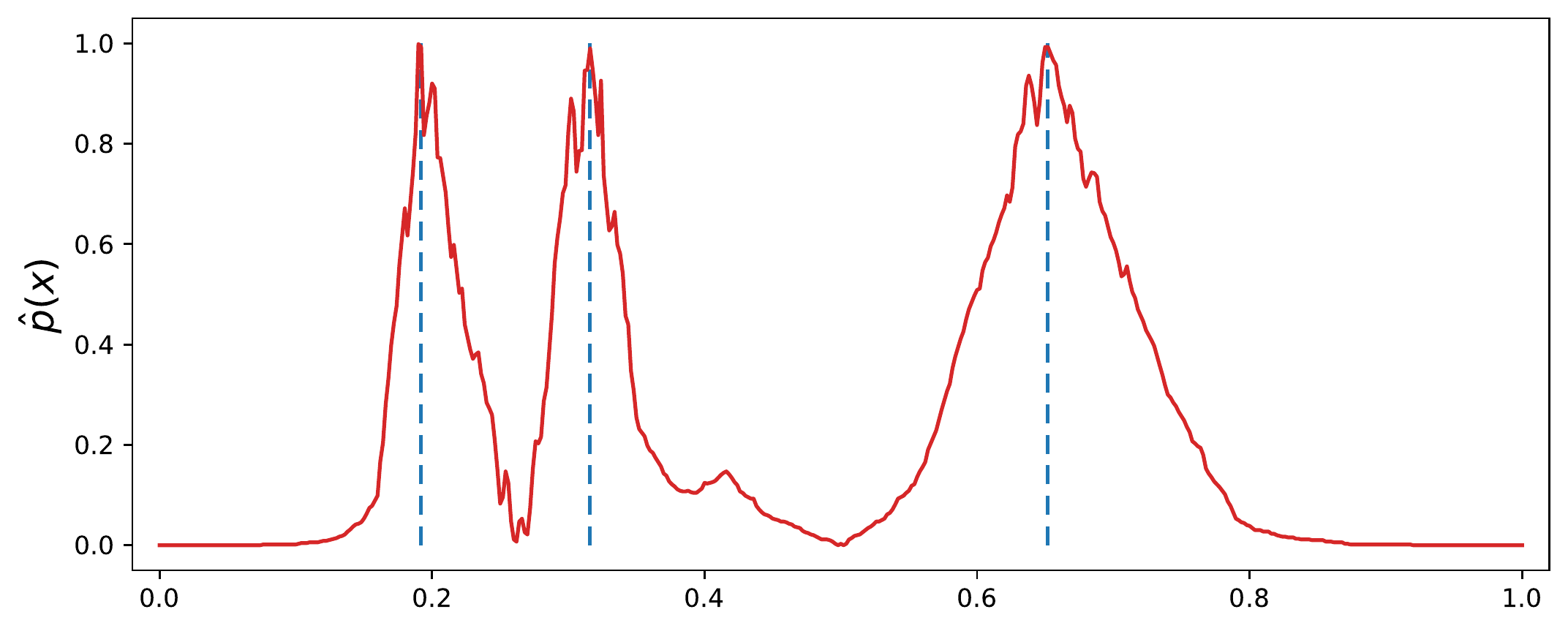}}
		\label{fig::comparative::bknnlocaldensity1}
	\end{minipage}
	\vskip 0.0in
	\caption{Examples of empirical PLLS with $B=25$, $s=0.9n$, $k_D=300$, and $k_L=750$. Three vertical dash lines present modes of these three Gaussian distributions, showing that our method can find out all modes with varied densities at a time and make density-based clustering much easier.}
	\label{fig::comparative::bknnlocaldensity}
	\vskip -0.1in
\end{figure} 

Now we can use a global threshold to recover the clusters with the following density-based clustering algorithm named \textit{bagged $k$-distance for mode-based clustering} (\textit{BDMBC}), which is summarized in Algorithm \ref{alg::BDMBC}.
BDMBC uses the \textit{probability of localized level sets} (\textit{PLLS}) to construct the graph on sample points in the upper level set and find connected components as clusters. 
By this approach, we can find the regions of \textit{locally} high density based on the upper level set of $\widehat{D}_{B}(\lambda)$ defined by \eqref{equ::detah}.
Then we recover the clusters according to the $k$-NN graph which utilizes the local density information to connect points.
We mention that in Algorithm \ref{alg::BDMBC}, we only consider the instances with $\widehat{p}_{k_L}^B(x)\geq \lambda$ since these instance are regarded as more important following from the statistically-principled approach in \cite{hartigan1975clustering}. 
Those instances not in the graph $G_B(\lambda)$ can be assigned to their closest clusters, see e.g.~\cite{jang2019dbscanpp}.

\begin{algorithm}[!h]
	\caption{Bagged $k$-Distance for Mode-Based Clustering (BDMBC)}
	\label{alg::BDMBC}
	\KwIn{
		A dataset $D := D_n := \{ X_1, \ldots, X_n \}$; \\
		\quad\quad\quad\,\,\,\, Bagging size $B$ and subsample size $s$;
		\\
		\quad\quad\quad\,\,\,\, Nearest neighbor $k_D$ for hypothetical density estimation;
		\\
		\quad\quad\quad\,\,\,\, Nearest neighbor $k_L$ for localized level set;
		\\
		\quad\quad\quad\,\,\,\, Nearest neighbor $k_G$ for graph.
		\\
	} 
	1. Subsample $s$ points as $\{ D_b \}_{b=1}^B$ from $D_n$ without replacement. \\
	2. Compute the bagged $k$-distance $R_{k_D}^B$ by \eqref{eq::bd} based on $(D_b)_{b=1}^B$. \\
	3. Compute the empirical PLLS $\widehat{p}_{k_L}^B(x)$ by \eqref{getah} for each $X_i$, $i=1,\ldots,n$. \\
	4. Construct $k_G$-nearest neighbor graph $G$ of all training samples $D$. \\
	5. Construct the subgraph of $G_B(\lambda)$ retaining the core-samples 
	\begin{align}\label{equ::detah}
		\widehat{D}_B(\lambda)=\{X_i\in D:\widehat{p}_{k_L}^B(X_i)\geq \lambda\}
	\end{align}
	and the mode set
	\begin{align}\label{equ::modesestimator}
		\widehat{\mathcal{M}}=\{X_i \in D:\widehat{p}_{k_L}^B(X_i)=1\}.
	\end{align}
	6. Compute the cluster estimators $
	{\mathcal{C}}_{B}(\lambda)$ that is the connected components of $G_B(\lambda)$.
	\\
	\KwOut{The proper cluster estimator $
		\mathcal{C}_{B}(\lambda)$.
	}
\end{algorithm}

From the definition of the empirical PLLS, we mention that it is critical to choose a proper number of nearest neighbors $k_L$ for localized level sets. 
On the one hand, if $k_L$ is too large, the neighborhood will contain more than one mode and thus can not reflect the local behavior of the densities.
On the other hand, if $k_L$ is too small, there will be too few instances in the neighborhood, which leads to an unreliable estimator for clustering.

When we replace the probability density function to the empirical PLLS $\widehat{p}_{k_L}^B(x)$ in \eqref{getah}, the set of modes can be naturally estimated by \eqref{equ::modesestimator}. 
From \eqref{equ::modesestimator}, we see that the mode set $\widehat{\mathcal{M}}$ picks the point with minimal bagged $k_D$-distance out of the $k_L$ nearest neighbors.
In this case, the difference in densities between mode estimations at dense and sparse regions can be reduced to zero with an appropriate $k_L$. 
We mention that our mode estimation is  different from gradient-based mode-seeking algorithms in the literature.
Examples of such procedures include the mean shift algorithm \cite{fukunaga1975estimation,comaniciud2002arobust}, the modal EM \cite{li2007nonparametric}, and the quick shift algorithm \cite{jiang2017consistency}.

Moreover, we highlight the role of mode estimation in the density-based clustering algorithm. As pointed out in \cite{hartigan1975clustering}, clusters can be identified as modes of the probability density function $f$ by the statistically-principled approach.

\section{Theoretical Results}\label{sec::results}

In this section, we establish theoretical results related to our algorithm BDMBC. 
As pointed out in Section \ref{sec::dcbc}, the ability of mode detection plays a fundamental role in density-based clustering,
so we begin with the convergence rates of mode estimators based on the probability of localized level set in Section \ref{sec::ratesbdmbc}. 
More specifically,
we present the convergence rates of BDMBC for mode estimation in Section \ref{sec::ratesbdmbc}. Our results reveal the benefits of bagging to reduce the number of nearest neighbors in bagged $k$-distance at each round.
Then we further show the convergence rates of the level set estimation in Section \ref{equ::ratelevelset}. Moreover, we show that BDMBC can find all clusters with varying densities using a single threshold.
Finally, we compare our studies with other existing ones in the literature in Section \ref{sec::comments}.

We first introduce the general assumptions needed throughout our theoretical analysis. We first make assumptions about the underlying density function in Assumption \ref{ass::cluster}.

\begin{assumption}\label{ass::cluster}
	Assume that $P$ has a Lebesgue density $f$ with the support $\mathcal{X}=[0,1]^d$. 
	\begin{enumerate}
		\item[(i)] [\textbf{Boundedness}]
		There exist constants $\underline{c}, \overline{c} > 0$ such that $\underline{c} \leq f \leq \overline{c}$.
		\item[(ii)] [\textbf{Smoothness}] 
		$f$ is $\alpha$-H\"{o}lder continuous, where $0 < \alpha \leq 1$, i.e., for all $x, x' \in \mathcal{X}$, there exists a constant $c_L > 0$ such that $|f(x) - f(y)| \leq c_L \|x - y\|^{\alpha}$.
	\end{enumerate}
\end{assumption}

Condition \textit{(i)} in Assumption \ref{ass::cluster} requires that the density is upper and lower bounded by positive constants, which is a mild density assumption, see, e.g., \cite{audibert2007fast,cannings2020local}.
Condition \textit{(ii)} in Assumption \ref{ass::cluster} requires the H\"{o}lder continuity on the global density function.
When $\alpha$ is small, the density function fluctuates more sharply, which results in the difficulty of estimating the density accurately.

Then we need to make the following assumption on the modes. 
Before we proceed, we denote the set of modes of $f$ by 
\begin{align*}
	\mathcal{M} := \{ x \in \mathcal{X} :\exists r > 0, \forall x' \in B(x, r), f(x') \leq f(x) \}. 
\end{align*}

\begin{assumption}[Twice Differentiability around Modes] \label{ass::modes}
	Assume that there exists some $r_{\mathcal{M}}>0$ such that $f$ is twice continuously differentiable around the disjoint neighborhood $B(m_i,r_{\mathcal{M}})$ of each $m_i\in \mathcal{M}$, $i = 1, \ldots, \#(\mathcal{M})$.
	Denote the gradient and Hessian of $f$ by $\nabla f$ and $H$, respectively,  and assume that $H(x)$ is negative definite at all $x \in \mathcal{M}$.
\end{assumption}

The above mode assumption is widely adopted for mode estimation \cite{dasgupta2014optimal,jiang2018quickshift++,jang2021meanshift++}, which requires that the density $f$ near the modes is concave \cite{saxena2017review,xu2019power}.
Compared to Assumption \ref{ass::modes}, this condition requires second-order smoothness of the density function near the modes.
Here we exclude modes at the boundary of support of $f$, where $f$ can not be continuously differentiable. In fact, this problem can be handled under an additional boundary smoothness assumption. This approach only complicates the analysis, while the main insights remain the same for interior modes. We refer the reader to \cite{dasgupta2014optimal} for more discussions.

We mention that Assumption \ref{ass::modes} holds for a large non-parametric class of functions including \textit{Morse} density functions, which are widely used in the density-based clustering and mode estimation and topological data analysis; see, e.g., \cite{chacon2015population,arias2016estimation,chacon2020modal} and the references therein. A map $f$ is a Morse function if its critical points are non-degenerate, i.e., the Hessian of $f$ at each critical point is non-singular.
As pointed out in \cite[Corollary 1.12]{matsumoto2002introduction}, critical points of morse functions are isolated. It thus follows that Morse functions on compact sets have finitely many critical points, which implies that Morse density functions satisfy Assumption \ref{ass::modes} if we choose a sufficiently small $r_{\mathcal{M}}$.

\subsection{Convergence Rates of BDMBC for Mode Estimation} \label{sec::ratesbdmbc}

To derive the convergence rates of our BDMBC for mode estimation, we need the following assumption under which clusters can be separated with respect to distinct modes. 

\begin{assumption}[Unflatness]\label{ass::flatness1}
	Assume that $P$ has a Lebesgue density $f$ with the support $\mathcal{X} = [0, 1]^d$ and
	there exist constants $\gamma > 0$, $c_{\gamma} > 0$, $\epsilon_0 > 0$ such that for all $\theta \in [0, \overline{c}]$ and $\epsilon \in(0, \epsilon_0]$, we have
	$P(x : |f(x) - \theta| \leq \epsilon) \leq c_{\gamma} \epsilon^{\gamma}$.
\end{assumption}

Assumption \ref{ass::flatness1} is a well-known condition introduced by \cite{polonik1995measuring} for the level set estimation problem. Clearly, the larger the $\gamma$, the more steeply $f$ must approach $\lambda$ from above. In fact, Assumption \ref{ass::flatness1} ensures there are no such flat regions where there is no change in density.  It is commonly adopted in cluster analysis \cite{steinwart2015fully, jiang2018quickshift++}.

The following theorem presents the convergence rates of the mode recovery based on the PLLS with respect to the Euclidean distance.

\begin{theorem}\label{thm::mode}
	Let Assumptions \ref{ass::cluster}, \ref{ass::modes} and \ref{ass::flatness1} hold with $2\alpha\gamma \leq 4+d$ and $\widehat{\mathcal{M}}$ be the mode estimator as in \eqref{equ::modesestimator}. 
	Then for every mode $m_i \in \mathcal{M}$ and $\lambda \geq c$ with the constant $c$ which will be specified in the proof, by choosing 
	\begin{align}\label{equ::relation}
		\begin{split}
			k_{D,n} \asymp & \log n,
			\qquad
			s_n \asymp  n^{\frac{d}{4+d}} (\log n)^{\frac{4}{4+d}},
			\qquad
			B_n \geq n^{\frac{3}{4+d}}(\log n)^{\frac{d+1}{4+d}},
			\\
			& k_{G,n} \asymp \log n,
			\quad \quad
			k_{L,n} \gtrsim n^{1-\frac{\alpha\gamma}{4+d}}(\log n)^{1+\frac{\alpha\gamma}{4+d}},
		\end{split}
	\end{align} 
	there exists a mode estimate $\widehat{m}_i$ such that with probability at least $1 - 3/n^2$, there holds 
	\begin{align*}
		\|\widehat{m}_i - m_i\|_2
		\lesssim (\log n/n)^{\frac{1}{4+d}}.
	\end{align*}
	Moreover, there exist distinct cluster estimators $\widehat{C}_i \in \mathcal{C}_B(\lambda)$, $1 \leq i \leq k$, such that $\widehat{m}_i \in \widehat{C}_i$. 
\end{theorem}

Theorem \ref{thm::mode} together with Theorem \ref{thm::modesingle} implies that up to a logarithm factor, the convergence rate of BDMBC turns out to be minimax optimal for mode estimation, if we choose the sub-sample size $s$, the number of nearest neighbors $k_D$ and $k_L$, and the bagging rounds $B$ according to \eqref{equ::relation}, respectively. In other words, when the bagging technique is combined with the $k$-distance for mode estimation, the optimal convergence rate is obtainable. Moreover, if we choose the number of nearest neighbors $k_G$ properly, then we can recover the cluster that corresponds to the modes in a subjective manner. 

Notice that for a given dataset, \eqref{equ::relation} yields that $k_D$ and $B$ is proportional to $s$ and ${k_D/s}$, respectively. Therefore, only a few independent bootstrap samples are required to use for the computation of $k$-distance at each bagging round. As a result, $k_{D,n}$ is reduced to $\mathcal{O}(\log n)$ in \eqref{equ::relation}, instead of $\mathcal{O}(n^{4/(4+d)} (\log  n)^{d/(4+d)})$ in the following Theorem \ref{thm::modesingle} in Section \ref{sec::ratesdmbc} for DMBC, the special case of BDMBC without bagging, i.e. $B=1$ and $s=n$.

\subsection{Convergence Rates of BDMBC for Level Set Estimation} \label{equ::ratelevelset}

In this section, we establish convergence rates of level set estimation for the PLLS of our BDMBC algorithm. 
Before we proceed, we need to introduce the population version of $R_i(x)$, namely $\overline{R}_i(x)$ defined by
\begin{align}\label{equ::overlinerx}
	\overline{R}_i(x):=\inf\{r\geq 0: P(B(x,r))\geq i/n\}.
\end{align} 
For $k_L\in \mathbb{N}$, we define the population version of the probability of the localized level set,
\begin{align}\label{equ::populationpkl}
	p_{k_L}(x):=P(f(y)\leq f(x)|y\in B(x,\overline{R}_{k_L}(x)).
\end{align}
where $\overline{R}_{k_L}(x)$ is defined by \eqref{equ::overlinerx}. Compared with the empirical version defined by \eqref{equ::empiricalpkl}, the local radius function in \eqref{equ::populationpkl} relies on the population version of the $k_L$-distance.

Then for $k_L\in \mathbb{N}$ and $\lambda \in [0,\overline{c}]$, we define the level set of $p_{k_L}(x)$ by $L_{k_L}(\lambda) := \{ x : p_{k_L}(x) \geq \lambda\}$. 
Then the level set estimation of our BDMBC is
$\widehat{L}_{k_L}(\lambda) := \{ x : \widehat{p}_{k_L}^B(x) \geq \lambda\}$ with $\widehat{p}_{k_L}^B(x)$ defined by \eqref{getah}.

To further conduct our analysis, we need the following assumption introduced in \cite{jiang2017density,jiang2019robustness} on the behavior of level set boundaries.

\begin{assumption}[$\beta$-regularity]\label{ass::flatness}
	Assume that $P$ has a Lebesgue density $f$ with the support $\mathcal{X}=[0,1]^d$. 
	Let $C$ be a connected component of the level set $L_{k_L}(\lambda)$. 
	Assume that there exists a constant $c_{\beta} > 0$ such that for $x \notin C$ and $k_L\in \mathbb{N}$, we have $c_{\beta} d(x, L_{k_L}(\lambda))^{\beta} \leq \lambda - p_{k_L}(x)$, where $d(x,A):=\inf_{y\in A} d(x,y)$.
\end{assumption}

The $\beta$-regularity in Assumption \ref{ass::flatness} ensures that there is a sufficient decay around level set boundaries so that the level sets are salient enough to be detected.
The next theorem gives the estimation rate in terms of the Hausdorff distance 
$d_{\mathrm{Haus}}(A,A')=\max\{\sup_{x\in A} d(x,A'),
$
\noindent
$\sup_{x'\in A'} d(x',A)\}$.

\begin{theorem}\label{thm::levelset}
	Let Assumptions \ref{ass::cluster}, \ref{ass::flatness1} and \ref{ass::flatness} hold with $\gamma > d/(2\alpha+d)$ and $\alpha\gamma \geq \beta$.
	By choosing 
	\begin{align*}
		k_{D,n} \asymp \log n,
		\
		s_n \asymp n^{\frac{d}{2\alpha+d}}(\log n)^{\frac{2\alpha}{2\alpha+d}},
		\
		B_n \geq n^{\frac{1+\alpha}{2\alpha+d}}(\log n)^{\frac{\alpha+d-1}{2\alpha+d}}, 
		\
		k_L \gtrsim n^{1 - \frac{\alpha \gamma-\beta}{2\alpha+d}} (\log n)^{\frac{\alpha\gamma-\beta}{2\alpha+d}}, 
	\end{align*}
	then with probability $P^n$ at least $1 - 3/n^2$, there holds
	\begin{align*}
		d_{\mathrm{Haus}} \Bigl( \widehat{L}_{k_L}(\lambda),L_{k_L}(\lambda) \Bigr)
		\lesssim (\log n/n)^{\frac{1}{2\alpha+d}}.
	\end{align*}
\end{theorem}

Note that the choice of $k_{D,n}$ in Theorem \ref{thm::levelset} is the same as that in Theorem \ref{thm::mode}, whereas the choice of $s_n$ and $B_n$ are different. In fact, compared with Theorem \ref{thm::mode}, we take H\"{o}lder smoothness assumptions in Theorem \ref{thm::levelset} instead of the twice differentiability in Assumption \ref{ass::modes}, and thus larger subsample size $s$ is required.
Moreover, the convergence rate in Theorem \ref{thm::levelset} turns out to be $\mathcal{O}(n^{-1/(2\alpha+d)})$ up to a logarithm factor, which matches the lower bound established in \cite{tsybakov1997nonparametric,jiang2017density}.

\subsection{Comments and Discussions}
\label{sec::comments}

This section presents some comments on the obtained results on the convergence rates for mode estimation and level set estimation and compares them with related findings in the literature.

\subsubsection{Comments on Convergence Rates for Mode Estimation}

Existing modal clustering algorithms using gradient ascent or borrowing from work in cluster tree estimation to seek modes. 
To the best of our knowledge, \cite{dasgupta2014optimal} first gives a procedure that recovers multiple modes of a density by using a top-down traversal of the density levels.
The best known practical approach for mode estimation is the mean-shift procedure and its variants \cite{fukunaga1975estimation,li2007nonparametric,chen2018modal,ghassabeh2018modified} consisting of gradient ascent of the appropriately smooth density estimator $f_D$. For the theoretical analysis, \cite{arias2016estimation} shows that mean-shift's updates converge to the correct gradient ascent steps. More recently, \cite{jiang2017consistency,jiang2018quickshift++}
show that Quick Shift and its variants can attain strong statistical guarantees without the second-order density assumption required to analyze mean-shift. However, most of these methods need a proper smooth density estimator as preliminaries.  Thus these clustering methods can be very sensitive to user-defined parameters.

In this paper, we first propose a new measurement called the \textit{probability of localized level sets} (\textit{PLLS}) built from the bagged $k$-distance. Then we provide a novel mode estimation and then establish optimal convergence rates for multi-level density problems.
Moreover, as a result of the analysis of bagged $k$-distance, we show that the bagging technique helps to reduce the subsample size and the number of neighbors simultaneously for mode estimation and thus increases the robustness in Theorem \ref{thm::mode}.

\subsubsection{Comments on Convergence Rates for Level Set Estimation}

We show in Theorem \ref{thm::levelset} the convergence rate turns out to be $\mathcal{O}(n^{-1/(2\alpha+d)})$, which matches the lower bound established in \cite{tsybakov1997nonparametric,jiang2017density}.
Previous studies were limited in that  they mainly focus on the level set estimation with a single mode. Therefore, these results are inadequate for multiple modes with varying densities. In this paper, our results recover localized level sets for varying densities (Theorem \ref{thm::levelset}). We show that under mild continuity and regularity assumptions on the density function, optimal convergence rates can be derived for the level set estimation of the PLLS.
Since the level set corresponds to the ``domain of attraction'' of the modes of $f$, i.e., population clusters, we can find all clusters using a single threshold.
This reveals the local adaptivity of our BDMBC in multi-level density clustering.

\section{Error and Complexity Analysis}
\label{sec::Error}

In this section, we first conduct error analysis related to the bagged $k$-distance in Section \ref{sec::analysisbagged}. We mention that the theoretical results for mode estimation and level set estimation in Section \ref{sec::results} are all built upon the results for bagged $k$-distance in this Section.
To be specific, in Section \ref{sec::analysisbagged}, we first conduct error decomposition for the hypothetical density estimation.
Then, in Subsections \ref{sec::BaggingError}-\ref{sec::ApproximationError}, we present the upper bounds for the bagging error, estimation error, and approximation error, respectively. 
With these preparations, we establish in Section \ref{sec::RatesPseudoDE} the uniform convergence rates for the hypothetical density estimation under mild smoothness Assumption \ref{ass::cluster}.
Moreover, in this Section, we further establish faster convergence rates for the hypothetical density estimation around the modes under mode Assumption \ref{ass::modes}. 
Finally, we conduct algorithm complexity analysis in Section \ref{sec::algorithmcomplexity} to demonstrate the efficiency of our algorithm.

\subsection{Error Analysis for Bagged $k$-Distance} \label{sec::analysisbagged}

The bagged $k$-distance can not be analyzed directly since it is not of the form of commonly used estimators. 
According to \eqref{equ::fbk}, the problem of analyzing the bagged $k$-distance can be reduced to the problem of analyzing the hypothetical density estimation.
Then we can apply standard techniques to the analysis of $f_B(x)$ and then use it for our bagged $k$-distance.

Let us turn to the empirical probability of the localized level set defined in \eqref{getah}. By using the hypothetical density estimation \eqref{equ::fbk}, $\widehat{p}_{k_L}^B(x)$ can be re-expressed as
\begin{align}\label{getah1}
	\widehat{p}_{k_L}^B(x):=\frac{1}{k_L}\sum_{i=1}^{k_L}\eins\{f_B(X_{(i)}(x))\geq f_B(x)\}.
\end{align}

Then we conduct the following error decomposition of hypothetical density estimation
\begin{align*}
	\bigl| f_B(x) - f(x) \bigr| 
	& = \biggl| \frac{\bigl( \sum_{i=1}^n p_i (i/n)^{1/d} \bigr)^d}{V_d \bigl( R_k^B(x) \bigr)^d} - f(x) \biggr| 
	\nonumber\\
	& = \biggl| \frac{\bigl( \sum_{i=1}^n p_i \bigl( (i/n) / (V_d f(x)) \bigr)^{1/d} \bigr)^d - \bigl( R_k^B(x) \bigr)^d}{\bigl( R_k^B(x) \bigr)^d} \biggr| \cdot f(x)
	\nonumber\\
	& = \biggl| \frac{\sum_{i=1}^n p_i \bigl( (i/n) / (V_d f(x)) \bigr)^{1/d}  -  R_k^B(x) }{\bigl( R_k^B(x) \bigr)^d} \biggr| \cdot  f(x) \cdot
	\\
	& \phantom{=}
	\cdot \sum_{j=0}^{d-1} \biggl( \sum_{i=1}^n p_i ((i/n) / (V_d f(x)))^{1/d}\bigg)^j \big(R_k^B(x)\big)^{d-1-j}.
\end{align*}

Let us consider the first term of the product on the right-hand side of the decomposition above. 
The numerator term is regarded as the difference between the weighted $k$-distance $\sum_{i=1}^n p_i \bigl( (i/n) / (V_d f(x)) \bigr)^{1/d} $ and the bagged $k$-distance, while the denominator term is the bagged $k$-distance $R_k^B(x)$ to the power $d$.

To conduct theoretical analysis for the bagged $k$-distance, we need to consider the estimator with infinite bagging rounds, which can be expressed as 
\begin{align}
	\widetilde{R}_k^B(x):=\mathbb{E}_{P_Z}^B[R_k^B(x) | \{ X_i \}_{i=1}^n],
\end{align}
where $P_Z$ denotes the sub-sampling probability distribution.

Note that the bagged $k$-distance can be re-expressed as a weighted $k$-distance, which is amenable to statistical analysis. To be specific, let $X_{(i)}(x)$ be the $i$-th nearest neighbor of $x$ in $D_n$ w.r.t. the Euclidean distance and $R_i(x)=\|x-X_{(i)}(x)\|$. For $1\leq b\leq B$, we can re-express the $k$-distance with respect to the set $D_b$ as
\begin{align*}
	R_k(x,D_b)=\sum^n_{i=1} p_i^b R_i(x)
\end{align*}
with $p_i^b:=\eins\{X_{(i)}(x) \text{ is the } k\text{-th nearest neighbor of } x \text{ in } D_b\}$.
Then the bagged $k$-distance in  \eqref{eq::bd} can be re-expressed as 
\begin{align*}
	R_k^B(x)=\frac{1}{B}\sum_{b=1}^B\sum_{i=1}^n 	p_i^b R_i(x).
\end{align*}
Therefore, we have the estimator with infinite bagging rounds
\begin{align}\label{equ::tilderbkx}
	\widetilde{R}_k^B(x)=\sum_{i=1}^n p_iR_i(x).
\end{align}
with $p_i$ defined by \eqref{equ::def}.

Finally, we are able to make the following error decomposition on the numerator as 
\begin{align*}
	& \biggl| R_k^B(x) - \sum_{i=1}^n p_i ((i/n) / (V_d f(x)))^{1/d} \biggr|
	\nonumber\\
	& \leq \bigl| R_k^B(x) -\widetilde{R}_k^B(x)\bigr|
	+ \biggl| \sum_{i=1}^n p_i \bigl( R_i(x) - \overline{R}_i(x) \bigr) \biggr| 
	+ \biggl| \sum_{i=1}^n p_i \bigl( \overline{R}_i(x) - ((i/n) / (V_d f(x)))^{1/d} \bigr) \biggr|.
\end{align*}
The three terms on the right hand side are called \textit{bagging error}, \textit{estimation error}, and \textit{approximation error}, respectively. 
More specifically, since we are not able to repeat the sampling strategy an infinite number of times, the bagging procedure brings about the first error term.
The second term is called the \textit{estimation error} since it is associated with the empirical measure $D_n$ and the last term is called \textit{approximation error} since it indicates how the error is propagated by the bagged $k$-distance for hypothetical density estimation. In the next three sections, we will bound these three terms respectively.

\subsubsection{Bounding the Bagging Error} \label{sec::BaggingError}

The next proposition shows that the bagging error term is determined by the number of bagging rounds $B$ and the ratio $k/s$.

\begin{proposition}\label{prop::weightB}
	Let Assumption \ref{ass::cluster} hold.
	Moreover, let $R_k^B(x)$ and $\widetilde{R}_k^B(x)$ be defined by \eqref{eq::bd} and \eqref{equ::tilderbkx}, respectively. 
	Then for all $x\in \mathcal{X}$, with probability $P^B_Z\otimes P^n$ at least $1-1/n^2$, there holds
	\begin{align*}
		\big| R_k^B(x) - \widetilde{R}_k^B(x) \big|
		\lesssim \sqrt{(k/s)^{2/d}\log n/B}+\log n/B.
	\end{align*}
\end{proposition}

\subsubsection{Bounding the Estimation Error} \label{sec::EstimationError}

We now establish the upper bound of the estimation error of weighted $k$-distance.
This oracle inequality will be crucial in establishing the convergence results of the estimator.

\begin{proposition}\label{prop::weightsamplerho}
	Let Assumption \ref{ass::cluster} hold. 
	Furthermore, let $R_{k}(x)$ be the $k$-nearest neighbor distance of $x$ and $\overline{R}_k(x)$ be the quantile diameter function of $x$ defined by \eqref{equ::overlinerx}.
	Moreover, let $p_i$ be the probability as in  \eqref{equ::def} and suppose that $(kn/s)^{1-d/2} \gtrsim (\log n)^{1+d/2}$. 
	Then for all $x\in \mathcal{X}$, with probability $P^n$ at least $1-2/n^2$, there holds
	\begin{align*}
		\Bigg| \sum_{i=1}^n p_i \bigl( R_i(x) - \overline{R}_i(x) \bigr) \Bigg|
		\lesssim (k/s)^{1/d-1/2}(\log n/n)^{1/2}.
	\end{align*}
\end{proposition}

\subsubsection{Bounding the Approximation Error} \label{sec::ApproximationError}

The following result on bounding the approximation error term shows that the approximation error can be small by choosing the ratio $k/s$ appropriately.

\begin{proposition}\label{prop::weightedrho}
	Let Assumption \ref{ass::cluster} hold. 
	Moreover, let $p_i$ be the probability as in \eqref{equ::def} and $\overline{R}_i(x)$ be the quantile diameter function of $x$ defined by \eqref{equ::overlinerx}. 
	Then for all $x\in \mathcal{X}$ we have 
	\begin{align*}
		\biggl| \sum_{i=1}^n p_i \overline{R}_i(x) - \sum_{i=1}^n p_i ((i/n) / (V_d f(x)))^{1/d} \biggr|
		\lesssim (k/s)^{(1+\alpha)/d}.
	\end{align*}
\end{proposition}

\subsubsection{Convergence Rates for Hypothetical  Density Estimation}
\label{sec::RatesPseudoDE}

The next proposition presents the convergence rates of the hypothetical density estimator induced by the bagged $k$-distance.

\begin{proposition}\label{thm::main3}
	Let Assumption \ref{ass::cluster} hold. Moreover, let $f_B(x)$ be the hypothetical density estimator as in \eqref{equ::fbk}. By choosing 
	\begin{align*}
		k_{D,n} \asymp \log n,
		\quad
		s_n \asymp n^{\frac{d}{2\alpha+d}}(\log n)^{\frac{2\alpha}{2\alpha+d}},
		\quad 
		B_n \geq n^{\frac{1+\alpha}{2\alpha+d}}(\log n)^{\frac{\alpha+d-1}{2\alpha+d}}, 
	\end{align*}
	then for all $x \in \mathcal{X}$, with probability $P^n \otimes P_Z^B$ at least $1-3/n^2$,  there holds
	\begin{align*}
		| f_B(x) - f(x) |
		\lesssim (\log n/n)^{\frac{\alpha}{2\alpha+d}}.
	\end{align*}
\end{proposition}

We establish the following finite sample bounds of the hypothetical density estimation near the modes in terms of $L_{\infty}$-norm.

\begin{proposition}\label{thm::main2}
	Let Assumptions \ref{ass::cluster} and \ref{ass::modes} hold. 
	Moreover, let $f_B(x)$ be the hypothetical density estimator as in \eqref{equ::fbk}. 
	By choosing 
	\begin{align*}
		k_{D,n} \asymp \log n,
		\qquad
		s_n \asymp  n^{\frac{d}{4+d}} (\log n)^{\frac{4}{4+d}},
		\qquad
		B_n \geq n^{\frac{3}{4+d}}(\log n)^{\frac{d+1}{4+d}},
	\end{align*}
	then for all $x\in \mathcal{M}_{r/2}$, with probability $P^n \otimes P_Z^B$ at least $1-3/n^2$, there holds
	\begin{align*}
		|f_B(x) - f(x)|
		\lesssim (\log n/n)^{\frac{2}{4+d}}.
	\end{align*}
\end{proposition}

We compare our results with previous theoretical analysis of the $k$-NN for density estimation. \cite{biau2011weighted} introduced a weighted version of the $k$-nearest neighbor density estimate and establish  pointwise consistency results. Recently, \cite{zhao2020analysis} analyzed the $L_{\alpha}$ and $L_{\infty}$ convergence rates of $k$ nearest neighbor density estimation method including two different cases depending on whether the support set is bounded or not.
It is worth pointing out that our analysis of the bagged $k$-distance presents in this study is essentially different from that in the previous works.

First of all, the core challenge in the analysis of bagged $k$-distance is that it cannot be analyzed using existing techniques for standard $k$-nearest neighbor methods.
To solve this problem, we consider the hypothetical density function in \eqref{equ::fbk}. Under the H\"{o}lder continuity assumptions, we derive optimal convergence rates of the hypothetical density function with properly selected parameters. 
Moreover, our results are different from the previous statistical analysis since it is conducted from a learning theory perspective \cite{cucker2007learning,steinwart2008support} using techniques such as approximation theory and empirical process theory \cite{vandervaart1996weak,Kosorok2008introduction}.
By exploiting arguments such as Bernstein's concentration inequality from the empirical process theory, we can derive the relationships among the number of bagging rounds $B$, the number of nearest neighbors $k_D$ and the sub-sample size $s$ (Theorem \ref{thm::mode}).
Moreover, \eqref{equ::relation} implies that $B= \mathcal{O}( n^{3/(4+d)}(\log n)^{(d+1)/(4+d)})$, which is relatively small especially when $d$ is large.

\subsection{Algorithm Complexity Analysis}\label{sec::algorithmcomplexity}

In this subsection, we consider the computational complexity of Algorithm \ref{alg::BDMBC} by adopting tree structures such as $k$-d trees. We here denote the subsample ratio $\rho=s/n$.
Firstly, we consider the average time and space complexity. 
In the first step of Algorithm \ref{alg::BDMBC}, the time and space complexities of bagged $k$-distance are $\mathcal{O}(B \rho nd \log_2 (\rho n))$ and $\mathcal{O}(B \rho n d)$, respectively.
In the second step, the time and space complexities of the PLLS are $\mathcal{O}(B \rho nd \log_2 (\rho n))$ and $\mathcal{O}(B \rho n d)$, respectively.
The third step is level-set clustering, which includes finding $k_G$ nearest neighbors, calculating core points, and calculating the connected components.
The cost of time in finding $k_G$ nearest neighbors is $\mathcal{O}(nd \log_2 n)$ with $\mathcal{O}(n d)$ memory. 
The costs of calculating both core points and connected components are $\mathcal{O}(n)$ in time and $\mathcal{O}(n)$ in memory.
Overall, we need $\mathcal{O}((B \rho + 1) nd \log_2 n)$ time and $\mathcal{O}((B \rho +1) nd)$ memory. 
On the other hand, since the time complexity for finding the $k$ nearest neighbors in the worst case is $\mathcal{O}(n^2)$, we can derive the worst case complexity of BDMBC to be $\mathcal{O}((B \rho^2 + 1) n^2d)$ by similar inductions.

The time and space complexities are summarized in Table \ref{tab::complexity}.

\begin{table}[!ht]
	\centering
	\captionsetup{justification=centering}
	\caption{Time and Space Complexity for BDMBC}
	\label{tab::complexity}
	\begin{tabular}{lcc}
		\toprule
		Steps & Time Complexity &  Space Complexity   \\ 
		\hline
		Bagged $k$-distance & $\mathcal{O}(B\rho nd \log_2 (\rho n))$ & $\mathcal{O}(B\rho n d)$ \\
		Calculation of the PLLS & $\mathcal{O}(B \rho nd \log_2 (\rho n))$ & $\mathcal{O}(B \rho n d)$ \\
		Finding $k_G$ nearest neighbors & $\mathcal{O}(nd \log_2 n)$ & $\mathcal{O}(n d)$ \\
		Calculating core points & $\mathcal{O}(n)$ & $\mathcal{O}(n)$ \\
		Calculating the connected components & $\mathcal{O}(n)$ & $\mathcal{O}(n)$ \\
		Labeling points below the level-set by 1NN & $\mathcal{O}(nd \log_2 n)$ & $\mathcal{O}(n d)$ \\
		\hline
		\textbf{Averaged cost} & $\mathcal{O}((B \rho + 1) nd \log_2 n)$ & $\mathcal{O}((B \rho +1) nd)$ \\
		Worst cost & $\mathcal{O}((B \rho^2 + 1) n^2 d)$ & $\mathcal{O}((B \rho +1) nd)$ \\
		\bottomrule
	\end{tabular}
\end{table}

Moreover, reducing the sampling ratio $\rho$ can also reduce the time and space complexity.
For large-scale datasets, we can combine with sampling and/or core-set techniques \cite{agarwal2005geometric, ros2021progressive} to reduce the time and space complexity.

\section{Experiments} \label{sec::Experiment}

In this section, we conduct numerical experiments to illustrate our proposed BDMBC:
Firstly, we give an illustrative example in Section \ref{sec::subsec::illustrative} to demonstrate how and why the proposed BDMBC works.
Secondly, we conduct the experiments of mode estimation on several two-dimensional synthetic datasets in Section \ref{sec::subsec::synmode}.  The ability of the BDMBC algorithm to identify all modes verifies the theoretical results about the mode estimation of the BDMBC. The success of mode estimation is an essential part of our BDMBC for clustering.
Thirdly, we evaluate our proposed BDMBC by comparing with other methods on publicly available real-world datasets in Section \ref{sec::subsec::realcomp}.
In Section \ref{sec::subsec::params}, we conduct parameter analysis of the BDMBC algorithm, reveal the relationship between the parameter choosing strategies and the performances of BDMBC, and empirically verify the fact that bagging can narrow the searching grid of parameters.
We further provide the scalability experiments in Section \ref{sec::subsec::scala} to show that bagging can significantly decrease the computational cost of algorithms without sacrificing accuracy.
All the experiments are implemented in Python and run on a machine within a high-performance computing cluster, where one node with 64GB main memory and a 24-core CPU cluster is used.

\subsection{Illustrative Example} \label{sec::subsec::illustrative}

We continue to use the one-dimensional synthetic dataset {\tt 3Mix} containing three Gaussian distributions $\mathcal{N}(0.20, 0.001)$, $\mathcal{N}(0.32, 0.002)$ and $\mathcal{N}(0.65, 0.007)$ with equal mixture component weights for illustration.
For each point, we assign its class to the Gaussian mixture component with the highest probability and this synthetic dataset contains two high-density small clusters and one low-density large cluster. 

\begin{figure}[htbp]
	\centering
	\vskip -0.05in
	\vspace{-6pt}
	\subfigure[Ground-True Clusters]{
		\begin{minipage}{0.8\columnwidth}
			\centering
			\centerline{\includegraphics[height=0.35\columnwidth]{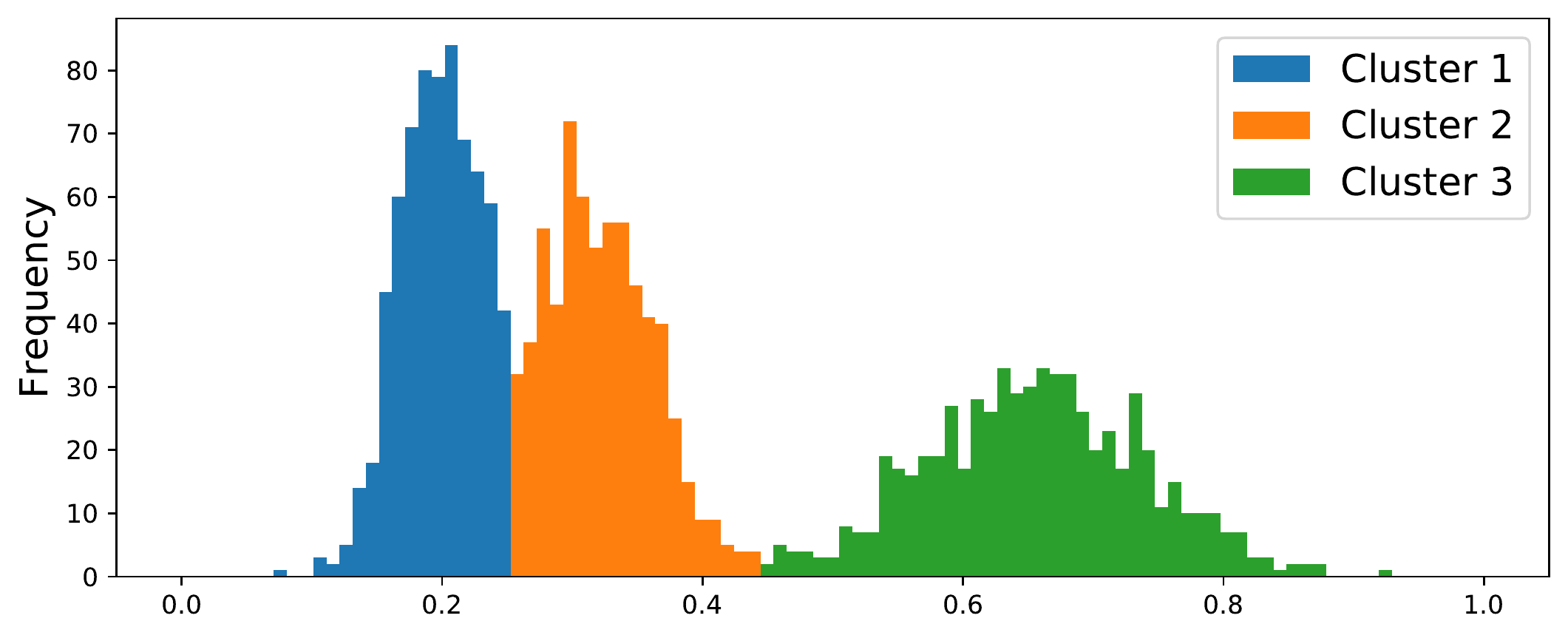}}
			\label{fig::illustrative::cluster}
		\end{minipage}
	} \\ 
	\vspace{-10pt}
	\subfigure[Level-Set Clustering with Empirical Probability of Localized Level Set]{
		\begin{minipage}{0.8\columnwidth}
			\centering
			\centerline{\includegraphics[height=0.35\columnwidth]{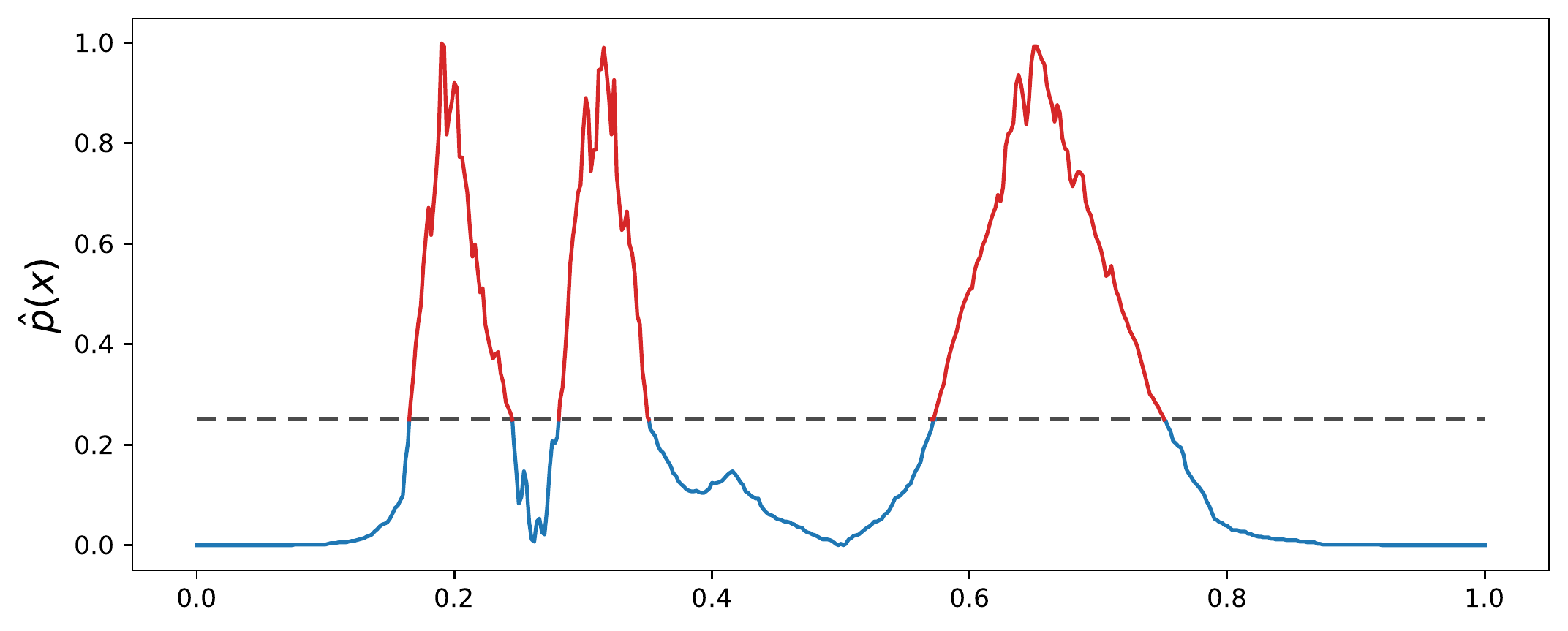}}
			\label{fig::illustrative::threshold}
		\end{minipage}
	} \\ 
	\vspace{-10pt}
	\subfigure[Inverse of Bagged $k$-Distance and the Cluster Tree Structure]{
		\begin{minipage}{0.8\columnwidth}
			\centering
			\centerline{\includegraphics[height=0.35\columnwidth]{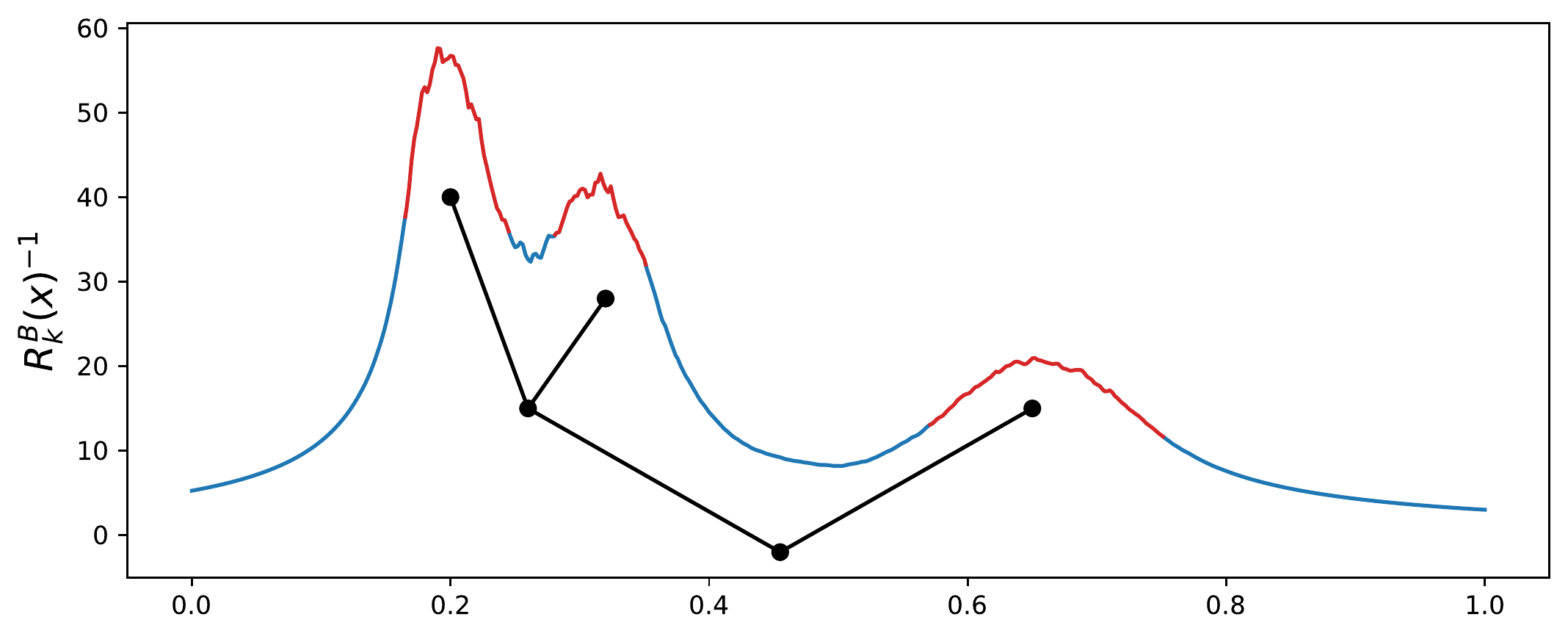}}
			\label{fig::illustrative::clustertree}
		\end{minipage}
	}
	\vspace{-10pt}
	\vskip 0.0in
	\caption{An illustrative example of the BDMBC algorithm, showing that our method can find out all modes with varied densities at a time and make density-based clustering much easier.}
	\label{fig::illustrative}
	\vspace{-6pt}
	\vskip -0.05in
\end{figure} 

We generate 2000 points from the distribution and plot the histogram in Figure \ref{fig::illustrative::cluster}. The clusters are filled in blue, orange, and green, respectively.
For out BDMBC, there are six hyper-parameters, including $B$, $s$, $k_D$, $k_L$, $k_G$, $\lambda$. We discuss the selection of these hyper-parameters later in Section \ref{sec::subsec::params}.
We mention that the points below the threshold $\lambda$ or belonging to clusters with too tiny cluster sizes will be simply affiliated to the nearest cluster.
For the three Gaussian distribution, the best ARI performance of BDMBC can reach 0.99. As the threshold $\lambda$ is applied on the PLLS instead of the density, the density-based clustering algorithm uses a single threshold to identify clusters successfully.

Figure \ref{fig::illustrative::threshold} shows the empirical probability function. The horizontal line $\lambda=0.2$ is the threshold. The line filled in red or blue represents the region whose PLLS is larger or smaller than the threshold, respectively. It can be seen that $\widehat{p}_{k_L}^B(x)$ can narrow the density difference between high-density clusters and low-density clusters by letting the value of local minimums close to zero and the value of maximums close to one. Thus, we can use a global threshold on the probability function $\widehat{p}_{k_L}^B(x)$ to separate different clusters. 

Figure \ref{fig::illustrative::clustertree} shows the BDMBC density estimates and the cluster tree structure. The line filled in red and blue represents regions whose $\widehat{p}_{k_L}^B(x)$ is larger or smaller than the threshold, respectively. More specifically, the density threshold depends on the original density function.
In other words, the threshold of probability function $\widehat{p}_{k_L}^B(x)$ can be adaptive to clusters with varying densities. In other words, the threshold corresponds to different split levels.
The experimental results verified that BDMBC can find out all modes and each cluster estimator corresponds to a leaf cluster as shown in the figures.

\subsection{Mode Detection}  \label{sec::subsec::synmode}

To demonstrate the ability of BDMBC to identify modes so that density-varying mode-based clusters can be detected, we use the following two-dimensional synthetic datasets: The synthetic dataset is generated by a Gaussian mixture model. The Gaussian distribution is consist of five covariance-varying Gaussian distributions with equal mixture component weights. The class of each generated point is the Gaussian mixture component with the highest density. 

\begin{figure}[htbp]
	\centering
	\vspace{-5pt}
	\subfigure[]{
		\begin{minipage}{0.3\columnwidth}
			\centering
			\centerline{\includegraphics[width=\columnwidth]{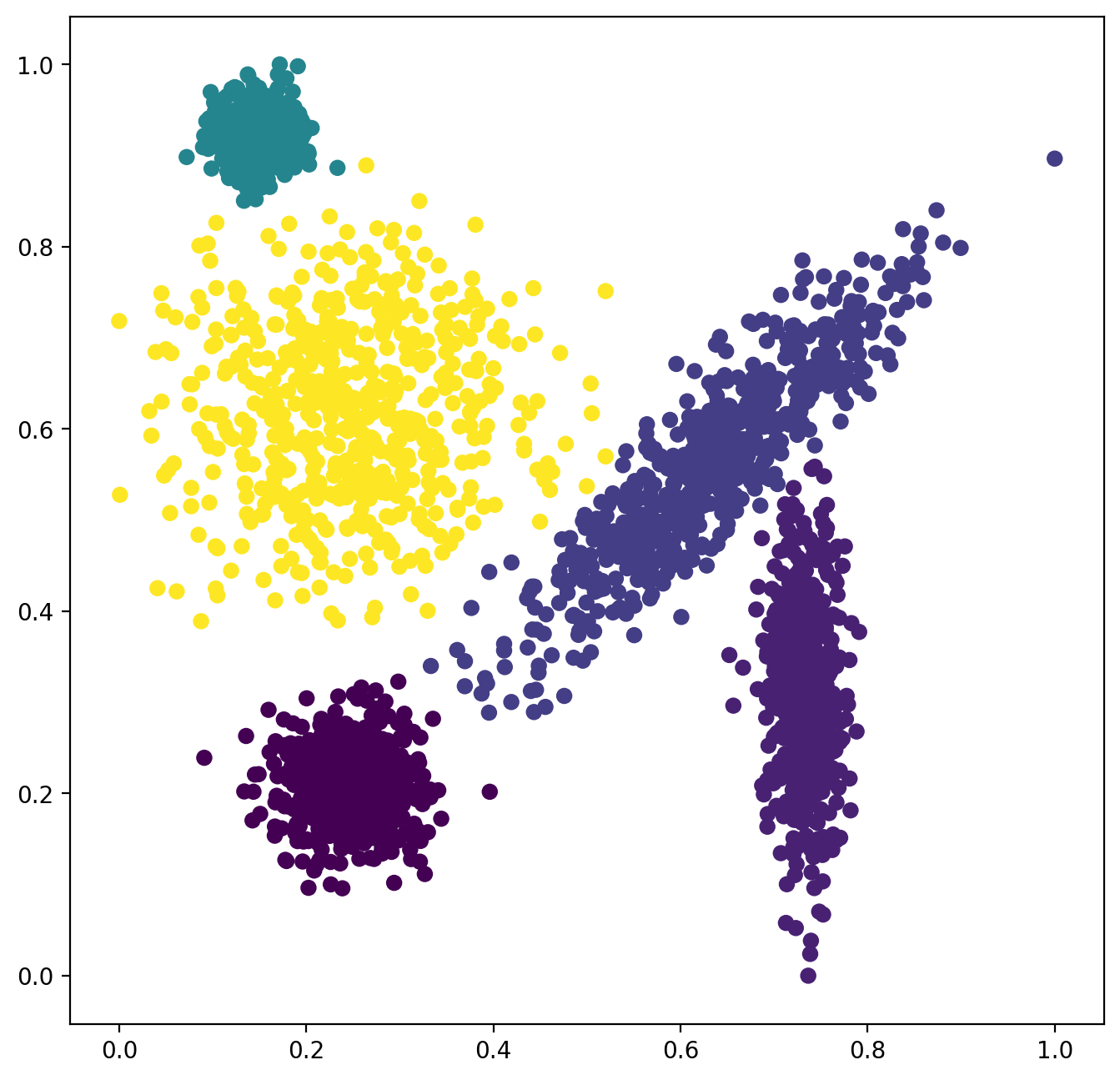}}
			\label{fig::mode:clusters}
	\end{minipage}}
	\subfigure[]{
		\begin{minipage}{0.3\columnwidth}
			\centering
			\centerline{\includegraphics[width=\columnwidth]{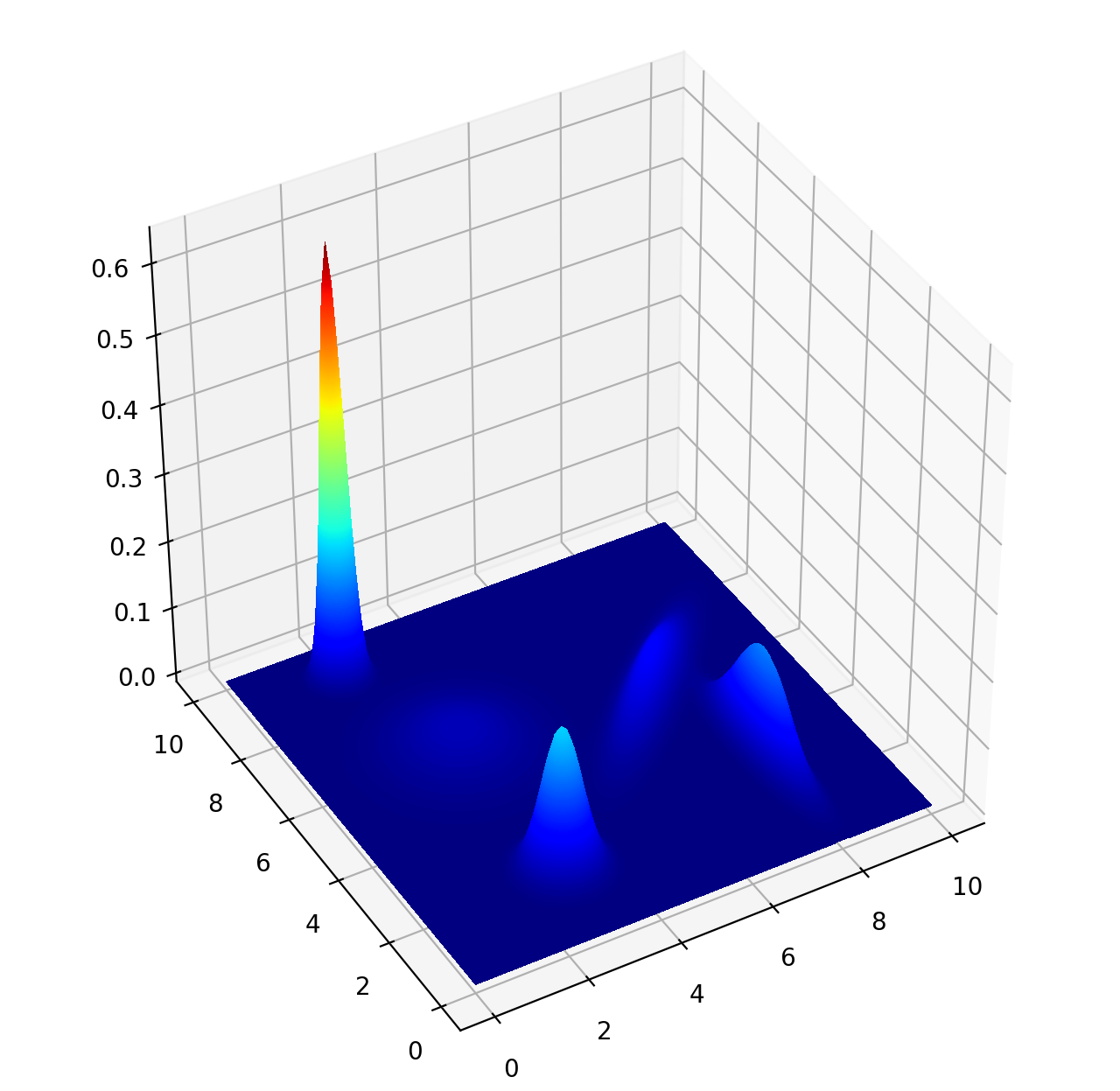}}
			\label{fig::mode:density}
	\end{minipage}}
	\subfigure[]{
		\begin{minipage}{0.3\columnwidth}
			\centering
			\centerline{\includegraphics[width=\columnwidth]{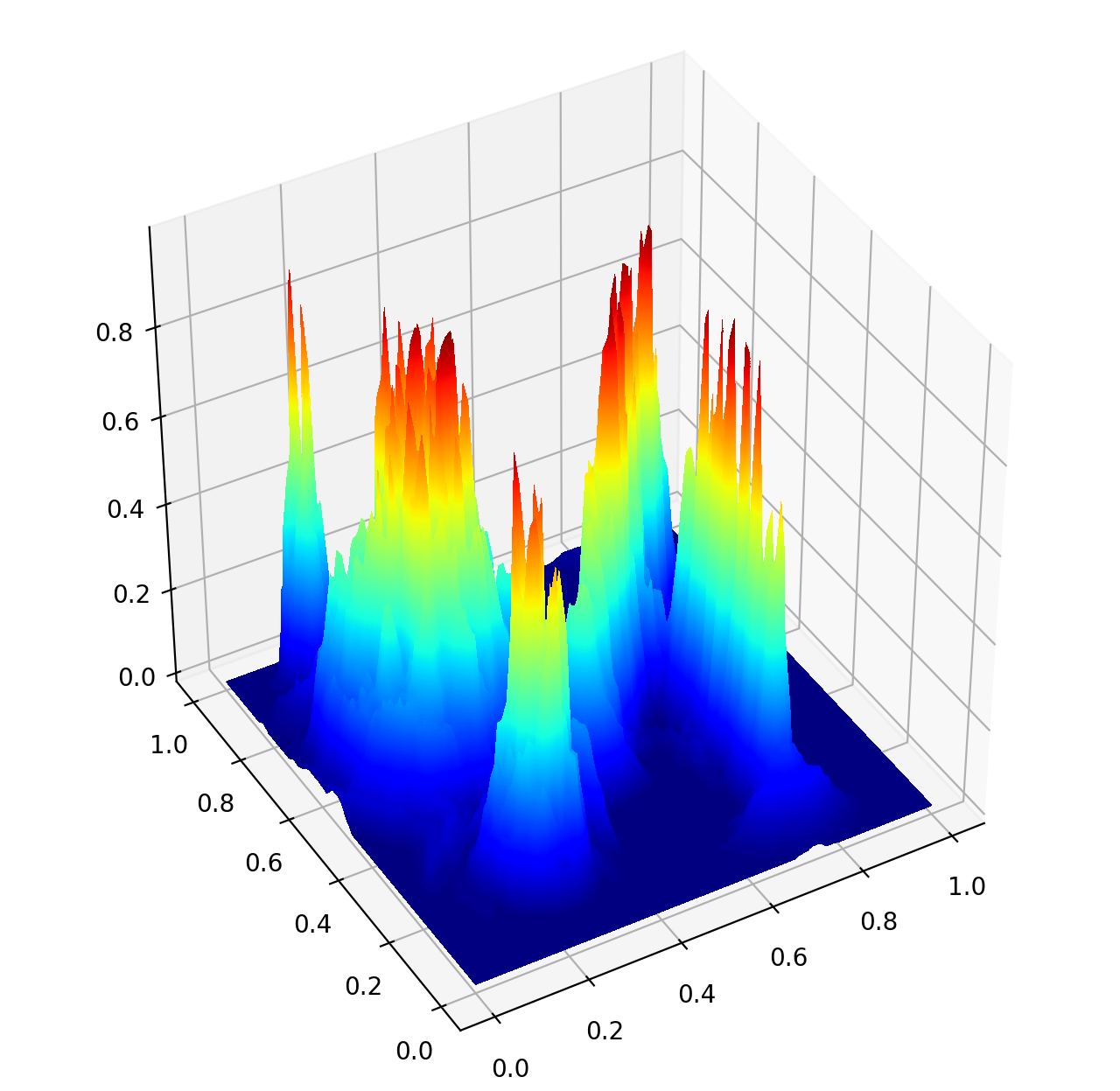}}
			\label{fig::mode:localden}
	\end{minipage}}
	\vskip -12pt
	\caption{Mode detection by the BDMBC algorithm for datasets with density-varying clusters on the Gaussian Mixture Model. (a) Raw Dataset generated from the synthetic distribution. (b) Density of the synthetic distribution. (c) Result of BDMBC on the estimated probability of localized level-set.}
	\label{fig::mode}
\end{figure} 

We generate 3000 points from the distribution, and show the scatter plot of the generated dataset in Figure \ref{fig::mode:clusters}. Different clusters are plotted in different colors. We also visualized the probability density function of the Gaussian mixture model in Figure \ref{fig::mode:density}. Figure \ref{fig::mode:density} shows that clusters are density-varied. The densities of the five modes are very different. 
We apply our BDMBC algorithm to this synthetic dataset, and the estimated PLLS are visualized in Figure \ref{fig::mode:localden}. 
Compared with Figure \ref{fig::mode:density}, the local minimums of the estimated PLLS are close to zero, and the local maximums are close to one.
As our BDMBC can narrow the density difference of high- and low-density clusters, our BDMBC can successfully distinguish five modes in Figure \ref{fig::mode:localden}. Moreover, Figure \ref{fig::mode2} on other three additional two-dimensional synthetic datasets \cite{barton2015clustering} also shows that all modes are covered as peaks. Note that we need not provide an accurate estimation of modes. 
Instead, we use non-overlapping clusters to cover modes and each mode is covered by only one cluster.
We mention that although our BDMBC may enlarge the difference of densities nearby local maximums in Figures \ref{fig::mode} and Figures \ref{fig::mode2}, these fluctuations do not affect the detection of modes and clusters.

\begin{figure}[htbp]
	\centering
	\subfigure[]{
		\begin{minipage}{0.32\columnwidth}
			\centering
			\centerline{\includegraphics[width=\columnwidth]{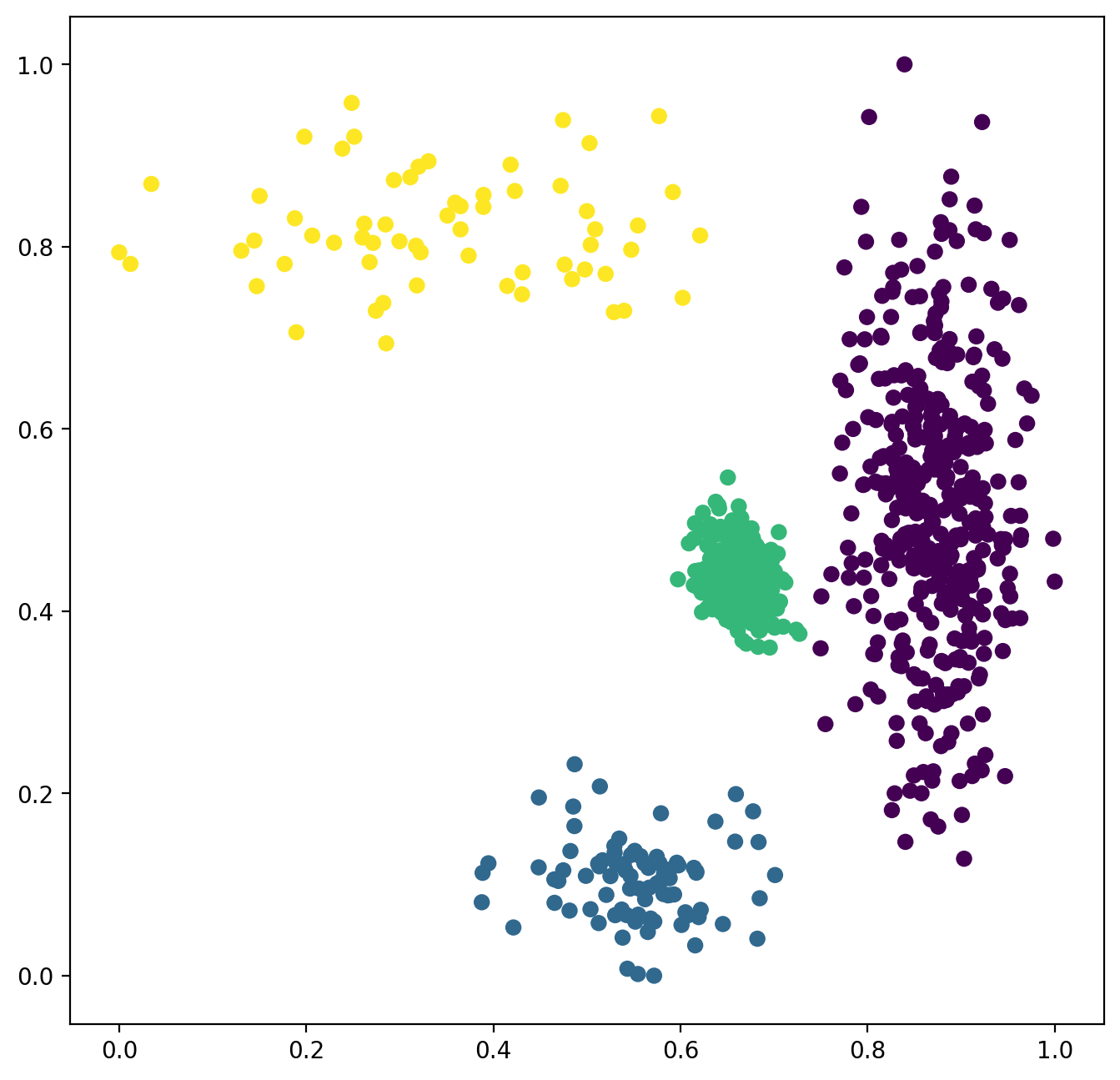}}
			\label{fig::mode2:clusters}
	\end{minipage}}
	\subfigure[]{
		\begin{minipage}{0.32\columnwidth}
			\centering
			\centerline{\includegraphics[width=\columnwidth]{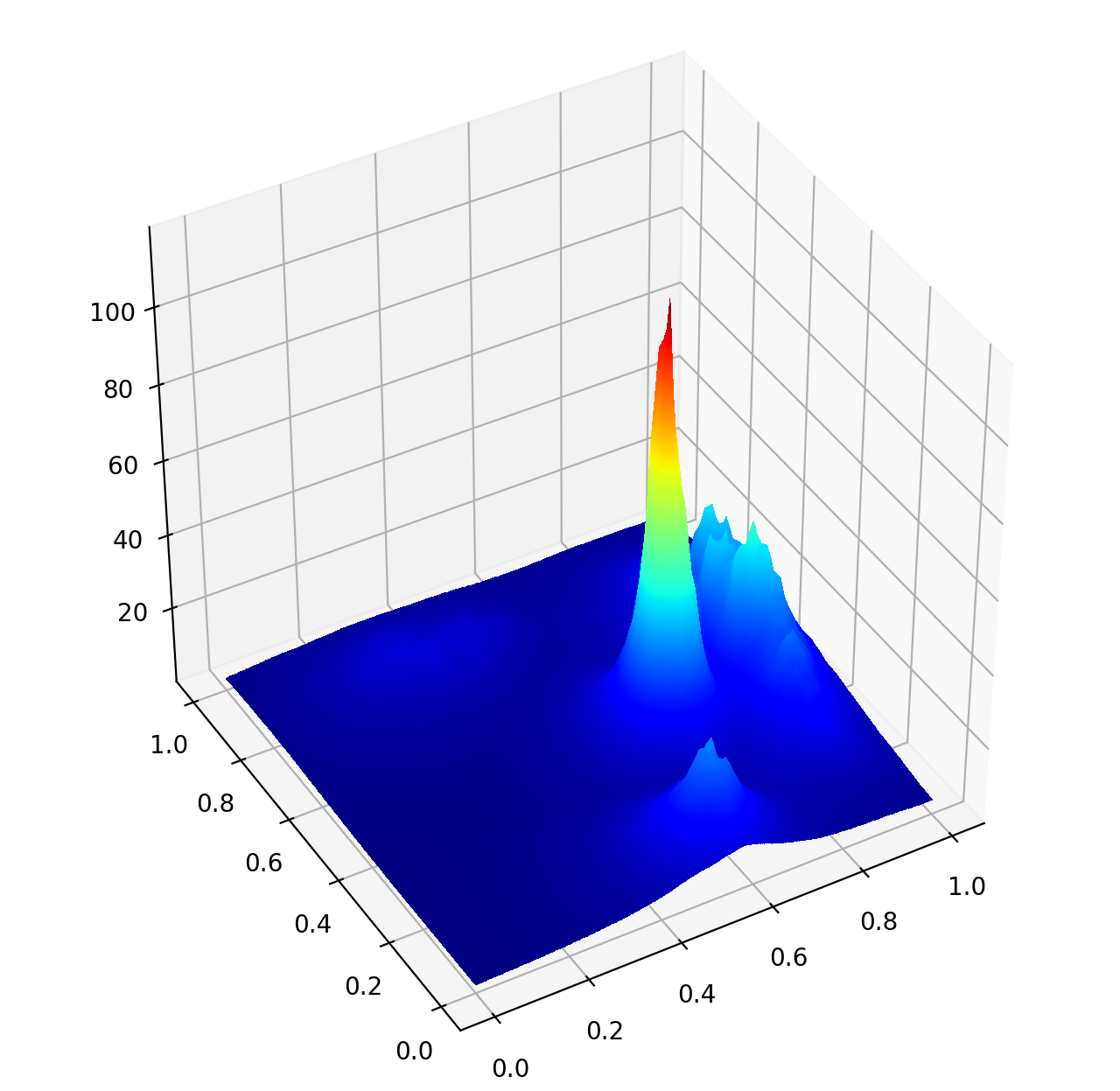}}
			\label{fig::mode2:estden}
	\end{minipage}}
	\subfigure[]{
		\begin{minipage}{0.32\columnwidth}
			\centering
			\centerline{\includegraphics[width=\columnwidth]{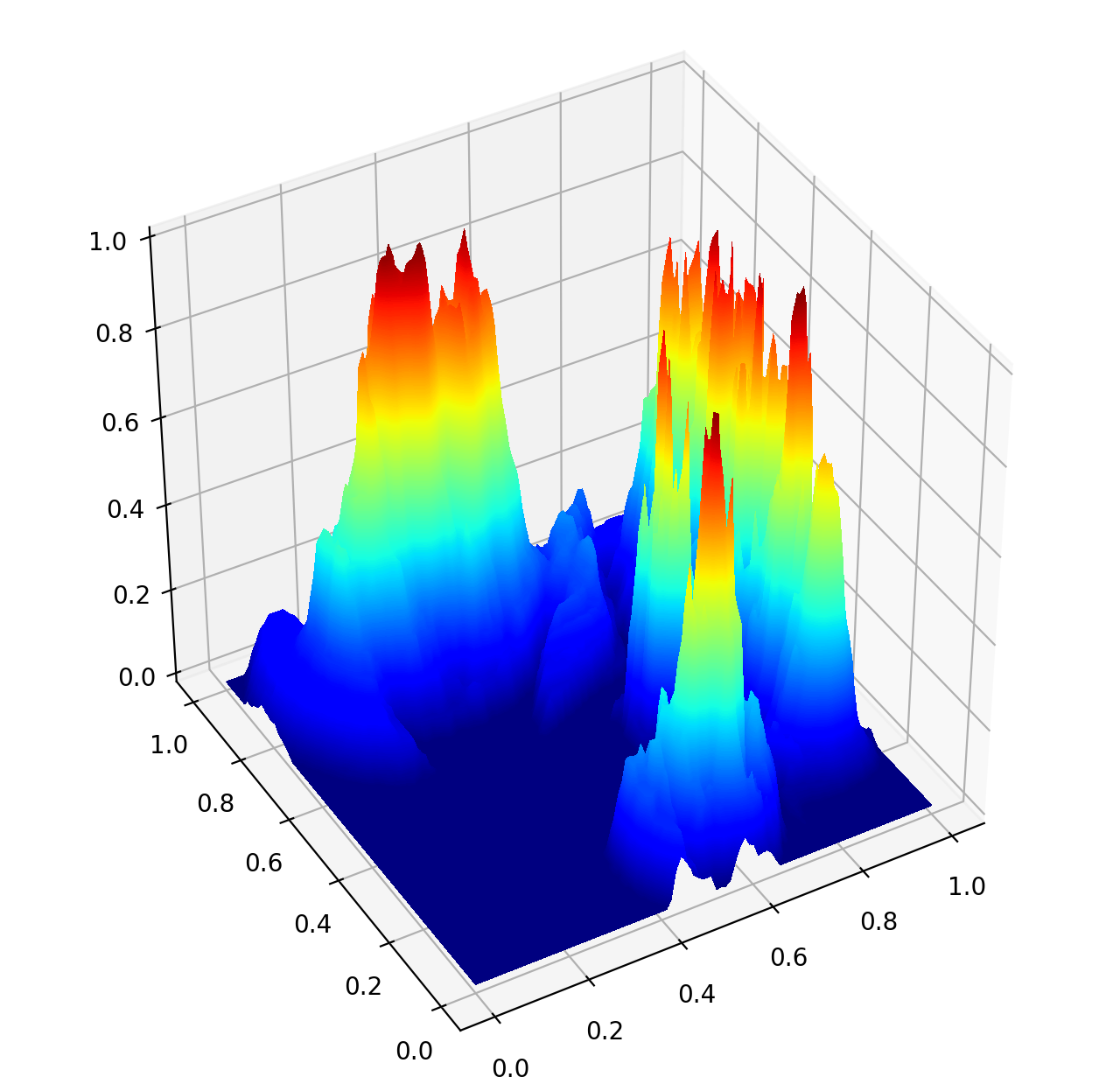}}
			\label{fig::mode2:localden}
	\end{minipage}} \\
	\subfigure[]{
		\begin{minipage}{0.32\columnwidth}
			\centering
			\centerline{\includegraphics[width=\columnwidth]{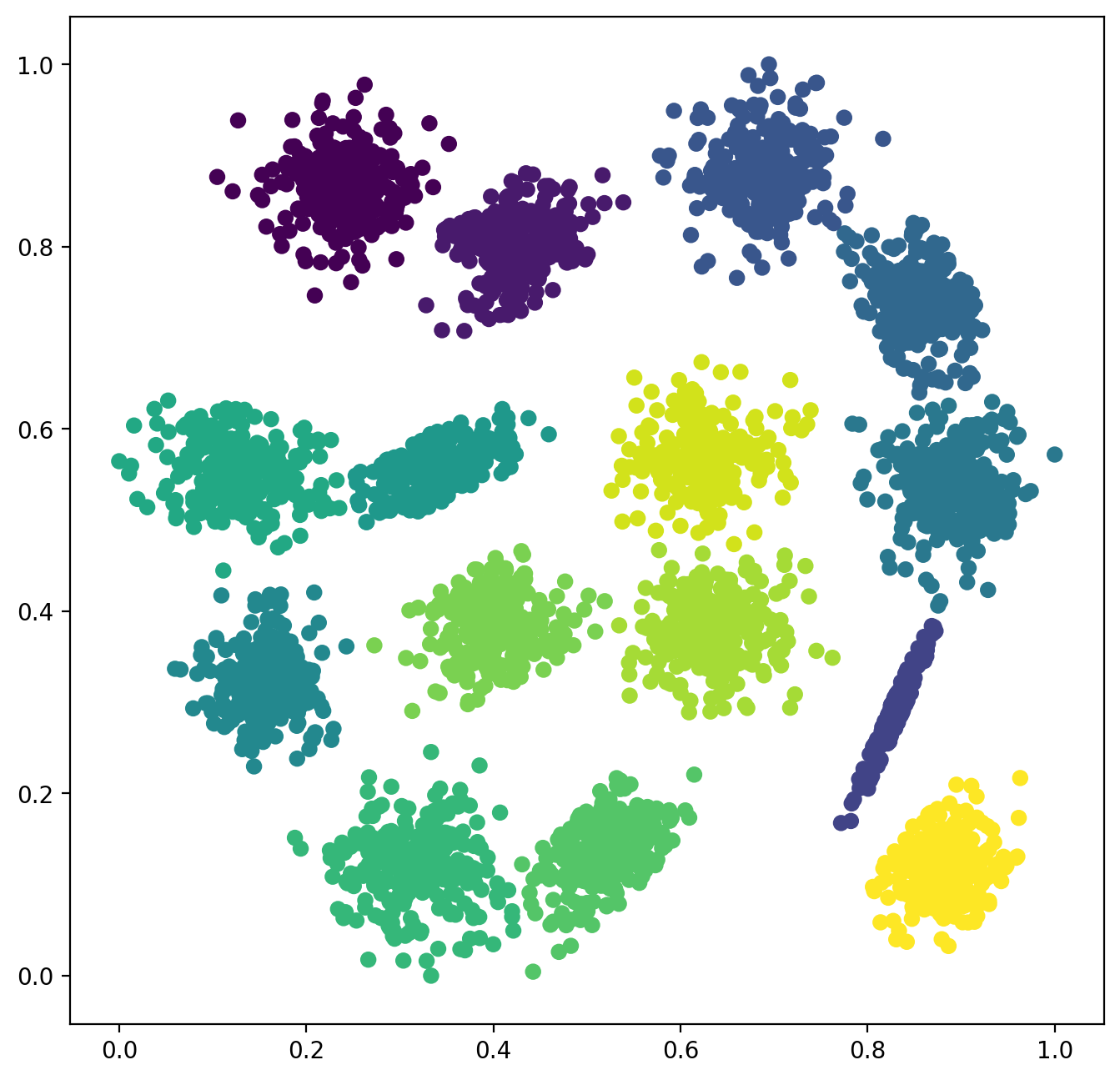}}
			\label{fig::mode3:clusters}
	\end{minipage}}
	\subfigure[]{
		\begin{minipage}{0.32\columnwidth}
			\centering
			\centerline{\includegraphics[width=\columnwidth]{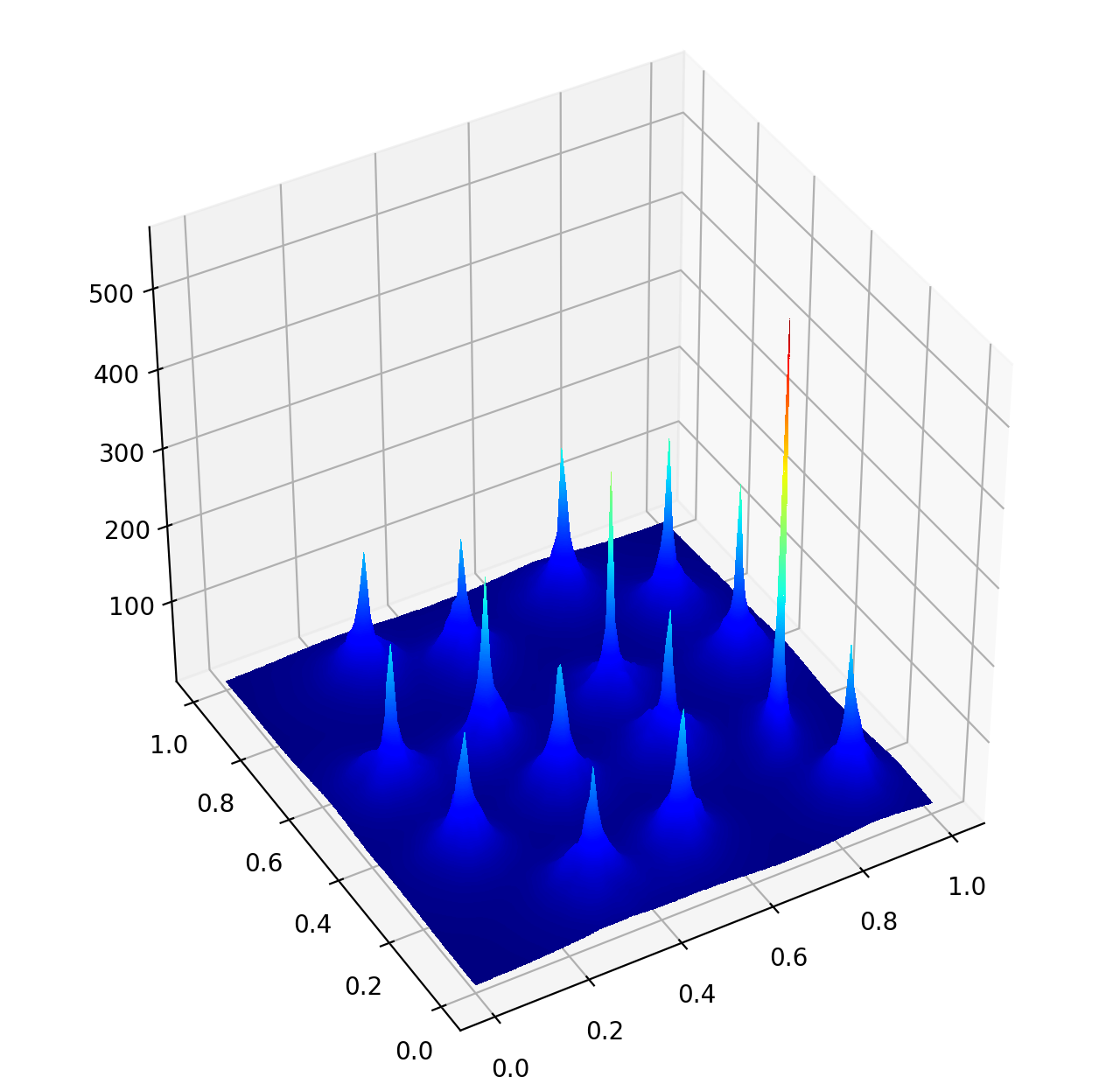}}
			\label{fig::mode3:estden}
	\end{minipage}}
	\subfigure[]{
		\begin{minipage}{0.32\columnwidth}
			\centering
			\centerline{\includegraphics[width=\columnwidth]{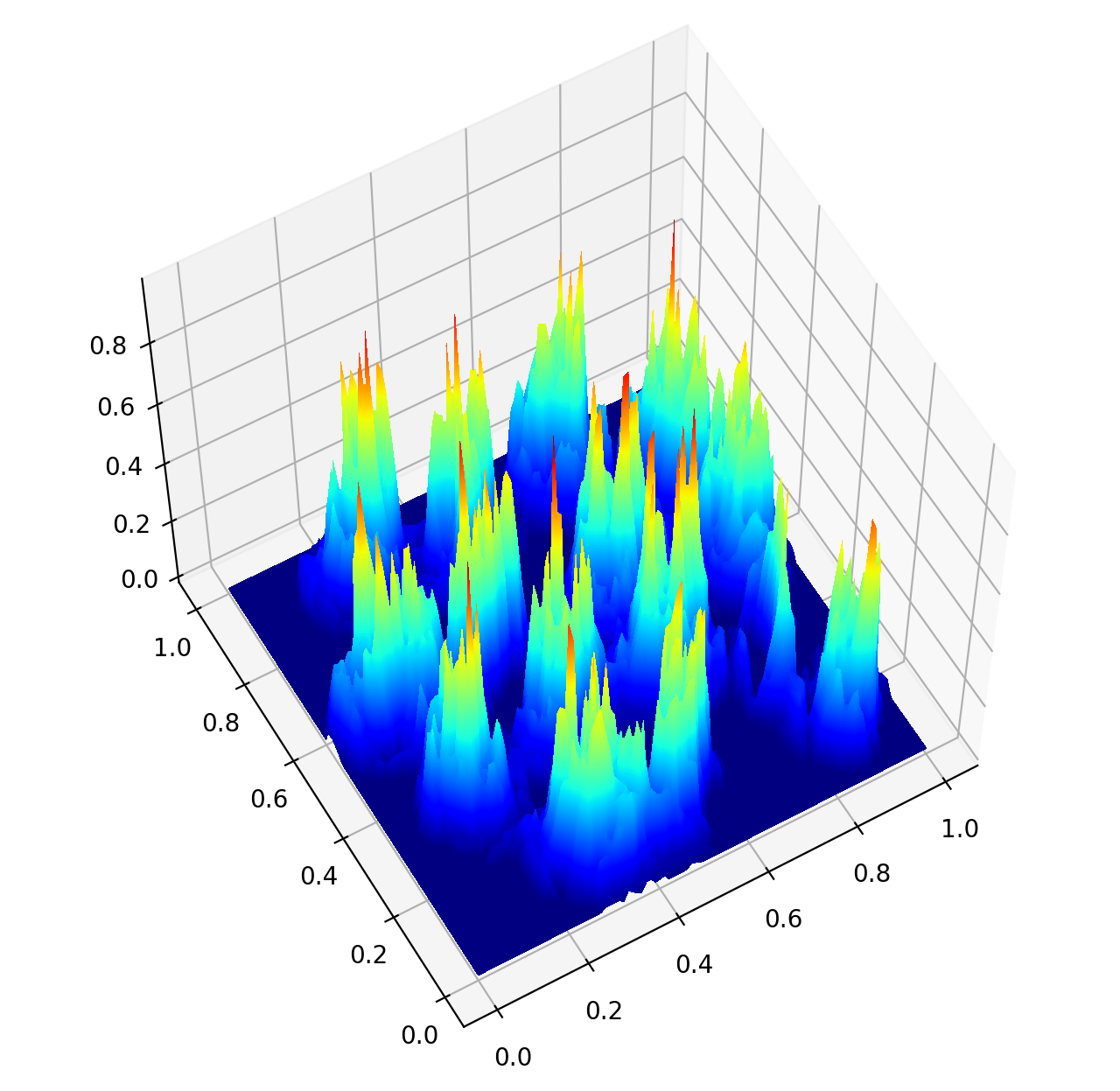}}
			\label{fig::mode3:localden}
	\end{minipage}} \\
	\subfigure[]{
		\begin{minipage}{0.32\columnwidth}
			\centering
			\centerline{\includegraphics[width=\columnwidth]{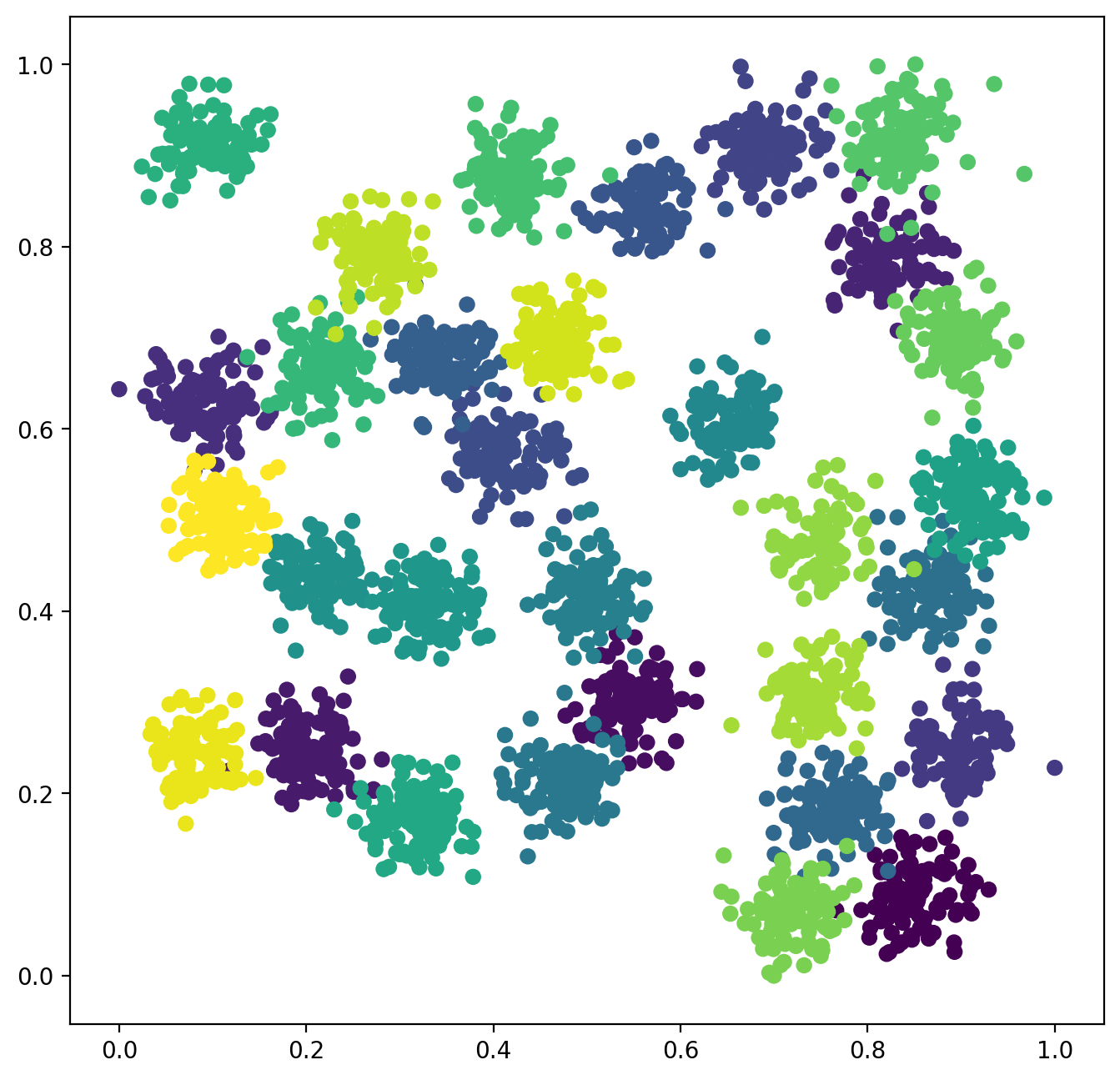}}
			\label{fig::mode4:clusters}
	\end{minipage}}
	\subfigure[]{
		\begin{minipage}{0.32\columnwidth}
			\centering
			\centerline{\includegraphics[width=\columnwidth]{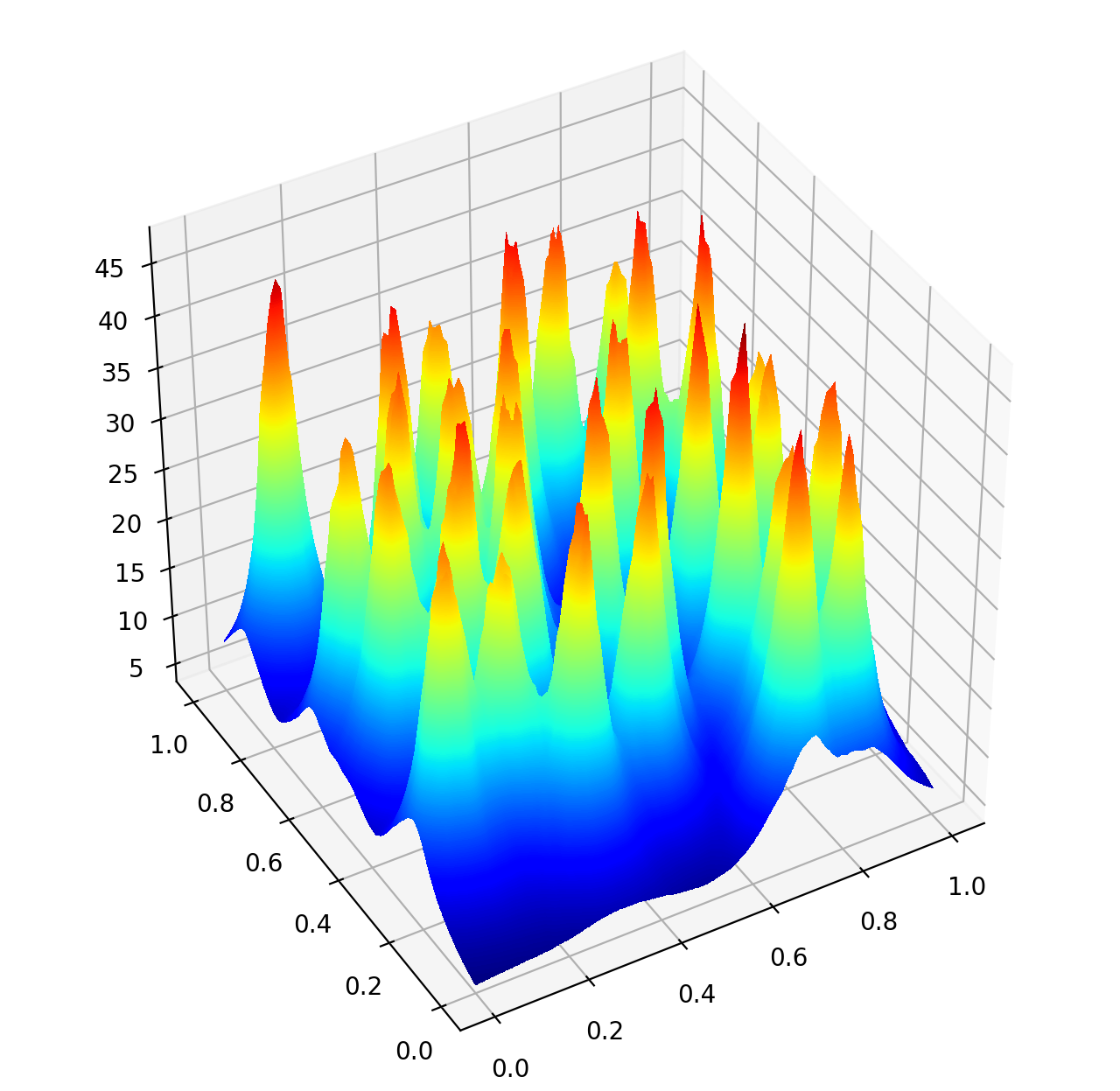}}
			\label{fig::mode4:estden}
	\end{minipage}}
	\subfigure[]{
		\begin{minipage}{0.32\columnwidth}
			\centering
			\centerline{\includegraphics[width=\columnwidth]{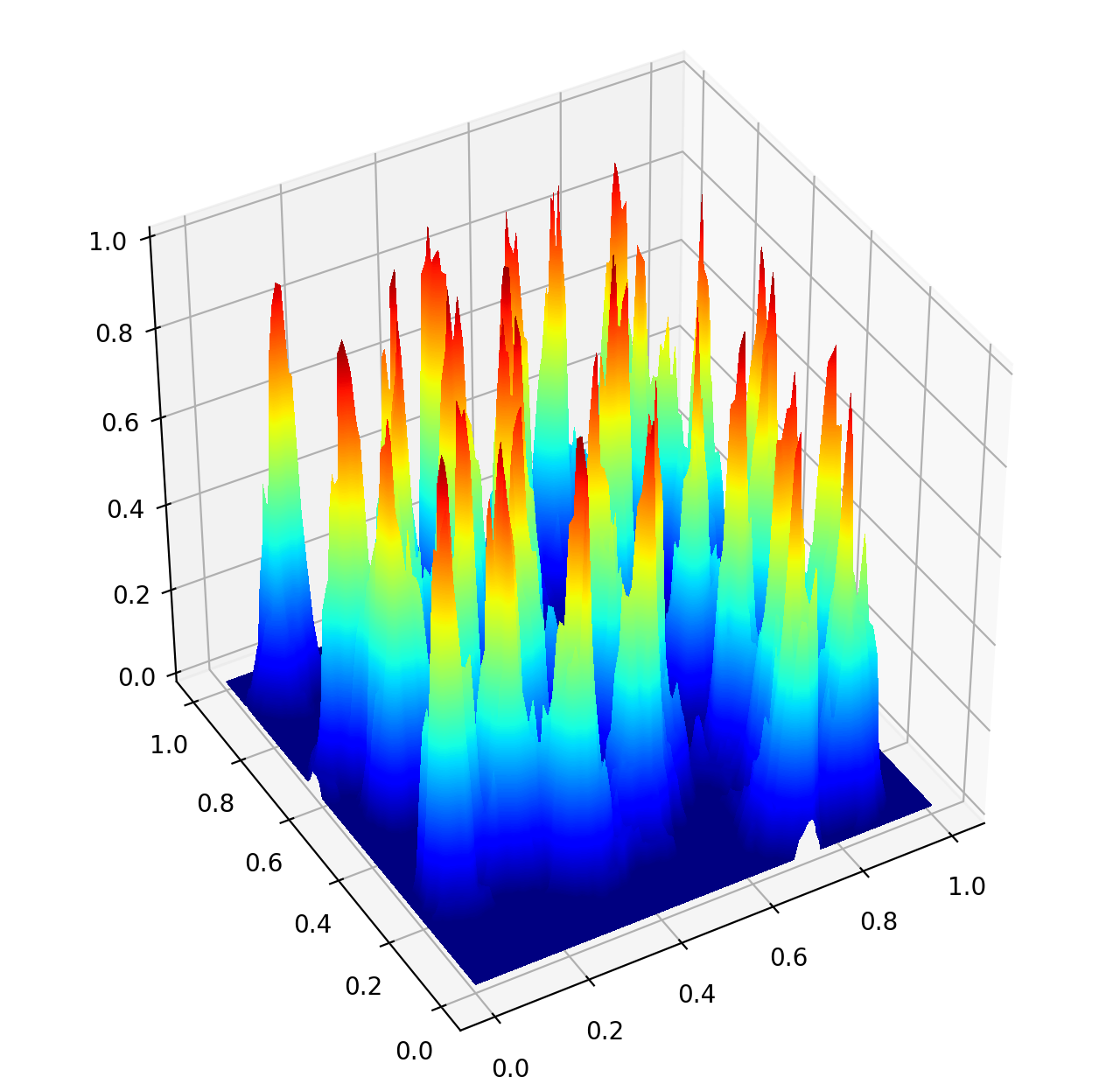}}
			\label{fig::mode4:localden}
	\end{minipage}}
	\caption{Mode detection by the BDMBC algorithm for datasets with density-varying clusters on other synthetic datasets. (a)(d)(g) Scatter plot of the raw dataset generated from the synthetic distribution. (b)(e)(h) Hypothetical density estimation of the synthetic distribution by bagged k-distance. (c)(f)(i) Result of BDMBC on the estimated probability of localized level-set.}
	\label{fig::mode2}
\end{figure}

\subsection{Comparisons on Real Datasets} \label{sec::subsec::realcomp}

We evaluate the clustering performance of our proposed BDMBC by comparing with other competing methods on real-world classification datasets.
Before illustrating the experimental results, we demonstrate the basic information of the real datasets, list all comparing methods and provide the parameter settings for each method, and introduce the metrics for clustering performance evaluation.

\subsubsection{Datasets}

We collect the binary and multi-class classification datasets from UCI Machine Learning Repository, including 
{\tt wine} \cite{UCI2017datasets}, 
{\tt banknote} \cite{UCI2017datasets}, 
{\tt HTRU2} \cite{lyon2016fifty}, 
{\tt iris} \cite{UCI2017datasets}, 
{\tt gisette} \cite{guyon2004result}. 
In addition, we further collect two image datasets for analyzing the high-dimension situations, including 
{\tt COIL} \cite{nene1996columbia} and 
{\tt USPS} \cite{uspsdataset}. 
These datasets are summarized in Table \ref{tab::datasetdescription}, where we list the number of samples $n$, the number of dimensions $d$, the number of clusters $c$, the number of samples falling into the smallest cluster {\tt min\_nums}, and the number of samples falling into the largest cluster {\tt max\_nums}.
Specifically, before experiments, we scale the datasets to the $[0,1]$ range on each dimension.

\begin{table}[!ht]
	\centering
	\captionsetup{justification=centering}
	\vspace{0pt}
	\caption{Descriptions of Datasets}
	\label{tab::datasetdescription}
	\vspace{-6pt}
	\begin{tabular}{c|ccccc}
		\toprule
		Dataset & $n$ & $d$ & $c$ & {\tt min\_nums} & {\tt max\_nums} \\ 
		\bottomrule
		{\tt iris} & 150 & 4 & 3 & 50 & 50 \\ 
		{\tt wine} & 178 & 13 & 3 & 48 & 71 \\ 
		{\tt seeds} & 210 & 7 & 3 & 70 & 70 \\ 
		{\tt banknote} & 1372 & 4 & 2 & 610 & 762 \\ 
		{\tt COIL} & 1440 & 1024 & 20 & 72 & 72 \\ 
		{\tt gisette} & 7000 & 5000 & 2 & 3500 & 3500 \\ 
		{\tt USPS} & 9298 & 256 & 10 & 708 & 1553 \\ 
		{\tt HTRU2} & 17898 & 8 & 2 & 1639 & 16259 \\ 
		\bottomrule
	\end{tabular}
\end{table}

\subsubsection{Baselines and Parameter Settings}

The three baselines include conventional and state-of-the-art clustering algorithms. They are an improved version of DBSCAN called DBSCAN++ \cite{jang2019dbscanpp}, HDBSCAN \cite{campello2013density}, and an improved version of mean-shift called quickshift++ \cite{jiang2018quickshift++}.

Parameter optimization is still an open question for clustering \cite{gan2020data}. 
For each baseline, we search the parameters according to the author's suggestion or try to search the best parameters within a reasonable range. 
For DBSCAN++, the sampling fraction $p$ is firstly set from 0.5 to 1; and the radius for determining the core points in clusters  $\varepsilon_d$ is searched from 0 to 0.8; in addition, the number of neighbors for a point to be labeled as a core point {\tt minPts} is searched from 1 to 20; and the radius for determining if two clusters are connected $\varepsilon_c$ is set from 0 to 0.8.
For HDBSCAN, the search range of the minimum restriction of clusters ${\tt ClusterSize}$ is from 2 to 100, and the search range of another parameter ${\tt MinSamples}$ is from 2 to 20.
For quickshift++, three parameters are included. The number of neighbors to calculate the density $k$ is set from 2 to 30; the threshold for mode detection $\beta$ from 0 to 1; the minimum restriction $\varepsilon$ for connecting clusters is set from 0 to 0.2. 
For DMBC, we set the parameter $k_D$ from 3 to 100 in default, which is used to construct the $k$-distance hypothetical density estimator. 
And we set the parameter grid that is used to constraint the region for calculating the localized density ratio $k_L$ from 3 to 100 in default. 
Finally, two parameters with respect to level-set clustering, including the threshold for the level-set $\lambda$ and the number of nearest neighbors for graph connection $k_G$, range from 0.05 to 0.95 and from 1 to 20 in default.
When it comes to BDMBC, we firstly set the number of bagging $B=100$ and try different sampling rates $\rho \in [0.1, 0.9]$.
And for the parameters for DMBC included in each bagging procedure, we set smaller parameter girds of $k_D$, ranging from 1 to 20.
The reason why we can set a much smaller grid will be detailedly discussed in Section \ref{sec::subsec::params}.

\subsubsection{Clustering Measures}

In our experiments, we use two clustering-based metrics and two classification-based metrics to evaluate the clustering performances of the BDMBC, including Adjusted Rand Index (ARI), Normalized Mutual Infomation (NMI), $F1$ score, and accuracy. The mathematical definition of each measure is defined as follows.

\begin{itemize}
	\item \textbf{ARI}:
	ARI \cite{hubert1985comparing} measures the differences between two clustering results, adjusted for the chance of grouping of elements for Rand Index (RI) \cite{rand1971objective}.
	\begin{align*}
		\displaystyle\mathrm{ARI} = \frac{RI - E[RI]}{\max(RI)-E[RI]},\quad RI=\frac{a+b}{\binom{n}{2}},
	\end{align*}
	where $a$ is the number of paired objects placed in the same cluster in both partitions and $b$ is the number of paired objects placed in different clusters in both partitions.
	\item \textbf{NMI}:
	The Mutual Information (MI) \cite{strehl2002cluster} is a symmetric measure that quantifies the mutual dependence between two random variables, or the information that two random variables share.
	\begin{align*}
		\displaystyle\mathrm{NMI} = \frac{I(Y, P)}{\sqrt{H(Y)H(P)}}
	\end{align*}
	where $H(x)$ represents the entropy of $x$, and $I(Y,P)$ represents the mutual information of $Y$ and $P$.
\end{itemize}
On the other hand, as for the classification measure $F$1 measure and accuracy, we have to first use the Kuhn-Munkres \cite{munkres1957algorithms,kuhn1955hungarian} methods to assign the clustering labels to the underlying labels of instances and then calculate the measure.
\begin{itemize}
	\item $\boldsymbol{F1}$:
	The $F1$ score can be interpreted as a harmonic mean of the precision and recall.
	\begin{align*}
		\displaystyle F1 = \frac{2*\mathrm{precision}*\mathrm{recall}}{\mathrm{precision}+\mathrm{recall}}
	\end{align*}
	where precision describes the ability to only predict really positive samples as samples, denoted as $\mathrm{precision}=\frac{TP}{TP+FP}$. And the recall, calculated by $\mathrm{recall} = \frac{TP}{TP+FN}$, can be interpreted as the ability of the classifier to find all the positive samples. 
	\item \textbf{Accuracy}:
	The accuracy measures the ratio of correct clustering.
	\begin{align*}
		\displaystyle ACC = \frac{\# \text{Correct Classification}}{n}
	\end{align*}
\end{itemize}

\subsubsection{Experimental Results}

In this subsection, we compare the performances of our algorithm with the other three baselines with four measures. The results are demonstrated in Table \ref{tab::realperformances}. The maximum obtained performances are highlighted in \textbf{bold} and the second maximum are highlighted in \textit{italic}. 
From Tables \ref{tab::realperformances}, DMBC and BDMBC outperform over other methods in most datasets. 
The traditional density-based clustering and modal-detecting method show less competitive performances. 
For some datasets with large sample size or high dimension, such as {\tt Gisette} and {\tt USPS}, we outperform these comparing methods by large margins.

\begin{table}[htpb]
	\centering
	\begin{tabular}{ccccccc}
		\toprule
		Data & Measure & DMBC & BDMBC & DBSCAN++ & HDBSCAN & Quickshift++ \\ \midrule
		\multirow{4}{*}{\tt Wine} & {\tt ARI} & \textbf{0.9133}  & \textbf{0.9133}    & \textit{0.8516}  & 0.4766  & 0.7316  \\
		~ & {\tt NMI} &  \textbf{0.8920}   &   \textbf{0.8920} & \textit{0.8364}  & 0.6281  & 0.7402  \\
		~ & {\tt F1} &  \textbf{0.9728} & \textbf{0.9728}    & \textit{0.9502}  & 0.5522  & 0.6813  \\
		~ & {\tt ACC} &  \textbf{0.9719}  & \textbf{0.9719}     & \textit{0.9494}  & 0.6517  & 0.8989  \\ \hline
		\multirow{4}{*}{\tt Iris} & {\tt ARI} & \textbf{0.9222}  & \textbf{0.9222}   & 0.8345  & 0.5681  & \textit{0.8753}  \\
		~ & {\tt NMI} &  \textbf{0.9144}   &  \textbf{0.9144} & 0.8334  & 0.7337  & \textit{0.8515}  \\
		~ & {\tt F1} & \textbf{0.9733}  &  \textbf{0.9733}   & \textit{0.9397}  & 0.5556  & 0.7175  \\
		~ & {\tt ACC} &  \textbf{0.9733}    & \textbf{ 0.9733}   & 0.9400  & 0.6667  & \textit{0.9533}  \\ \hline
		\multirow{4}{*}{\tt Seeds} & {\tt ARI} & \textit{0.8509}  & \textbf{0.8647}    & 0.7789  & 0.5046  & 0.7457  \\
		~ & {\tt NMI} & \textit{ 0.8178}  & \textbf{0.8450}  & 0.7595  & 0.6132  & 0.6973  \\
		~ & {\tt F1} & \textit{0.9475}  &   \textbf{0.9520} & 0.9185  & 0.5425  & 0.8983  \\
		~ & {\tt ACC} &  \textit{0.9476}   &   \textbf{0.9524} & 0.9170  & 0.6571  & 0.9003 \\ \hline
		\multirow{4}{*}{\tt Banknote} & {\tt ARI} & \textit{0.9682}   & \textbf{0.9710}  & 0.9190  & \textit{0.9682}  & 0.8526  \\
		~ & {\tt NMI} & 0.9347 & \textit{0.9382} & 0.8153  & \textbf{0.9402} & 0.8235  \\
		~ & {\tt F1} &  \textit{0.9919}   & \textbf{0.9926} & 0.5663  & \textit{0.9919}  & 0.6422  \\
		~ & {\tt ACC} &    \textit{0.9920}  &  \textbf{0.9927}  & 0.9453  & \textit{0.9920}  & 0.9344  \\ \hline
		\multirow{4}{*}{\tt COIL} & {\tt ARI} &  0.8628 & \textbf{0.8865}  & 0.2080  & \textit{0.8797}  & 0.7179  \\
		~ & {\tt NMI} &  0.9602      & \textit{0.9613}   & 0.5893  & \textbf{0.9639}  & 0.8833  \\
		~ & {\tt F1} &  \textit{0.8807}   &\textbf{ 0.9047}  & 0.2423  & 0.8694  & 0.6114 \\
		~ & {\tt ACC} &   0.8917 & \textbf{0.9153}  & 0.3264  & \textit{0.8944} & 0.7910  \\ \hline
		\multirow{4}{*}{\tt Gisette} & {\tt ARI} & \textit{0.6250}  &  \textbf{0.6482}  & 0.1238  & 0.0822  & 0.0000  \\
		~ & {\tt NMI} &  \textit{0.5268}  & \textbf{0.5401}  & 0.1938  & 0.1790  & 0.0000  \\
		~ & {\tt F1} &  \textit{0.8951 }   & \textbf{0.9026}   & 0.0534   & 0.5767  & 0.0000  \\
		~ & {\tt ACC} &   \textit{0.8953} &  \textbf{0.9026}  & 0.4739  & 0.6277  & 0.0000  \\ \hline
		\multirow{4}{*}{\tt HTRU2} & {\tt ARI} &  \textit{0.8151} & \textbf{0.8327}  & 0.7834  & 0.7401  & 0.7178  \\
		~ & {\tt NMI} &  \textit{0.6733}   & \textbf{0.6782}  & 0.5638  & 0.5504  & 0.4917  \\
		~ & {\tt F1} &   \textit{0.9198} & \textbf{0.9254} & 0.8784  & 0.7719  & 0.3088  \\
		~ & {\tt ACC} &   \textit{0.9755}   & \textbf{0.9765} & 0.9641  & 0.9451  & 0.9465  \\ \hline
		\multirow{4}{*}{\tt USPS} & {\tt ARI} &  \textit{0.8671}  & \textbf{0.8672}  & 0.3125  & 0.6016  & 0.6104  \\
		~ & {\tt NMI} &  \textit{0.8483} & \textbf{0.8490}  & 0.3993  & 0.6734  & 0.7050  \\
		~ & {\tt F1} &   \textit{0.9162} &  \textbf{0.9203}  & 0.2819  & 0.5360  & 0.0398  \\
		~ & {\tt ACC} & \textit{0.9235}    & \textbf{0.9276}   & 0.4133  & 0.6595  & 0.6638 \\ \bottomrule
	\end{tabular}
	\caption{Comparison with baselines on real-world datasets. For each dataset and each measure, we denote the best performance with \textbf{bold} and the second best performance with \textit{italic}.}
	\label{tab::realperformances}
\end{table}

\subsection{Parameter Analysis} \label{sec::subsec::params}

In this subsection, we firstly apply parameter analysis of four hyper-parameters including the number of nearest neighbors for hypothetical density estimation $k_D$, the number of nearest neighbors for the PLLS $k_L$, the number of nearest neighbors for graph connection $k_G$ and the level-set threshold $\lambda$ on the synthetic dataset {\tt 3Clusters}.
Then, we discuss how bagging helps with parameter tuning by comparing the optimal parameters between DMBC and BDMBC on the {\tt 3Clusters} dataset and five additional synthetic datasets. 
Lastly, we give some practical suggestions on the selection of hyper-parameters.
The synthetic datasets introduced in this subsection are from the {\tt Python} package {\tt sklearn} \cite{scikit-learn} and \cite{iglesias2019mdcgen}. For each synthetic dataset, we set the sample size to be $2000$, the noise rate as $0.05$, and visualize the dataset in Figure \ref{fig::syn_datasets}.

\begin{figure}[htbp]
	\centering
	\subfigure[{\tt 3Clusters}]{
		\begin{minipage}{0.32\columnwidth}
			\centering
			\centerline{\includegraphics[width=\columnwidth]{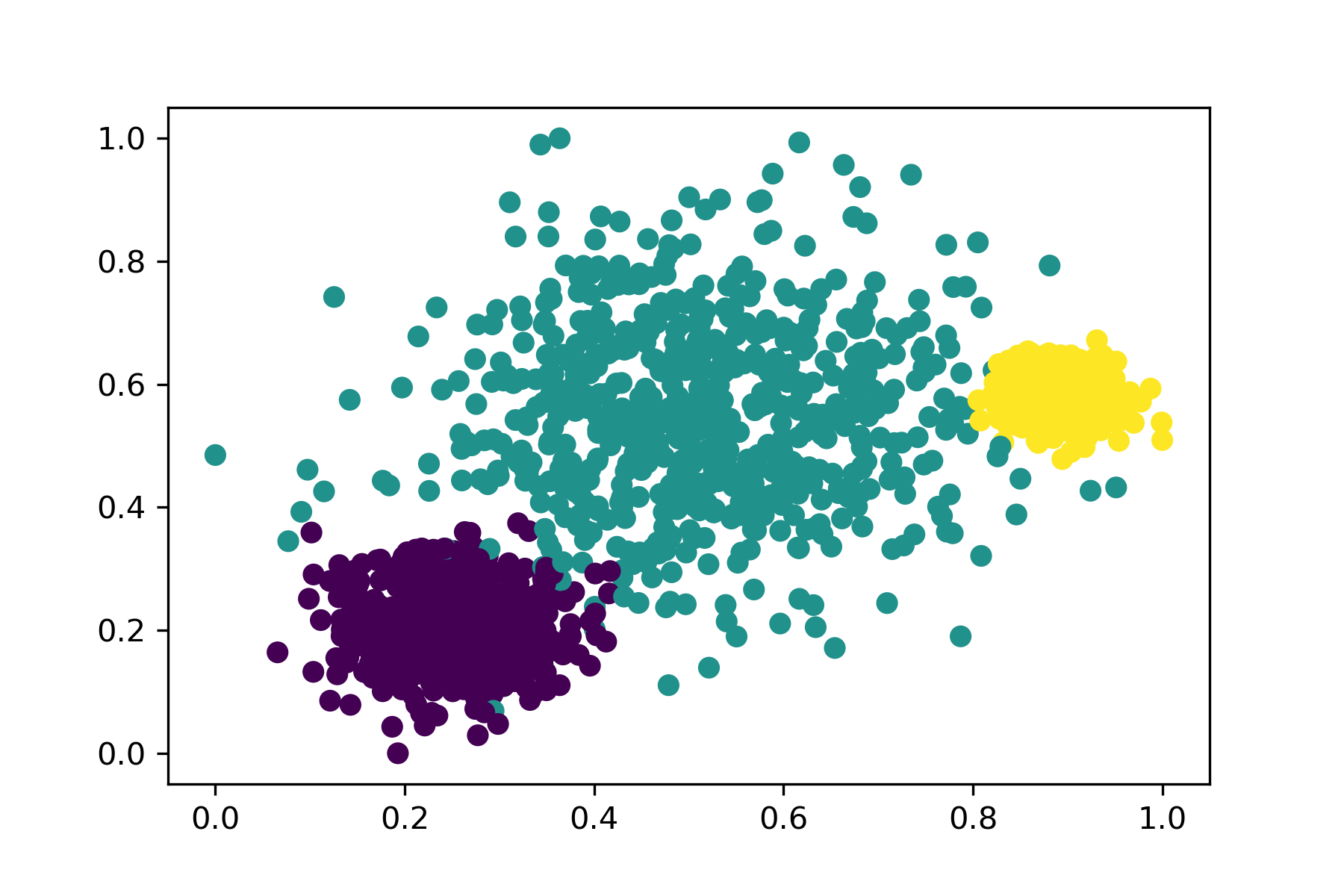}}
			\label{fig::syn:3Clusters}
	\end{minipage}}
	\subfigure[{\tt Anisotropic}]{
		\begin{minipage}{0.32\columnwidth}
			\centering
			\centerline{\includegraphics[width=\columnwidth]{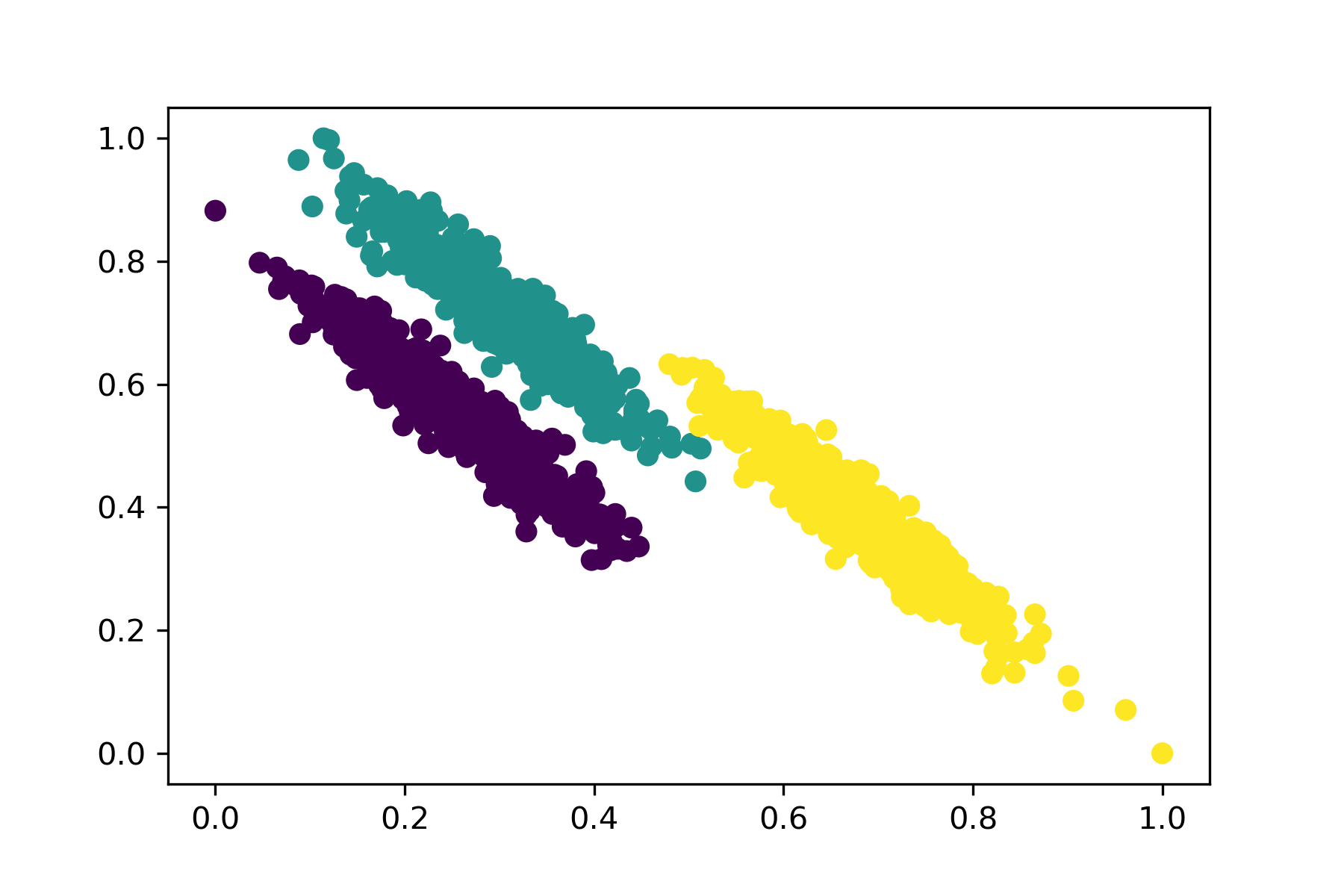}}
			\label{fig::syn:Anisotropic}
	\end{minipage}}
	\subfigure[{\tt Blobs}]{
		\begin{minipage}{0.32\columnwidth}
			\centering
			\centerline{\includegraphics[width=\columnwidth]{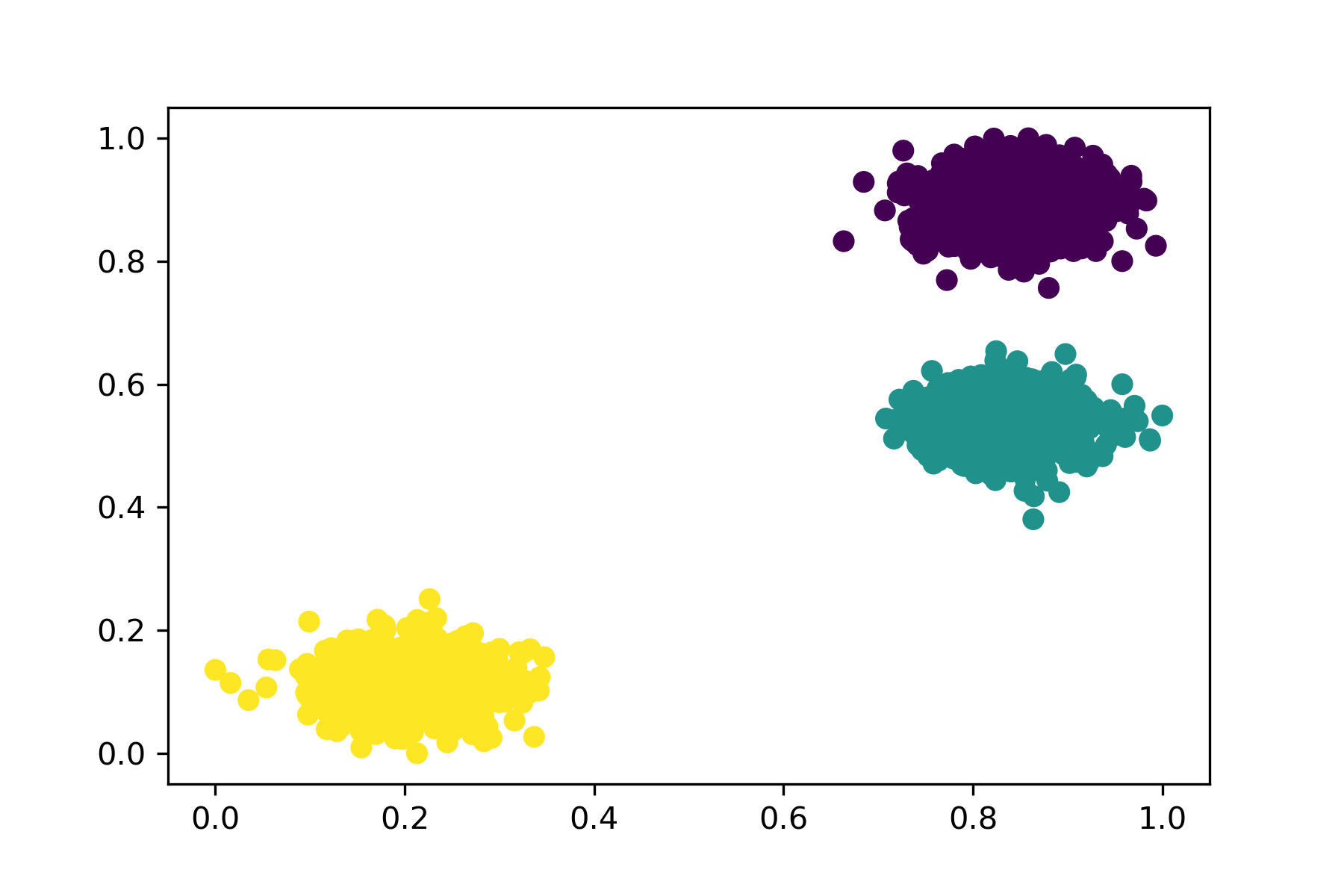}}
			\label{fig::syn:Blobs}
	\end{minipage}} \\ 
	\vspace{-15pt}
	\subfigure[{\tt Circles}]{
		\begin{minipage}{0.32\columnwidth}
			\centering
			\centerline{\includegraphics[width=\columnwidth]{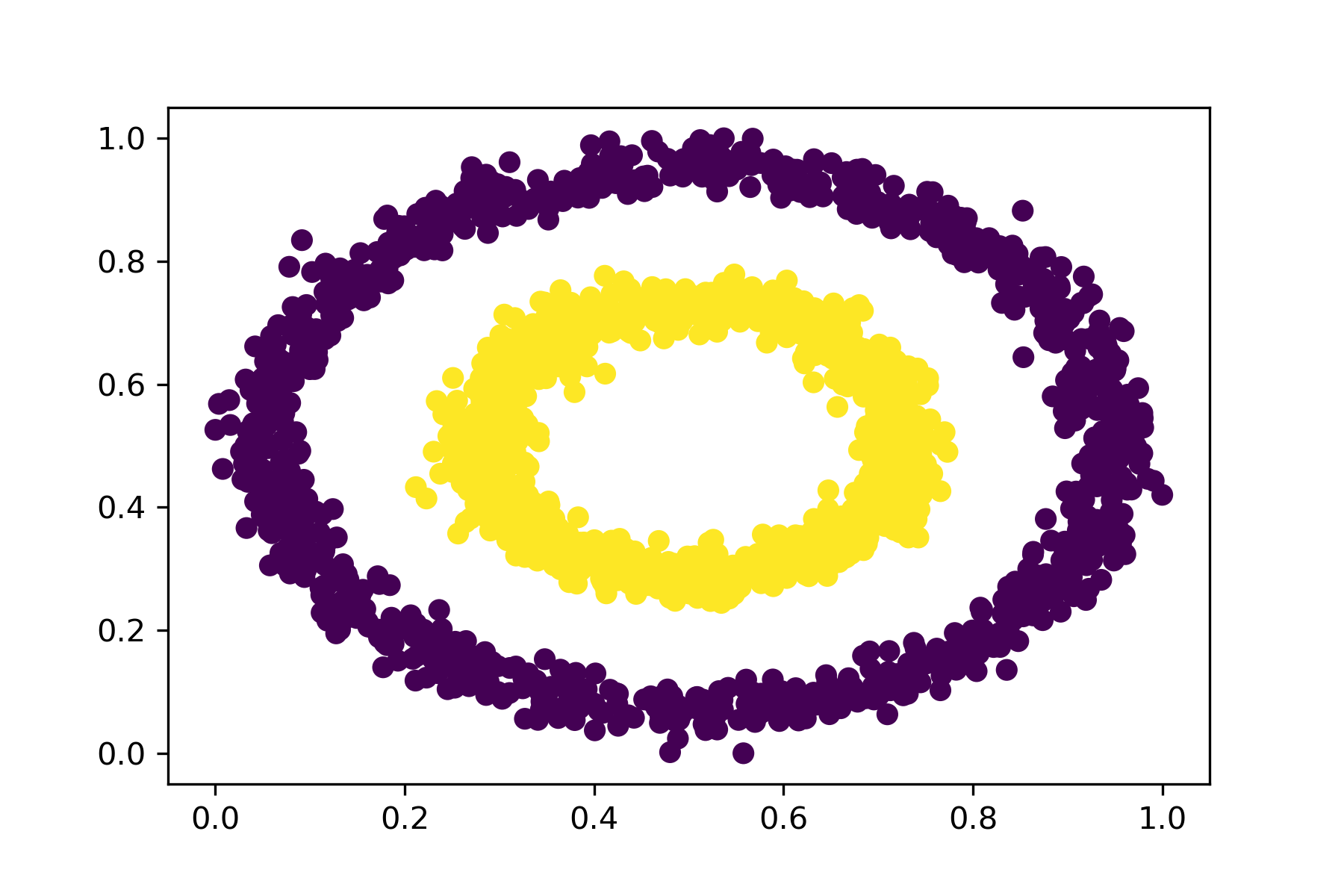}}
			\label{fig::syn:Circles}
	\end{minipage}}
	\subfigure[{\tt MDCGen}]{
		\begin{minipage}{0.32\columnwidth}
			\centering
			\centerline{\includegraphics[width=\columnwidth]{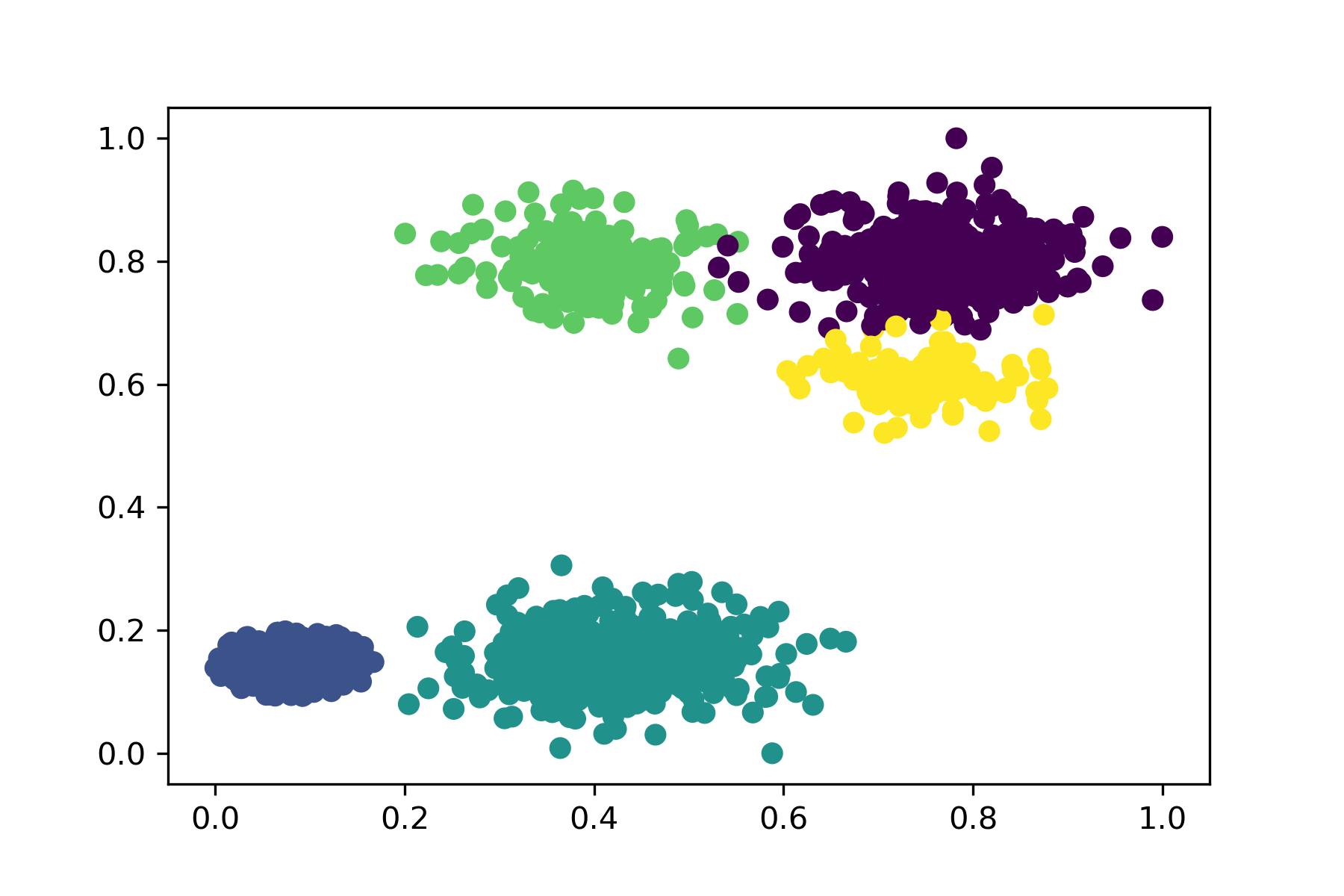}}
			\label{fig::syn:MDCGen}
	\end{minipage}} 
	\subfigure[{\tt Moons}]{
		\begin{minipage}{0.32\columnwidth}
			\centering
			\centerline{\includegraphics[width=\columnwidth]{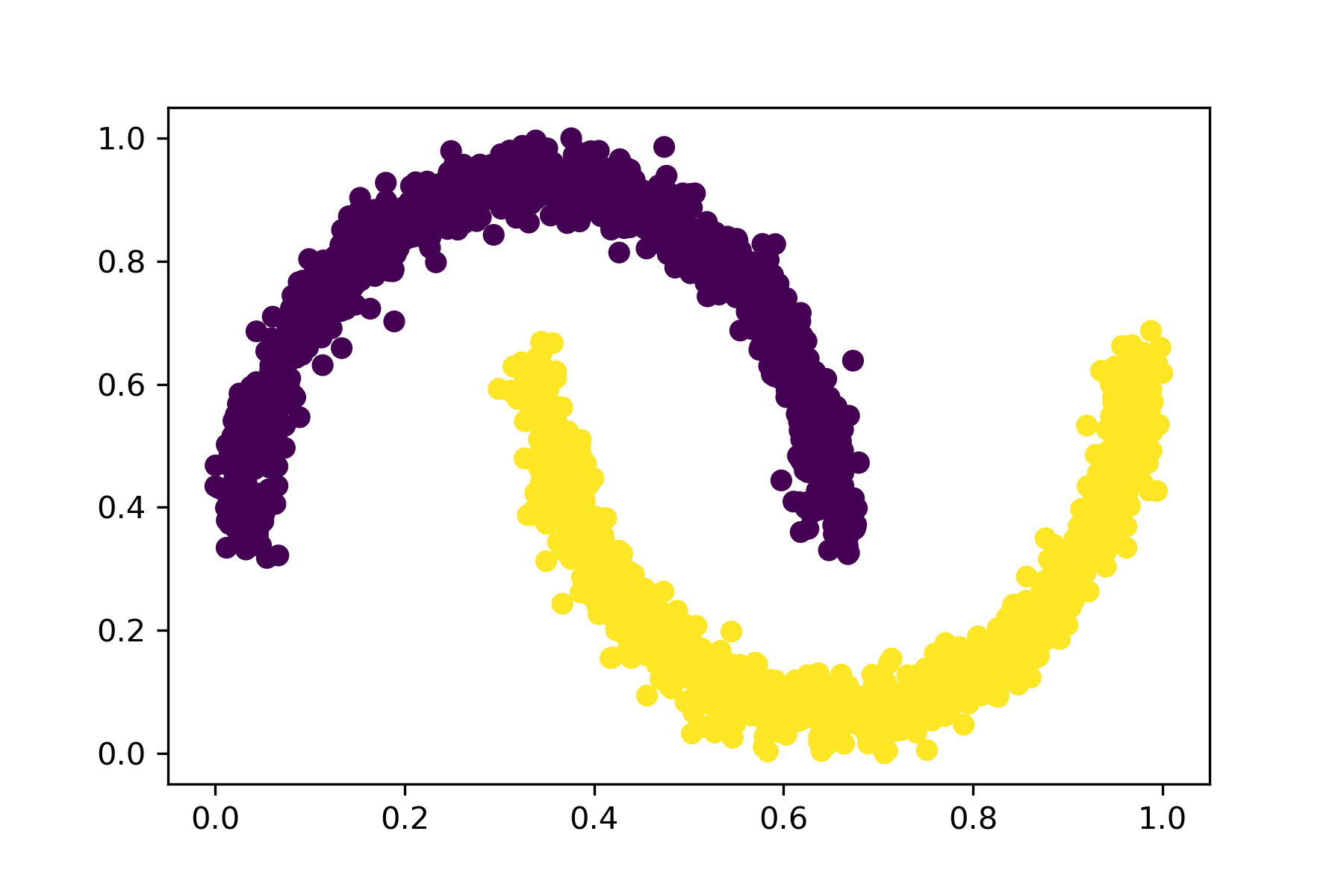}}
			\label{fig::syn:Moons}
	\end{minipage}} \\
	\vspace{-10pt}
	\centering
	\captionsetup{justification=centering}
	\caption{Visualization of the synthetic datasets.}
	\label{fig::syn_datasets}
\end{figure}

\subsubsection{Parameter analysis of hyper-parameters $k_D$ and $k_L$}

Firstly, we fix the number of nearest neighbors for graph connection $k_G=17$ and the level-set threshold $\lambda=0.75$ which are suitable hyper-parameters for a good clustering performance. We vary the number of nearest neighbors for hypothetical density estimation $k_D$ and for the PLLS $k_L$. The ARI scores and the number of clusters on {\tt 3Clusters} as the function of $(k_D, k_L)$ are visualized in Figure \ref{fig::param::kdkl}.
We find that the clustering performance is relatively insensitive to the parameters of $k_D$ and $k_L$: If $k_D$ and $k_L$ are not too small nor too large, the clustering performance is good, see the dark red filled region on the left side of Figure \ref{fig::param::kdkl}.
Moreover, the good clustering performance attributes to the performance of mode estimation. See the right side of Figure \ref{fig::param::kdkl}. A wide range of $k_D$ and $k_L$ can obtain the correct number of clusters (filled in green), which means that all the three modes are detected successfully.

\begin{figure}[htbp]
	\centering
	\vskip -0.05in
	\begin{minipage}{0.48\columnwidth}
		\centering
		\centerline{\includegraphics[width=\columnwidth]{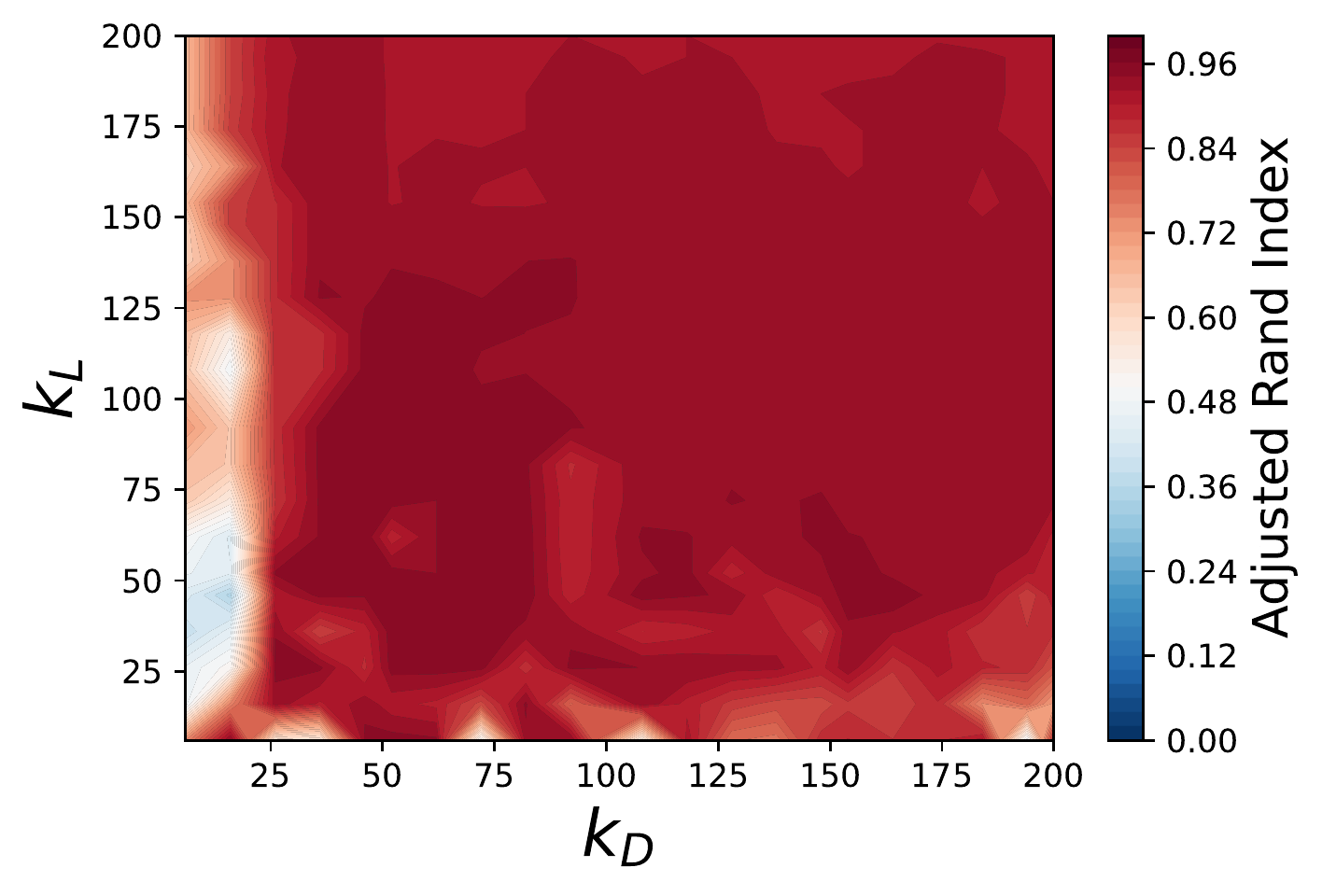}}
	\end{minipage}
	\begin{minipage}{0.48\columnwidth}
		\centering
		\centerline{\includegraphics[width=\columnwidth]{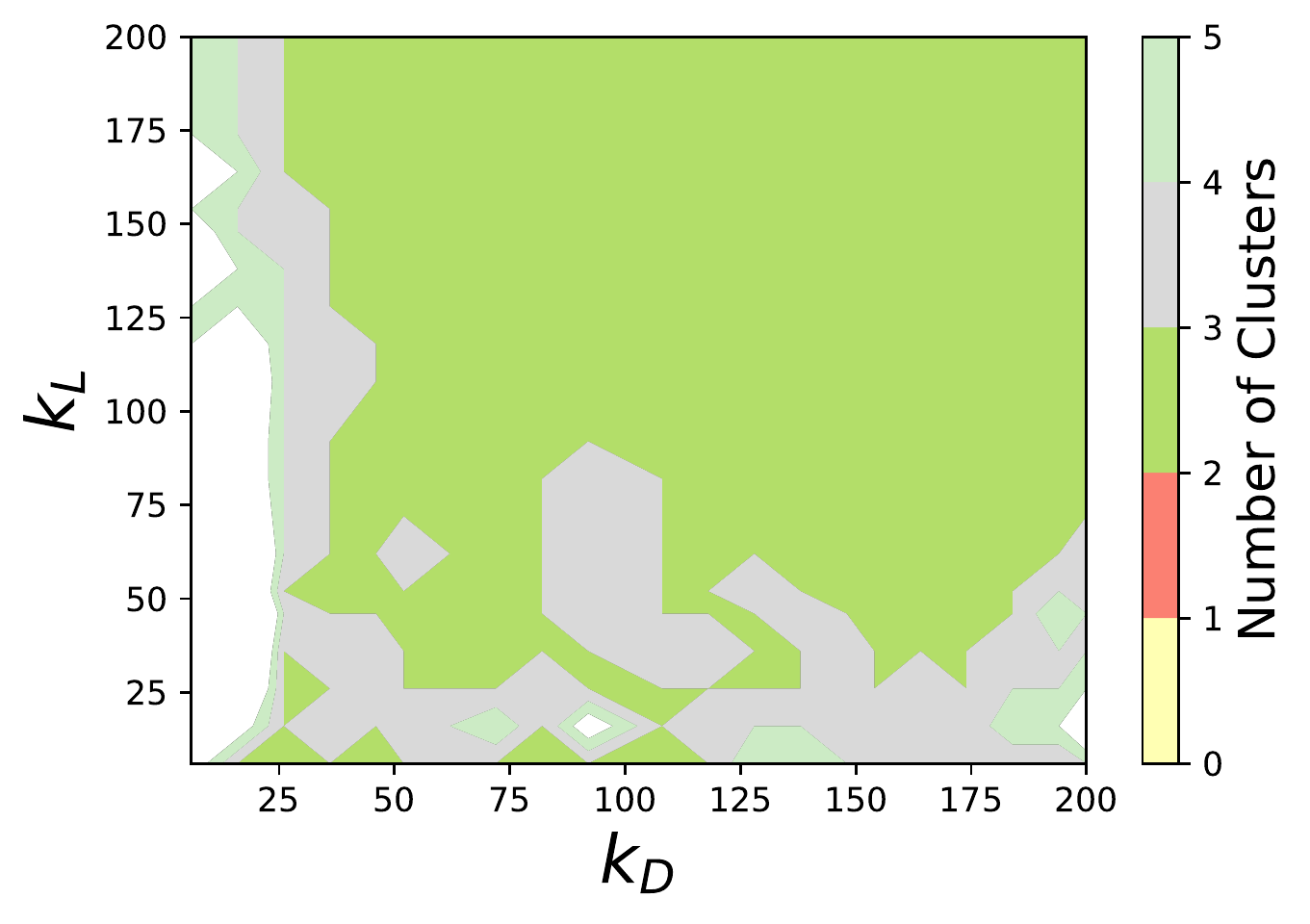}}
	\end{minipage}
	\vskip -10pt
	\caption{Visualization of {\tt 3Clusters} among with ARI scores and the number of clusters as $k_D$ and $k_L$ are changed and $k_G=18$, $\lambda=0.22$ are fixed. They shows that a wide range of $(k_D, k_L)$ obtain good clustering performance with correct mode estimation.}
	\label{fig::param::kdkl}
	\vskip -0.05in
\end{figure}

\subsubsection{Parameter analysis of hyper-parameters $k_G$ and $\lambda$}

Secondly, we fix the number of nearest neighbors for hypothetical density estimation and localized level-set $k_D=6$ and $k_L=108$, and explore the selection of hyper-parameters $k_G$ and $\lambda$.
In Figure \ref{fig::param::kglam}, we vary the $k_G$ and $\lambda$, and we visualize the ARI scores and the number of clusters on {\tt 3Clusters} dataset.
See the red filled region on the left figure which means good clustering performance, and we observe that there is a positive linear correlation between optimal $k_G$-s and optimal $\lambda$-s: we can achieve good clustering performance by selecting a pair of relatively small parameters $(k_G, \lambda)$ or a pair of relatively large parameters $(k_G, \lambda)$. This can guide the selection of these two hyper-parameters.
Similarly, the performance of mode estimation is also good for hyper-parameters with good clustering performance. (See the green-filled region on the right.)

\begin{figure}[htbp]
	\centering
	\vskip -10pt
	\begin{minipage}{0.48\columnwidth}
		\centering
		\centerline{\includegraphics[width=\columnwidth]{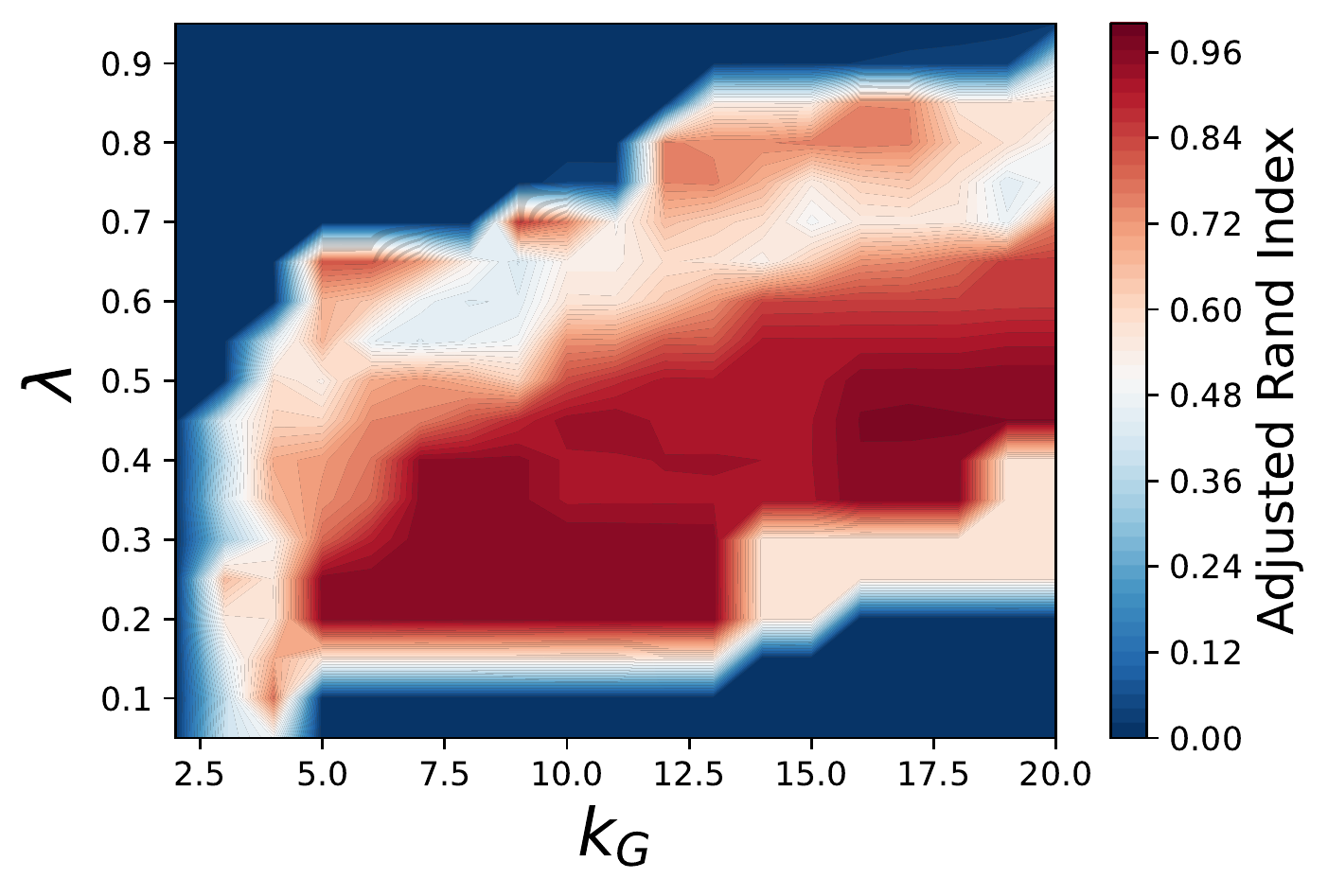}}
	\end{minipage}
	\begin{minipage}{0.48\columnwidth}
		\centering
		\centerline{\includegraphics[width=\columnwidth]{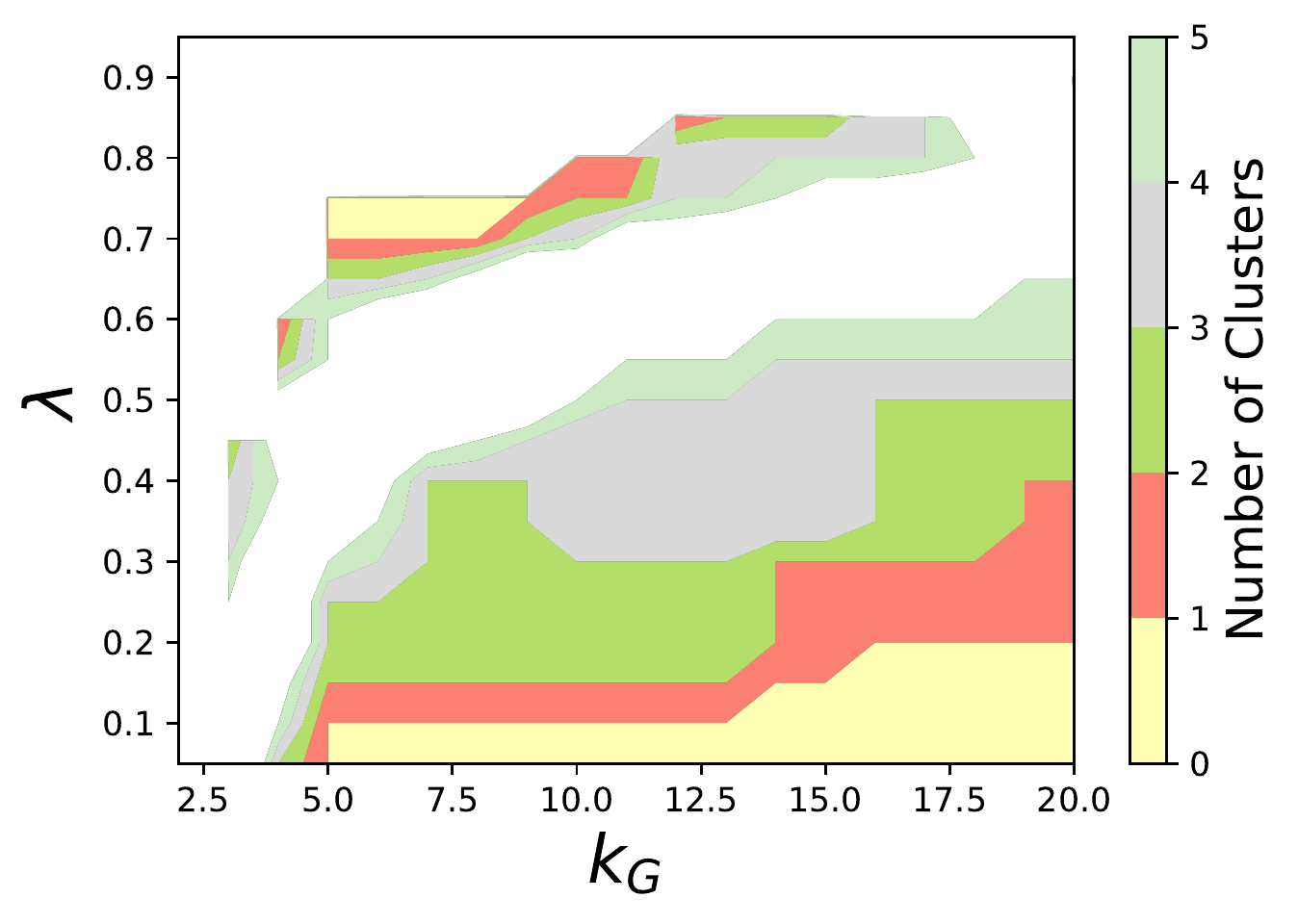}}
	\end{minipage}
	\vskip -10pt
	\caption{Visualization of {\tt 3Clusters} among with ARI scores and the number of clusters as $k_G$ and $\lambda$ are changed and $k_D=50$, $k_L=300$ are fixed. They show that a proper selection of $(k_D, k_L)$ obtains good clustering performance with correct mode estimation.}
	\label{fig::param::kglam}
\end{figure}

\subsubsection{The Effects of Bagging}

In this subsection, we list the optimal parameters for DMBC and BDMBC in various synthetic datasets in Table \ref{tab::parameters} to demonstrate the effects of bagging on parameter tuning, i.e., with a small sampling ratio, bagging can accelerate the algorithm by narrowing the range of parameter $k_D$.
In the experiments for synthetic datasets, we set the number of bagging iterations as $B=10$ and the sampling ratio $\rho=0.1$. 
As we can see from Table \ref{tab::parameters}, the optimal parameter of $k_D$ for BDMBC is much smaller than that for DMBC. Therefore, bagging enables BDMBC to have a more narrow searching grid of $k_D$ and prevents the algorithm from tedious parameter searching.
In addition, bagging with a relatively small $\rho$ can further speed up the algorithm by decreasing the number of training samples in each iteration. To be specific, bagging makes it possible to learn the distributional pattern of training datasets with only a small fraction of samples.
Meanwhile, bagging can also increase the randomness and boost the clustering performance. This is empirically verified in Table \ref{tab::realperformances}, where the clustering performances of BDMBC are significantly better than DMBC in many cases.

\begin{table}[htbp]
	\centering
	\captionsetup{justification=centering}
	\caption{The comparison of optimal parameters for DMBC and BDMBC in synthetic datasets}
	\label{tab::parameters}
	\begin{tabular}{ccccccc}
		\toprule
		Data & Bagging & $r$ & $k_D$ & $k_L$ & $\lambda$ & $k_G$ \\ \midrule
		\multirow[c]{2}{*}{\tt 3Clusters}  & No & - & 23 &  9 &  0.65  & 17 \\
		~ & Yes & 0.1 & 3 & 9 & 0.65 & 15 \\ \hline
		\multirow[c]{2}{*}{\tt Anisotropic} & No & - & 17 & 17 & 0.5 & 15 \\
		~ & Yes & 0.1 & 3 & 19 & 0.5 & 16 \\ \hline
		\multirow[c]{2}{*}{\tt Blobs} & No & - & 16 & 8 & 0.4 &  15 \\
		~ & Yes & 0.1 & 4 & 8 & 0.5 & 9 \\ \hline
		\multirow[c]{2}{*}{\tt Circles} & No & - & 17 & 14 & 0.45  &  19 \\
		~ & Yes & 0.1 & 7 & 13 & 0.55 & 15 \\ \hline
		\multirow[c]{2}{*}{\tt MDCGen}  & No & - & 17 & 26 & 0.6 & 19 \\
		~ & Yes & 0.1 & 9 & 28 & 0.4 & 9 \\ \hline
		\multirow[c]{2}{*}{\tt Moons} & No & - & 11 & 9 & 0.10  & 16 \\
		~ & Yes & 0.1 & 2 & 8 & 0.05 & 12 \\ \bottomrule
	\end{tabular}
\end{table}

\subsubsection{Practical suggestions for hyper-parameter selection}

In summary, we give some practical suggestions for the selection of hyper-parameters below:
\begin{itemize}
	\item Considering the computational cost, $B=10$ or $20$ is acceptable. As for $\rho$, an empirical rule is $\rho=0.1$ or $0.3$ to reduce the parameter grid.
	\item $k_D$ and $k_L$ are related to the hypothetical density estimation and the corresponding probability of localized level sets.  
	With a small $\rho$, $k_D$ can search in the range between 3 and 30; and the $k_L$ can search in the range between $3$ and $100$ (the range can be slightly increasing as the number of samples goes up). 
	\item $k_G$ and $\lambda$ are two hyper-parameters for level-set clustering. $k_G$ are quite stable for various dataset, and an empirical rule is $k_G \in \{5, 10, 15, 20\}$.
	A large $\lambda$ is more robust to noise samples and noisy density estimates in practice, so an empirical rule is to try various $\lambda$ from a relatively small $\lambda=0.20$ to a relatively large $\lambda=0.50$ or even larger (0.70 or 0.90).
\end{itemize}

\subsection{Scalability Experiments} \label{sec::subsec::scala}

In this subsection, we use a large-scale synthetic data named {\tt Artset}  \cite{iglesias2019mdcgen} to explore the clustering running times of the BDMBC algorithm. 
We fix the feature dimension $d = 10$ and the number of clusters $k = 10$, and change the sample size $n \in \{1\times {10}^{5}, 2\times {10}^{5}, 5\times {10}^{5}, 1\times {10}^{6}\}$.
Then we train BDMBC to compare the following two settings: the first is the bagging version with $B = 10$ and $\rho = 0.001$, and the second is the non-bagging version with $B = 1$ and $\rho = 1.0$.
For each setting, we select the optimal parameters including $k_D$, $k_L$, $\lambda$, $K_G$, calculate four clustering measures on behalf of the performances, and record the time consumptions of training the $k$-distance-based PLLS for each setting. The running time is measured in seconds.
As we can see from Table \ref{tab::time}, the clustering performance of the bagging version of the BDMBC algorithm with a very small sampling ratio $\rho$ is comparable with the non-bagging version. However, the time of training the PLLS for the bagging version of BDMBC can be ten times or even a hundred times less than that of the non-bagging version, which empirically verifies that bagging can improve the computational efficiency of BDMBC by training the $k$-distance-based PLLS with much fewer samples.

\begin{table}[htpb]
	\centering
	\caption{The comparison between the DMBC and BDMBC algorithms on the large-scale synthetic dataset {\tt Artset} with four different sample sizes. Four clustering measures and the time for training the $k$-distance-based PLLS are included.}
	\label{tab::time}
	\begin{tabular}{ccccccc}
		\toprule
		Sample Size & Bagging & ARI & NMI & F1 & Accuracy & Time (s)\\ \midrule
		\multirow{2}{*}{\tt $1\times {10}^{5}$} & No & $0.9910$ & $0.9800$ & $0.7009$ & $0.9886$ & $27.77$ \\
		~ & Yes & $0.9910$ & $0.9803$ & $0.7038$ & $0.9897$ & $0.41$ \\ \hline
		\multirow{2}{*}{\tt $2\times {10}^{5}$} & No & $0.9935$ & $0.9864$ & $0.5922$ & $0.9956$ & $85.30$ \\
		~ & Yes & $0.9936$ & $0.9865$ & $0.9922$ & $0.9956$ & $0.91$ \\ \hline
		\multirow{2}{*}{\tt $5\times {10}^{5}$} & No & $0.9875$ & $0.9737$ & $0.8862$ & $0.9805$ & $278.99$ \\ 
		~   & Yes & $0.9867$ & $0.9724$ & $0.8975$ & $0.9829$ & $3.94$ \\ \hline
		\multirow{2}{*}{\tt $1\times {10}^{6}$} & No & $0.9881$ & $0.9733$ & $0.8878$ & $0.9822$ & $1086.36$ \\ 
		~       & Yes & $0.9901$ & $0.9787$ & $0.9815$ & $0.9888$ & $13.15$ \\ \bottomrule
	\end{tabular}
\end{table}

\section{Proofs} \label{sec::proofs}

This section presents the proofs concerning the theoretical analysis.
We first present the convergence rate of the DMBC algorithm, i.e. the special case of BDMBC with $B = 1$ in Section \ref{sec::ratesdmbc}.
Section \ref{sec::baggedk} presents the proofs related to the $k$-distance and bagged $k$-distance in Section \ref{sec::analysisbagged}.
Section \ref{sec::proofmode} gives the proofs related to mode estimation in Section \ref{sec::ratesbdmbc}.
Section \ref{sec::prooflevelset} provides all proofs related to level set estimation for the proposed probability function PLLS in Section \ref{equ::ratelevelset}.

\subsection{Convergence Rates of DMBC for Mode Estimation} \label{sec::ratesdmbc}

To demonstrate the benefits of bagging in mode estimation, we consider the DMBC algorithm, which can be viewed as a special case of BDMBC in Algorithm \ref{alg::BDMBC} with $B = 1$ and $s = n$. More specifically, we only use $k$-distance for mode-based clustering without bagging. The procedure of DMBC can be described as follows. Firstly, we compute the empirical PLLS with respect to the $k$-distance by
\begin{align}\label{getah111}
	\widehat{p}_{k_L}^{k_D}(x)
	:= \frac{1}{k_L} \sum_{i=1}^{k_L} \eins \bigl\{ R_{k_D}(X_{(i)}(x)) \geq R_{k_D}(x) \bigr\}.
\end{align}
Then we construct the subgraph $G_k(\lambda)$ retaining the core-samples by
\begin{align}\label{equ::dklambda}
	\widehat{D}_k(\lambda)
	= \{ X_i \in D : \widehat{p}_{k_L}^{k_D}(X_i) \geq \lambda \}
\end{align}
and the mode set with respect to the $k$-distance by
\begin{align}\label{equ::modesestimator1}
	\widehat{\mathcal{M}}^k
	= \{ X_i \in D : \widehat{p}_{k_L}^{k_D}(X_i) = 1 \}.
\end{align}
Finally, we compute the cluster estimators ${\mathcal{C}}_k(\lambda)$, i.e., the connected components of $G_k(\lambda)$.

The next theorem presents the convergence rates of DMBC,~i.e., $k$-distance for multi-modal distribution under the above mild assumptions.

\begin{theorem}\label{thm::modesingle}
	Let Assumptions \ref{ass::cluster}, \ref{ass::modes} and \ref{ass::flatness1} hold with $2\alpha \gamma \leq 4+d$ and $\widehat{\mathcal{M}}^k$ be the mode estimator as in \eqref{equ::modesestimator1}. Then for every mode $m_i \in \mathcal{M}$ and $\lambda \geq c$ with the constant $c$ specified in the proof, by choosing 
	\begin{align*}
		k_{D,n} := n^{\frac{d}{4+d}} (\log n)^{\frac{d}{4+d}},
		\qquad
		k_{G,n}\asymp \log n,
		\qquad
		k_{L,n} \gtrsim n^{1-\frac{\alpha\gamma}{4+d}}(\log n)^{1+\frac{\alpha\gamma}{4+d}}, 
	\end{align*}
	there exists a mode estimate $\widehat{m}_i$ such that with probability $P^n$ at least $1 - 2/n^2$, there holds 
	\begin{align*}
		\|\widehat{m}_i - m_i\|_2
		\lesssim (\log n/n)^{\frac{1}{4+d}}.
	\end{align*}
	Moreover, there exist distinct cluster estimators $\widehat{C}_i\in \mathcal{C}_k(\lambda)$, $1\leq i\leq k$, such that $\widehat{m}_i\in \widehat{C}_i$. 
\end{theorem}

Theorem \ref{thm::mode} shows that if $k_D$ and $k_L$ are chosen properly, then the convergence rate of DMBC matches the lower bound established in \cite{tsybakov1990recursive} up to a logarithmic factor. 
Therefore, Theorem \ref{thm::mode} coincides with the optimal recovery for multiple modes established in \cite{dasgupta2014optimal, jiang2017consistency, jiang2018quickshift++}.
Finally, we mention that the mode estimation returned by \eqref{equ::modesestimator} corresponds to the true modes of $f$ in a subjective manner.

\subsection{Proofs Related to Section \ref{sec::analysisbagged}}
\label{sec::baggedk}

In this section, we present the proofs related to the bagged $k$-distance. To be specific, in Sections \ref{sec::weightB}-\ref{sec::weightedrho}, we provide the proofs related to the bagging error, estimation error, and approximation error for the hypothetical density estimation in Sections \ref{sec::BaggingError}-\ref{sec::ApproximationError}, respectively. With these preparations, in Section \ref{sec::RatesPseudoDEproof}, we provide proofs related to Section \ref{sec::RatesPseudoDE}, we first establish convergence rates for hypothetical density estimation. Then we propose an important lemma related to Taylor's expansion of the density function around the modes, which supplies the key to proofs of the mode estimation and mode-based clustering. Finally, we derive faster convergence rates of the hypothetical density estimation around the modes using this lemma. These theoretical results play a fundamental role in the proof of mode estimation and level set estimation for BDMBC and DMBC in Sections \ref{sec::proofmode} and \ref{sec::prooflevelset}.

Before we proceed, we list the well-known Bernstein's inequality that will be used frequently in the proofs. Lemma \ref{lem::bernstein} was introduced in \cite{bernstein1946theory} and can be found in many statistical learning textbooks, see e.g., \cite{massart2007concentration, cucker2007learning, steinwart2008support}.

\begin{lemma}[Bernstein's inequality] \label{lem::bernstein}
	Let $B>0$ and $\sigma>0$ be real numbers, and $n\geq 1$ be an integer. Furthermore, let $\xi_1,\ldots,\xi_n$ be independent random variables satisfying $\mathbb{E}_{P}\xi_i=0$, $\|\xi_i\|_\infty\leq B$, and $\mathbb{E}_{P}\xi_{\color{blue} i}^2\leq \sigma^2$ for all $i=1,\ldots,n$. Then for all $\tau>0$, we have
	\begin{align*}
		P\biggl(\frac{1}{n}\sum_{i=1}^n\xi_i\geq \sqrt{\frac{2\sigma^2\tau}{n}}+\frac{2B\tau}{3n}\biggr)\leq e^{-\tau}.
	\end{align*}
\end{lemma}

\subsubsection{Proofs Related to Section \ref{sec::BaggingError}}
\label{sec::weightB}

To prove Proposition \ref{prop::weightB}, we need to bound the number of reorderings of the data. To be specific, for fixed $x\in \mathbb{R}^d$, we reorder samples, $X_1,\ldots,X_n$, according to increasing values of $\|X_i-x\|$ with breaking ties by considering indices, i.e., $\|X_{\sigma_1}-x\|\leq \cdots \leq \|X_{\sigma_n}-x\|$, where $(\sigma_1, \ldots, \sigma_n)$ is a permutation of $(1, \ldots, n)$.  Then we define the inverse of the permutation, namely the rank $\Sigma_i$ by ${ \Sigma_i := \{ 1 \leq \ell \leq n : X_{\sigma_{\ell}} = X_i \}}$. Since we break ties by  considering indices, the rank $\Sigma_i$ is unique for all $1\leq i\leq n$. Therefore, the rank vector $(\Sigma_1,\ldots,\Sigma_n)$ for $x\in \mathbb{R}^d$ is well-defined.
Let $\mathcal{S} = \{(\Sigma_1,\ldots,\Sigma_n),x\in \mathbb{R}^d\}$ be the set of all rank vectors one can observe by moving $x$ around in space and we use the notation $|\mathcal{S}|$ to represent the cardinality of $\mathcal{S}$.

The next lemma, which plays a crucial role to derive the uniform bound for the proof of Propositions \ref{prop::weightB}, provides the upper bound for the number of reorderings, see also Lemma 20 in \cite{hang2022under}.

\begin{lemma}\label{lem::numberreorder}
	For any $d\geq 1$  and all $n \geq 2d$, there holds $|\mathcal{S}| \leq (25/d)^dn^{2d}$.
\end{lemma}

To further our analysis, we first need to recall the definitions of \textit{VC dimension} and \textit{covering number}, which are frequently used in capacity-involved arguments and measure the complexity of the underlying function class \cite{vandervaart1996weak,Kosorok2008introduction,gine2021mathematical}.

\begin{definition}[VC dimension] \label{def::VC dimension}
	Let $\mathcal{B}$ be a class of subsets of $\mathcal{X}$ and $A \subset \mathcal{X}$ be a finite set. The trace of $\mathcal{B}$ on $A$ is defined by $\{ B \cap A : B \subset \mathcal{B}\}$. Its cardinality is denoted by $\Delta^{\mathcal{B}}(A)$. We say that $\mathcal{B}$ shatters $A$ if $\Delta^{\mathcal{B}}(A) = 2^{\#(A)}$, that is, if for every $A' \subset A$, there exists a $B \subset \mathcal{B}$ such that $A' = B \cap A$. For $n \in \mathrm{N}$, let
	\begin{align}\label{equ::VC dimension}
		m^{\mathcal{B}}(n) := \sup_{A \subset \mathcal{X}, \, \#(A) = n} \Delta^{\mathcal{B}}(A).
	\end{align}
	Then, the set $\mathcal{B}$ is a Vapnik-Chervonenkis class if there exists $n<\infty$ such that $m^{\mathcal{B}}(n) < 2^n$ and the minimal of such $n$ is called the VC dimension of $\mathcal{B}$, and abbreviate as $\mathrm{VC}(\mathcal{B})$.
\end{definition}

Since an arbitrary set of $n$ points $\{x_1,\ldots,x_n\}$ possess $2^n$ subsets, we say that $\mathcal{B}$ \textit{picks out} a certain subset from $\{ x_1, \ldots, x_n\}$ if this can be formed as a set of the form $B\cap \{x_1,\ldots,x_n\}$ for a $B\in \mathcal{B}$. The collection $\mathcal{B}$ \textit{shatters} $\{x_1,\ldots,x_n\}$ if each of its $2^n$ subsets can be picked out in this manner. 
From Definition \ref{def::VC dimension} we see that the VC dimension of the class $\mathcal{B}$ is the smallest $n$ for which no set of size $n$ is shattered by $\mathcal{B}$,  that is,
\begin{align*}
	\mathrm{VC}(\mathcal{B}) =\inf \Bigl\{n:\max_{x_1,\ldots,x_n} \Delta^{\mathcal{B}}(\{ x_1,\ldots,x_n \})\leq 2^n\Bigr\},
\end{align*}
where $\Delta^{\mathcal{B}}(\{ x_1, \ldots,x_n \})=\#\{B\cap \{x_1,\ldots,x_n\}:B\in \mathcal{B}\}$.
Clearly, the more refined $\mathcal{B}$ is, the larger its index.
Let us recall the definition of the covering number in \cite{vandervaart1996weak}.

\begin{definition}[Covering Number] \label{def::CovNum}
	Let $(\mathcal{X}, d)$ be a metric space and $A \subset \mathcal{X}$. For $\varepsilon>0$, the $\varepsilon$-covering number of $A$ is denoted as 
	\begin{align*}
		\mathcal{N}(A, d, \varepsilon) 
		:= \min \biggl\{ n \geq 1 : \exists x_1, \ldots, x_n \in \mathcal{X} \text{ such that } A \subset \bigcup^n_{i=1} B(x_i, \varepsilon) \biggr\},
	\end{align*}
	where $B(x, \varepsilon) := \{ x' \in \mathcal{X} : d(x, x') \leq \varepsilon \}$.
\end{definition}

The following Lemma, which is needed in the proof of Lemma \ref{lem::Rrho}, provides the covering number of the indicator functions on the collection of balls in $\mathbb{R}^d$, see also Lemma 25 in \cite{hang2022under}. 

\begin{lemma}\label{lem::CoveringNumber}
	Let $\mathcal{B} := \{ B(x, r) : x \in \mathbb{R}^d, r > 0 \}$ and $\eins_{\mathcal{B}} := \{ \eins_B : B \in \mathcal{B} \}$. Then for any $\varepsilon \in (0, 1)$, there exists a universal constant $C$ such that 
	\begin{align*}
		\mathcal{N}(\eins_{\mathcal{B}}, \|\cdot\|_{L_1(Q)}, \varepsilon)
		\leq C (d+2) (4e)^{d+2} \varepsilon^{-(d+1)}
	\end{align*}
	holds for any probability measure $Q$.
\end{lemma}

To prove Proposition \ref{prop::weightB}, we need the following lemma, which provides the uniform bound on the distance between any point and its $k$-th nearest neighbor with high probability.

\begin{lemma}\label{lem::Rrho}
	Let Assumption \ref{ass::cluster} hold. Let 
	$R_{k}(x) := \|X_{(k)}(x) - x\|$ be the distance from $x$ to its $k$-th nearest neighbor and $\overline{R}_k(x)$ be the population version defined by \eqref{equ::overlinerx} for $1 \leq k \leq n$. Then for all $x\in \mathcal{X}$, if $k \geq 32(d+4)\log n$, there holds
	\begin{align}\label{equ::infrksuprk}
		R_k(x)\asymp (k/n)^{1/d}
	\end{align}
	with probability $P^n$ at least $1-2/n^2$. Moreover, we have
	\begin{align}\label{equ::rkxoverrkx}
		|\overline{R}_k^d(x)-R_k^d(x)|\lesssim \sqrt{k \log n}/n.
	\end{align}
\end{lemma}

\begin{proof}[Proof of Lemma \ref{lem::Rrho}]
	For $x \in \mathcal{X}$ and $q \in [0,1]$, we define the $q$-quantile diameter 
	\begin{align*}
		\rho_x(q) := \inf \bigl\{ r : P(B(x, r)) \geq q \bigr\}.
	\end{align*}
	
	Let us consider the set $\mathcal{B}_k^- :=  \big\{ B \bigl( x, \rho_x \bigl( (k - 2\sqrt{\tau k})/n \bigr) \bigr) : x \in \mathcal{X} \bigr\} \subset \mathcal{B}$. Lemma \ref{lem::CoveringNumber} implies that for any probability $\mathrm{Q}$, there holds
	\begin{align} \label{Bk-CoveringNumber}
		\mathcal{N}(\eins_{\mathcal{B}_k^-}, \|\cdot\|_{L_1(\mathrm{Q})}, \varepsilon)
		\leq \mathcal{N}(\eins_{\mathcal{B}}, \|\cdot\|_{L_1(\mathrm{Q})}, \varepsilon)
		\leq C (d+2) (4e)^{d+2} \varepsilon^{-(d+1)}.
	\end{align}
	By the definition of the covering number, there exists an $\varepsilon$-net $\{A_j^-\}_{j=1}^J \subset \mathcal{B}_k^-$ with $J := \lfloor C (d+2) (4e)^{d+2} \varepsilon^{-(d+1)} \rfloor$ and for any $x \in \mathcal{X}$, there exists some $j \in \{ 1, \ldots, J \}$ such that 
	\begin{align} \label{eq::approxAj-1}
		\bigl\| \eins \bigl\{ B \bigl( x, \rho_x \bigl( (k -2\sqrt{\tau k})/n \bigr) \bigr) \bigr\} - \eins_{A_j^-} \bigr\|_{L_1(D)} 
		\leq \varepsilon.
	\end{align}
	For any $i = 1, \ldots, n$, let the random variables $\xi_i$ be defined by $\xi_i = \eins_{A_j^-}(X_i) - (k - 2\sqrt{\tau k})/n$. Then we have $\mathbb{E}_{P}\xi_i = 0$, $\|\xi_i\|_{\infty} \leq 1$, and $\mathbb{E}_{P}\xi_i^2 \leq \mathbb{E}_{P}\xi_i = (k - 2\sqrt{ \tau k})/n$. Applying Bernstein's inequality in Lemma \ref{lem::bernstein}, we obtain 
	\begin{align*}
		\frac{1}{n} \sum_{i=1}^n \eins_{A_j^- }(X_i) - (k - 2\sqrt{\tau k})/n
		\geq - \sqrt{2 \tau (k - 2\sqrt{\tau k})} / n - 2 \tau / (3n)
	\end{align*}
	with probability $P^n$ at least $1-e^{-\tau}$. Then the union bound together with the covering number estimate \eqref{Bk-CoveringNumber} implies that for any $A_j^-$, $j = 1, \cdots, J$, there holds 
	\begin{align*}
		& \frac{1}{n} \sum_{i=1}^n \eins_{A_j^-}(X_i) - (k - 2\sqrt{(\tau + \log J) k}) / n
		\\
		& \geq - \sqrt{2 (\tau + \log J) \bigl( k - 2\sqrt{(\tau + \log J) k} \bigr)} / n 
		- 2 (\tau + \log J) / (3 n).
	\end{align*}
	This together with \eqref{Bk-CoveringNumber} yields that for all $x\in \mathcal{X}$, there holds
	\begin{align*}
		&\frac{1}{n} \sum_{i=1}^n \eins\{X_i\in \rho_x\big((k-2\sqrt{\tau k}/n)\big)\} - (k - 2\sqrt{(\tau + \log J) k}) / n\\
		& \geq - \sqrt{2 (\tau + \log J) \bigl( k - 2\sqrt{(\tau + \log J) k} \bigr)} / n 
		- 2 (\tau + \log J) / (3 n)-\varepsilon.
	\end{align*}
	Now, if we take $\varepsilon = 1/n$, then for any $n >  (4 e) \vee (d + 2) \vee C$, there holds $\log J =  \log C + \log(d+2) +  (d+2)\log(4e) + (d+1) \log n \leq (2d+5)\log n$. 
	Let $\tau := 3 \log n$. A simple calculation yields that if $k \geq 32(d+4)\log n$, then we have
	\begin{align*}
		\sqrt{2 (\tau + \log J) \bigl( k - 2\sqrt{(\tau + \log J) k} \bigr)} / n \leq \sqrt{3(\tau+\log J)k}/n.
	\end{align*}
	Consequently, for all $n > (4 e) \vee (d + 2) \vee C$, there holds
	\begin{align*}
		\sqrt{2 (\tau + \log J) \bigl( k - 2\sqrt{ (\tau + \log J) k} \bigr)} / n 
		+ 2 (\tau + \log J) / (3 n)
		+ 1/n
		\leq 2\sqrt{(\tau+\log J) k} / n.
	\end{align*}
	Therefore, for all $x \in \mathcal{X}$, there holds $\frac{1}{n} \sum_{i=1}^n \eins \bigl\{ B \bigl( x, \rho_x \bigl( (k - 2\sqrt{\tau k})/n \bigr) \bigr) \bigr\} (X_i) \geq k/n$ with probability $P^n$ at least $1-1/n^3$. By the definition of $R_{k}(x)$, there holds 
	\begin{align}\label{equ::Rrho1-}
		R_{k}(x) \geq \rho_x \bigl( (k - 2\sqrt{\tau k})/n \bigr)
	\end{align}
	with probability $P^n$ at least $1-1/n^3$. 
	For any $x \in \mathcal{X}$, we have $P \bigl( B \bigl( x, \rho_x \bigl( (k - 2\sqrt{\tau k})/n \bigr) \bigr) \bigr) = (k - 2\sqrt{\tau k})/n$. By Assumption \ref{ass::cluster}, we have
	\begin{align*}
		P \bigl( B \bigl( x, \rho_x \bigl( (k - 2\sqrt{ \tau k})/n \bigr) \bigr) \bigr)
		= (k - 2\sqrt{\tau k})/n
		\leq V_d\overline{c}\rho_x^d \bigl( (k - 2\sqrt{\tau k})/n \bigr),
	\end{align*}
	which yields 
	\begin{align}\label{eq::condminus}
		\rho_x \bigl( (k - 2\sqrt{\tau k})/n \bigr) 
		\geq \bigl( (k - 2\sqrt{\tau k})/(V_d\overline{c}n) \bigr)^{1/d}
		\geq \bigl( (k /(4V_d\overline{c} n) \bigr)^{1/d}.
	\end{align}
	Combining \eqref{equ::Rrho1-} with \eqref{eq::condminus}, we obtain that $R_{k}(x) \geq \bigl( k /(4V_d\overline{c} n) \bigr)^{1/d}$ holds for all $x \in \mathcal{X}$ with probability $P^n$ at least $1-1/n^3$. Therefore, a union bound argument yields that for all $x \in \mathcal{X}$, all $k\geq 32(d+4)\log n$, and all sufficiently large $n$, there holds
	\begin{align}\label{equ::rkxlower}
		R_k(x)	
		\geq \rho_x((k-2\sqrt{\tau k})/n)
		\geq \bigl( k /(4V_d\overline{c} n) \bigr)^{1/d}
	\end{align}
	with probability $P^n$ at least $1 - 1/n^2$. 
	This proves the first inequality of \eqref{equ::infrksuprk}.
	
	On the other hand, let us consider the set $\mathcal{B}_k^+ :=  \big\{ B \bigl( x, \rho_x \bigl( (k + 2\sqrt{ \tau k})/n \bigr) \bigr) : x \in \mathcal{X} \bigr\} \subset \mathcal{B}$. 
	Similar to the proof of \eqref{equ::infrksuprk}, we can show that for all sufficiently large $n$, there holds
	\begin{align}\label{equ::rkxupper}
		R_{k}(x) 
		\leq \rho_x \bigl( (k + 2\sqrt{ \tau k})/n \bigr) 
		\leq \bigl( (k + 4 \sqrt{k \log n})/ (\underline{c} n) \bigr)^{1/d} 
		\leq \bigl( 2 k / (\underline{c} n) \bigr)^{1/d}
	\end{align}
	with probability $P^n$ at least $1-1/n^2$.
	
	Finally, combining \eqref{equ::rkxlower} and \eqref{equ::rkxupper}, we get
	\begin{align*}
		\rho_x((k-2\sqrt{\tau k})/n)
		\leq \rho_x(k/n)
		= \overline{R}_k(x)
		\leq \rho_x((k+2\sqrt{\tau k})/n)
	\end{align*}
	and consequently for all $k \geq 32(d+4)\log n$, there holds
	\begin{align*}
		|P(B(x,\overline{R}_k(x)))-P(B(x,R_k(x)))|
		\leq 2 \sqrt{3 k \log n} / n.
	\end{align*}
	Therefore, by Assumption \ref{ass::cluster} with the condition $\mathcal{X}:=[0,1]^d$, we have that for all $x\in \mathcal{X}$,
	\begin{align*}
		|\overline{R}_k^d(x)-R_k^d(x)|
		\leq 2^d|P(B(x,\overline{R}_k(x)))-P(B(x,R_k(x)))|/\underline{c} 
		\leq 2^{d+1}\sqrt{3k \log n}/(n\underline{c}),
	\end{align*}
	which proves \eqref{equ::rkxoverrkx}. 
	This completes the proof of Lemma \ref{lem::Rrho}.
\end{proof}

The following Lemma is needed in the proof of Proposition \ref{prop::weightB}.

\begin{lemma}\label{lem::ribound}
	Let $p_i$ be the probability as in \eqref{equ::def}.
	Then we have
	\begin{align}\label{equ::piibetabound}
		\sum_{i=1}^n p_i (i/n)^{\beta}
		\leq 2 (4k/s)^{\beta}, 
		\qquad \beta \in (0, 2] \cup \{ 3 \}.
	\end{align} 
	Moreover, if $d \geq 1$, then we have
	\begin{align}\label{eq::bothlower}
		\sum_{i=1}^n p_i (i/n)^{1/d} \geq  (k/s)^{1/d}/64.
	\end{align}
\end{lemma}

\begin{proof}[Proof of Lemma \ref{lem::ribound}]
	Using the substitution $z = i - k$, we get
	\begin{align}\label{equ::boundpibeta}
		\sum_{i=1}^n p_i i^{\beta} 
		= \sum_{z=0}^{n-s} p_{z+k} (k+z)^{\beta}.
	\end{align}
	Define the random variable $Z$ by $P(Z = z) = p_{z+k}$, $z = 0, \ldots, n - s$. 
	It is easy to verify that $Z$ follows the beta-binomial distribution with parameters $n-s$, $k$, and $s-k+1$ by \eqref{equ::def}. 
	The moments of $Z$ are $\mathbb{E} Z = (n - s) k / (s + 1)$, $\mathbb{E} Z^2 = k (n - s) (n - k + k n - k s + 1) / ((s + 1) (s + 2))$, and $\mathbb{E} Z^3 = k (n - s) [(n - s)^2 (k^2 + 3 k + 2) + 3 k (n - s) (s - k + 1) + (s - k + 1) (s - 2 k + 1)]/((s + 1) (s + 2) (s + 3))$.
	We refer the reader to \cite{johnson2005univariate} for more discussions on this distribution. 
	
	Let us first consider the case $\beta \in (0, 1]$. 
	Since $(k/(k+z))^{\beta} + (z/(z+d))^{\beta} \geq k/(k+z) + z/(k+z) = 1$, we have $(k+z)^{\beta} \leq z^{\beta} + k^{\beta}$.
	This together with \eqref{equ::boundpibeta} yields 
	\begin{align*}
		\sum_{i=1}^n p_i i^{\beta}
		\leq \sum_{z=0}^{n-s} p_{z+k} z^{\beta} + \sum_{z=0}^{n-s} p_{z+k} k^{\beta}
		= k^{\beta} + \sum_{z=0}^{n-s} p_{z+k} z^{\beta}.
	\end{align*}
	Since the function $g(x) := x^{\beta}$, $\beta \in (0, 1]$, is concave on $[0,\infty)$, by using Jensen's inequality, we get
	$\sum_{z=0}^{n-s} p_{z+k} z^{\beta}
	= \mathbb{E} Z^{\beta}
	\leq (\mathbb{E} Z)^{\beta}
	\leq (kn/s)^{\beta}$
	and consequently 
	\begin{align*}
		\sum_{i=1}^n p_i (i/n)^{\beta}
		\leq (k/n)^{\beta} + (k/s)^{\beta}
		\leq 2 (4k/s)^{\beta},
		\qquad  
		\beta \in (0, 1).
	\end{align*}
	
	Next, let us consider the case $\beta \in (1, 2)$ or equivalently $2 - \beta \in (0, 1)$. 
	Using H\"{o}lder's inequality, we get
	\begin{align*}
		\sum_{i=1}^n p_i i^{\beta}
		= \sum_{i=1}^n (p_i i)^{2-\beta} (p_i i^2)^{\beta-1}
		\leq \biggl( \sum_{i=1}^n p_i i\biggr)^{2-\beta} \cdot \biggl( \sum_{i=1}^n p_i i^2 \biggr)^{\beta-1}.
	\end{align*}	
	With the substitution $z = i - k$ we get $\sum_{i=1}^n p_i i = \sum_{z=0}^{n-s} p_{z+k} (z+k) = k + \mathbb{E} Z = k + (n - s) k /(s + 1) \leq k n / s$ and $\sum_{i=1}^n p_i i^2 = \sum_{z=0}^{n-s} p_{z+k} (z+k)^2 \leq 2 \sum_{z=0}^{n-s} p_{z+k} (z^2+k^2) \leq 2 k^2 + 2 \sum_{z=0}^{n-s} z^2 p_{z+k} = 2 k^2 + 2 \mathbb{E} Z^2 = 2 k^2 + k(n-s)(n-k+kn-ks+1)/((s+1)(s+2)) \leq 4 k^2 n^2 / s^2$.
	Consequently we obtain
	\begin{align*}
		\sum_{i=1}^n p_i i^{\beta}
		\leq (kn/s)^{2-\beta} (4k^2n^2/s^2)^{\beta-1}
		= 4^{\beta-1} (kn/s)^{\beta}
		\leq (4kn/s)^{\beta}.
	\end{align*}
	It is easy to see that this inequality also holds when $\beta = 2$. 
	Therefore, we have
	\begin{align*}
		\sum_{i=1}^n p_i (i/n)^{\beta}
		\leq (4k/s)^{\beta}
		< 2 (4k/s)^{\beta}, 
		\qquad  
		\beta \in (1, 2].
	\end{align*}
	
	Finally, for the case $\beta = 3$, we have 
	\begin{align*}
		\sum_{i=1}^n p_i i^3
		= \sum_{z=0}^{n-s} p_{z+k}(z+k)^3
		\leq \sum_{z=0}^{n-s} p_{z+k}
		\leq 4 \sum_{z=0}^{n-s} p_{z+k}(z^3 + k^3)
		= 4 k^3 + 4 \mathbb{E} Z^3
		\leq 64 (kn/s)^3
	\end{align*}
	and consequently
	$\sum_{i=1}^n p_i (i/n)^3\leq 2 (4k/s)^3$, which proves
	\eqref{equ::piibetabound}.
	
	Now we turn to the lower bound \eqref{eq::bothlower}. 
	Using H\"{o}lder's inequality, we get
	\begin{align*}
		\sum_{i=1}^n p_i i
		\leq \biggl( \sum_{i=1}^n p_i i^{1/d} \biggr)^{1/2} \biggl( \sum_{i=1}^n p_i i^{2-\frac{1}{d}} \biggr)^{1/2}
		\leq \biggl( \sum_{i=1}^n p_i i^{1/d} \biggr)^{1/2} \biggl( \sum_{i=1}^n p_i i \biggr)^{\frac{1}{2d}} \biggl( \sum_{i=1}^n p_i i^2 \biggr)^{\frac{d-1}{2d}}, 
	\end{align*}
	which leads to
	\begin{align*}
		\sum_{i=1}^n p_i i^{1/d} \geq \biggl( \sum_{i=1}^n p_i i \biggr)^{\frac{2d-1}{d}} \biggl( \sum_{i=1}^n p_i i^2 \biggr)^{-\frac{d-1}{d}}.
	\end{align*}
	With the substitution $z = i - k$ we get $\sum_{i=1}^n p_i i = k + \mathbb{E} Z =  k +  (n - s) k /(s + 1) \geq (k + k (n - s) / (s + 1)) / 2 = n k / (4s)$ and consequently 
	\begin{align*}
		\sum_{i=1}^n p_i i^{1/d} 
		\geq (n k / (4s))^{2-1/d} (2kn/s)^{-2+2/d}
		= 2^{-6+4/d}(nk/s)^{1/d}
		\geq (k/s)^{1/d}/64,
	\end{align*}
	which completes the proof.
\end{proof}

With the above results, we are in the position of deriving the bound for the bagging error. \\

\begin{proof}[Proof of Proposition \ref{prop::weightB}]
	By the definition of $R_k^B(x)$ and $\widetilde{R}_k^B(x)$, we have 
	\begin{align*}
		\bigl| R_k^B(x) - \widetilde{R}_k^B(x) \bigr|
		= \biggl| \frac{1}{B} \sum_{b=1}^B \sum_{i=1}^n p_i^b R_i(x) - \sum_{i=1}^n p_i R_i(x) \biggr|.
	\end{align*}
	For any $b = 1, \ldots, B$, define the random variables $\zeta_b(x)$ by $\zeta_b(x) := \sum_{i=1}^n (p_i^b-p_i) R_i(x)$.
	Then we have $\|\zeta_b\|_{\infty} \leq ( \sum_{i=1}^n p_i^b \bigvee \sum_{i=1}^n p_i) R_i(x) \leq R_n(x) \leq \mathrm{diam}(\mathcal{X})$.
	By the definition of $p_i$ in \eqref{equ::tilderbkx}, we have $\mathbb{E}_{P_Z}(\zeta_b(x) | D_n) = 0$ and $\mathrm{Var}(\zeta_b(x) | D_n) = \mathrm{Var} ( \sum_{i=1}^n (p_i^b - p_i) R_i(x) | D_n )$.
	For $1 \leq i < j \leq n$, we have $\mathrm{Cov} ( (p_i^b - p_i) R_i(x), (p_j^b - p_j) R_j(x) ) = R_i(x) R_j(x) \mathrm{Cov} (p_i^b - p_i, p_j^b - p_j) = R_i(x) R_j(x) \mathrm{Cov} (p_i^b, p_j^b)$.
	By the definition of $p_i^b$ and $p_j^b$, we have $p_i^b p_j^b=0$ and thus $\mathrm{Cov}(p_i^b, p_j^b) = \mathbb{E}(p_i^b p_j^b) - \mathbb{E} p_i^b \cdot \mathbb{E} p_j^b = - p_i p_j \leq 0$, which implies $\mathrm{Cov} \bigl( (p_i^b - p_i) R_i(x), (p_j^b - p_j) R_j(x) \bigr) \leq 0$ for $1 \leq i < j \leq n$.
	Consequently we have
	\begin{align}\label{equ::varzetab}
		\mathrm{Var}(\zeta_b(x) | D_n) 
		\leq \sum_{i=1}^n R_i^2(x) \mathrm{Var}(p_i^b)
		= \sum_{i=1}^n R_i^2(x) p_i (1 - p_i)
		\leq \sum_{i=1}^n p_i R_i^2(x).
	\end{align}
	Let $c_{d,n} := \lceil 32 (d+4) \log n \rceil$. 
	Lemma  \ref{lem::Rrho}  implies that with probability $P^n$ at least $1-2/n^2$, there holds 
	\begin{align*}
		\sup_{x\in\mathcal{X}} R_i(x)\lesssim 
		\begin{cases}
			(\log n/n)^{1/d}, & \text{ if } 1 \leq i \leq c_{d,n}, \\
			(i / n)^{1/d}, & \text{ if } c_{d,n} \leq i \leq n.
		\end{cases}
	\end{align*}
	Consequently we have
	\begin{align*}
		\sum_{i=1}^n p_i R_i^2(x) 
		= \sum_{i=1}^{c_{d,n}} p_i R_i^2(x) + \sum_{i=c_{d,n}}^n p_i R_i^2(x)
		\lesssim  \log n ( \log n /n)^{2/d} + \sum_{i=1}^n p_i (i/n)^{2/d}.
	\end{align*}
	Using Lemma \ref{lem::ribound}, we get $\sum_{i=1}^n p_i R_i^2(x) \lesssim (\log n) \cdot (\log n/n)^{2/d} + (k/s)^{2/d}$. This together with \eqref{equ::varzetab} yields $\mathrm{Var}(\zeta_b | D_n) \lesssim (k/s)^{2/d}$. 
	Applying Bernstein's inequality in Lemma \ref{lem::bernstein}, we obtain that for any $\tau > 0$, there holds
	\begin{align*}
		P_Z^B \biggl(  \biggl| \frac{1}{B} \sum_{b=1}^B \zeta_b(x) \biggr| \gtrsim \sqrt{\frac{2 \tau  (k/s)^{2/d}}{B}} + \frac{2 \tau\mathrm{diam}(\mathcal{X})}{3B } \bigg| D_n  \biggr) 
		\leq e^{-\tau}.
	\end{align*}
	Let $\tau := (2d+4) \log n$. Then we have
	\begin{align} \label{equ::pointwise2}
		P_Z^B \Bigl( \bigl| R_k^B(x) - \widetilde{R}_k^B(x) \big|
		\lesssim \sqrt{(k/s)^{2/d} \log n / B}+\log n/B \Big| D_n \Bigr)
		\geq 1 - 1 / n^{2d+4}.
	\end{align}
	In order to derive the uniform upper bound over $\mathcal{X}$, let
	\begin{align*}
		\mathcal{S} := \bigl\{ (\sigma_1, \ldots, \sigma_{n}) : \text{all permutations of } (1, \ldots, n)\text{ obtainable by moving } x \in \mathbb{R}^d \bigr\}
	\end{align*}
	and $\varepsilon\asymp \sqrt{(k/s)^{2/d}\log n/B}+\log n/B$.
	Then we have
	\begin{align*}
		& P_Z^B \biggl( \sup_{x \in \mathbb{R}^d} \biggl( \bigl| R_k^B(x) - \widetilde{R}_k^B(x) \bigr| - \varepsilon \biggr) > 0 \bigg| D_n \biggr)
		\\
		& \leq P_Z^B \biggl( \bigcup_{(\sigma_1, \ldots, \sigma_n) \in \mathcal{S}} \biggl| \frac{1}{B} \sum_{b=1}^B \sum_{i=1}^n p_{i,\sigma}^b R_{i,\sigma}(x) - \sum_{i=1}^n p_{i,\sigma} R_{i,\sigma}(x) \biggr| > \varepsilon \bigg| D_n \biggr)
		\\
		& \leq \sum_{(\sigma_1, \ldots, \sigma_n) \in \mathcal{S}} P_Z^B \biggl( \biggl| \frac{1}{B} \sum_{b=1}^B \sum_{i=1}^n p_{i,\sigma}^b R_{i,\sigma}(x) - \sum_{i=1}^n p_{i,\sigma} R_{i,\sigma}(x) \biggr| > \varepsilon \bigg| D_n \biggr),
	\end{align*}
	where $p_{i,\sigma}^b := \eins \{ \|x - X_{\sigma_i}\| = R_k(x; D_b)\}$ and  $R_{i,\sigma}(x) := \|x - X_{\sigma_i}\|$.
	For any $(\sigma_1,\ldots,\sigma_n)\in \mathcal{S}$, \eqref{equ::pointwise2} implies
	\begin{align*}
		P_Z^B \biggl( \sup_{x \in \mathcal{X}} \biggl| \frac{1}{B} \sum_{b=1}^B \sum_{i=1}^n p_{i,\sigma}^b R_{i,\sigma}(x) - \sum_{i=1}^n p_{i,\sigma} R_{i,\sigma}(x) \biggr| > \varepsilon \bigg| D_n \biggr)
		\leq 2 / n^{2d+3}.
	\end{align*}
	This together with Lemma \ref{lem::numberreorder} yields that for all $n \geq 2d$, there holds
	\begin{align*}
		P_Z^B \biggl( \sup_{x \in \mathbb{R}^d} \bigl( \bigl| R_k^B(x) - \widetilde{R}_k^B(x) \bigr| - \varepsilon \bigr) > 0 \bigg| D_n \biggr)
		\leq 2(25/d)^d / n^3.
	\end{align*}
	Consequently we obtain 
	\begin{align*}
		P_Z^B \otimes P^n \Bigl( \bigl\| R_k^B - \widetilde{R}_k^B \bigr\|_{\infty} \lesssim \sqrt{(k/s)^{2/d} \log n / B} +\log n/B\Bigr)
		\geq 1 - 1/n^2,
	\end{align*}
	which completes the proof.
\end{proof}

\subsubsection{Proofs Related to Section \ref{sec::EstimationError}} 
\label{sec::weightsamplerho}

In this section, we present the proof of the upper bound for the estimation error. \\

\begin{proof}[Proof of Proposition \ref{prop::weightsamplerho}]
	Let $c_{d,n} := \lceil 32 (d+4) \log n \rceil $. 
	Using the triangular inequality, we get
	\begin{align}\label{equ::sumprx}
		\biggl| \sum_{i=1}^n p_i \bigl( R_{i}(x) - \overline{R}_i(x) \bigr) \biggr|
		& \leq \sum_{i=1}^n p_i \bigl| R_i(x) - \overline{R}_i(x) \bigr|
		\nonumber\\
		& = \sum_{i=c_{d,n}}^n p_i \bigl| R_i(x) - \overline{R}_i(x) \bigr|
		+ \sum_{i=1}^{c_{d,n}-1} p_i \bigl| R_i(x) - \overline{R}_i(x) \bigr|.
	\end{align}
	
	Let us consider the first term of \eqref{equ::sumprx}.
	Lemma \ref{lem::Rrho} implies that for all $x \in \mathbb{R}^d$ and $i \geq 32 (d+4) \log n$, with probability $P^n$ at least $1-2/n^2$, there hold
	\begin{align} \label{equ::overlinerid}
		\bigl| \overline{R}_i^d(x) - R_i^d(x) \bigr|
		\lesssim \sqrt{i \log n} / (nV_d\underline{c})
	\end{align}
	and 
	\begin{align} \label{equ::rkxlowerbound}
		R_i(x)
		\gtrsim (i / n)^{1/d}.
	\end{align}
	By Assumption \ref{ass::cluster}, we have $i/n = P(B(x,\overline{R}_i(x))) \leq \overline{c} V_d \overline{R}_i(x)^d$ and consequently $\overline{R}_i(x) \geq (i/(\overline{c} V_d n)^{1/d}$.
	This together with \eqref{equ::rkxlowerbound} yields 
	\begin{align} \label{equ::sumj0d1}
		\sum_{j=0}^{d-1} R_i^j(x) \overline{R}_i^{d-1-j}(x)
		\geq \sum_{j=0}^{d-1} ( i / n)^{j/d} \cdot (i / (\overline{c} V_d n))^{(d-1-j)/d}
		\gtrsim  ( i / n) )^{1-1/d}.
	\end{align}
	Combining \eqref{equ::overlinerid} and \eqref{equ::sumj0d1}, we obtain
	\begin{align*}
		\bigl| R_i(x) - \overline{R}_i(x) \bigr|
		= \frac{\bigl| \overline{R}_i^d(x) - R_i^d(x) \bigr|}{\sum_{j=0}^{d-1}  R_i^j(x) \overline{R}_i^{d-1-j}(x)}
		\lesssim i^{1/d-1/2}n^{-1/d} (\log n)^{1/2}.
	\end{align*}
	Consequently we have
	\begin{align} \label{equ::sumpicase1}
		\sum_{i=c_{d,n}}^n p_i \bigl| R_i(x) - \overline{R}_i(x) \bigr|
		& \lesssim \sum_{i=c_{d,n}}^n p_i i^{1/d-1/2} n^{-1/d} (\log n)^{1/2}
		\nonumber\\
		& \lesssim n^{-1/d} (\log n)^{1/2} \sum_{i=1}^n p_i i^{1/d-1/2} 
		\lesssim (k/s)^{1/d-1/2} (\log n/n)^{1/2}.
	\end{align}
	
	Next, let us consider the second term of \eqref{equ::sumprx}. 
	Lemma \ref{lem::Rrho} implies that for all $x \in \mathbb{R}^d$, if $i \leq \lceil 32(d+4)\log n\rceil$, then $R_i(x) \leq ( 64 (d+5) \log n / (\underline{c} V_d n) )^{1/d}$. 
	Using Assumption \ref{ass::cluster} with $\mathcal{X} := [0, 1]^d$, we get $i/n = P(B(x,\overline{R}_i(x))) \geq \underline{c} V_d \overline{R}_i(x)^d/2^d$ and consequently $\overline{R}_i(x) \leq 2 (i / (\underline{c} V_d n))^{1/d} \lesssim  (\log n/n)^{1/d}$.
	Therefore, we have $\bigl| R_i(x) - \overline{R}_i(x) \bigr| \lesssim  (\log n / n)^{1/d}$ and thus 
	\begin{align} \label{equ::sumpicase2}
		\sum_{i=c_{d,n}}^n p_i \bigl| R_i(x) - \overline{R}_i(x) \bigr| 
		\lesssim  3 (c_{d,n}+1) ( \log n / n)^{1/d}
		\lesssim \log n(\log n/n)^{1/d}.
	\end{align}
	Combining \eqref{equ::sumprx}, \eqref{equ::sumpicase1}, \eqref{equ::sumpicase2}, and using the assumption $(kn/s)^{1-d/2} \geq (\log n)^{1+d/2}$, we obtain
	\begin{align*}
		\biggl| \sum_{i=1}^n p_i \bigl( R_{i}(x) - \overline{R}_i(x) \bigr) \biggr|
		& \lesssim \log n (\log n/ n)^{1/d} + (k/s)^{1/d-1/2} (\log n/n)^{1/2}
		\\
		& \lesssim (k/s)^{1/d-1/2} (\log n/n)^{1/2}
	\end{align*}
	for all $x \in \mathcal{X}$, which completes the proof.
\end{proof}

\subsubsection{Proofs Related to Section \ref{sec::ApproximationError}}
\label{sec::weightedrho}

In this section, we present the proof of the upper bound for the approximation error. \\

\begin{proof}[Proof of Proposition \ref{prop::weightedrho}]
	By Assumption \ref{ass::cluster}, we have that for all $x \in \mathcal{X}$, 
	\begin{align} \label{equ::ridxinvd}
		\biggl| \overline{R}_i^d(x) - \frac{i/n}{V_d f(x)} \biggr|
		= \biggl| \frac{i/n - V_d f(x) \overline{R}_i^d(x)}{V_d f(x)} \biggr|
		\leq \biggl| \frac{i/n - V_d f(x) \overline{R}_i^d(x)}{V_d \underline{c}} \biggr|.
	\end{align}
	By the definition of $\overline{R}_i(x)$ and the H\"{o}lder continuity in Assumption \ref{ass::cluster}, we have
	\begin{align} \label{equ::invdfx}
		\bigl| i/n & - V_d f(x) \overline{R}_i^d(x) \bigr|
		= \biggl| \int_{B(x, \overline{R}_i(x))} f(x') \, dx' - \int_{B(x, \overline{R}_i(x))} f(x) \, dx' \biggr|
		\nonumber\\
		& \leq \int_{B(x, \overline{R}_i(x))} |f(x') - f(x)| \, dx'
		\leq c_L \int_{B(x, \overline{R}_i(x))} \|x' - x\|^{\alpha} \, dx'
		\leq c_d c_L \overline{R}^{d+\alpha}_i(x),
	\end{align}
	where $c_d$ is a constant depending only on $d$. Moreover, by Assumption \ref{ass::cluster} and the definition of $\overline{R}_i(x)$, we have
	$\underline{c} V_d \overline{R}_i^d(x) / 2^d \leq P(B(x, \overline{R}_i(x))) = i/n \leq V_d \overline{c} \overline{R}_i^d(x)$ and consequently
	\begin{align}\label{equ::overliner}
		((i/n) / (V_d \overline{c}))^{1/d}
		\leq \overline{R}_i(x)
		\leq 2 ((i/n) / (\underline{c} V_d))^{1/d}.
	\end{align}
	Combining \eqref{equ::overliner} and \eqref{equ::invdfx}, we get
	$|i/n - V_d f(x) \overline{R}_i^d(x)|
	\leq 2^{d+\alpha} c_d c_L ((i/n) / (\underline{c} V_d))^{(d+\alpha)/d}$.
	This together with \eqref{equ::ridxinvd} yields 
	$|\overline{R}_i^d(x) - (i/n) / (V_d f(x))|
	\leq (2^{d+\alpha} c_d c_L / (V_d \underline{c})) \cdot ((i/n) / (\underline{c}V_d))^{(d+\alpha)/d}$.
	The first inequality of \eqref{equ::overliner} implies
	\begin{align*}
		& \sum_{j=0}^d \overline{R}_i(x)^j ((i/n) / (V_d f(x)))^{(d-i-j)/d}
		\\
		& \geq 	\sum_{j=0}^d ((i/n) / (V_d \overline{c}))^{j/d} ((i/n) / (V_d \underline{c}))^{(d-1-j)/d}
		\geq ((i/n) / (V_d \overline{c}))^{(d-1)/d}.
	\end{align*}
	Using the equality $x^d - y^d = (x-y)(\sum_{i=0}^{d-1} x^i \cdot y^{d-1-i})$, we get
	\begin{align*}
		\bigl| \overline{R}_i(x) - ((i/n) / (V_d f(x)))^{1/d} \bigr|
		= \frac{\bigl| \overline{R}^d_i(x) - (i/n) / (V_d f(x)) \bigr|}{\sum_{j=0}^d \overline{R}_i(x)^j ((i/n) / (V_d f(x)))^{(d-1-j)/d}}
		\lesssim (i/n)^{(1+\alpha)/d}
	\end{align*}
	and consequently
	\begin{align*}
		& \biggl| \sum_{i=1}^n p_i \overline{R}_i(x) -  \sum_{i=1}^n p_i ((i/n) / (V_d f(x)))^{1/d} \biggr| 
		\\
		& \lesssim \sum_{i=1}^n p_i	\biggl| \overline{R}_i(x) -  \sum_{i=1}^n p_i ((i/n) / (V_d f(x)))^{1/d} \biggr|
		\lesssim \sum_{i=1}^n p_i (i/n)^{(1+\alpha)/d}.
	\end{align*}
	Lemma \ref{lem::ribound} implies
	$\sum_{i=1}^n p_i (i/n)^{(1+\alpha)/d}	
	\lesssim (k/s)^{(1+\alpha)/d}$
	and thus we have
	\begin{align*}
		\biggl| \sum_{i=1}^d p_i \overline{R}_i(x) - \sum_{i=1}^n p_i ((i/n) / (V_d f(x)))^{1/d} \biggr| 
		\lesssim (k/s)^{(1+\alpha)/d},
	\end{align*}
	which completes the proof.
\end{proof}

\subsubsection{Proofs Related to Section \ref{sec::RatesPseudoDE}}\label{sec::RatesPseudoDEproof}

The following Lemma, which is needed in the proof of Proposition \ref{thm::main3}, bounds the difference between the bagged $k$-distance and its infinite version.

\begin{lemma}\label{prop::dist}
	Let Assumption \ref{ass::cluster} hold. 
	Furthermore, let $R_k^B(x)$ and $p_i$ be as in \eqref{eq::bd}  and \eqref{equ::def}, respectively. 
	Moreover, suppose that $(kn/s)^{1-d/2} \geq (\log n)^{1+d/2}$. 
	Then for all $x \in \mathbb{R}^d$,
	with probability $P_Z^B \otimes P^n$ at least $1 - 3/n^2$,
	there holds
	\begin{align*}
		& \biggl| R_k^B(x) - \sum_{i=1}^n p_i ((i/n) / (V_d f(x)))^{1/d} \biggr|
		\nonumber\\
		& \lesssim \sqrt{{(k/s)^{2/d}\log n}/{B}}+\log n/B+ (k/s)^{1/d-1/2}(\log n/n)^{1/2} + (k/s)^{(1+\alpha)/d}.
	\end{align*}
	
\end{lemma}

\begin{proof}[Proof of Lemma \ref{prop::dist}]
	Using the triangle inequality, we get
	\begin{align*}
		& \biggl| R_k^B(x) - \sum_{i=1}^n p_i ((i/n) / (V_d f(x)))^{1/d} \biggr|
		\nonumber\\
		& \leq \bigl| R_k^B(x) -\widetilde{R}_k^B(x)\bigr|
		+ \biggl| \sum_{i=1}^n p_i \bigl( R_i(x) - \overline{R}_i(x) \bigr) \biggr| 
		+ \biggl| \sum_{i=1}^n p_i \bigl( \overline{R}_i(x) - ((i/n) / (V_d f(x)))^{1/d} \bigr) \biggr|.
	\end{align*}
	Then Propositions \ref{prop::weightB}, \ref{prop::weightsamplerho}, and \ref{prop::weightedrho} yield that for all $x \in \mathcal{X}$, 
	with probability at least $1 - 3/n^2$,
	there holds
	\begin{align*}
		& \biggl| R_k^B(x) - \sum_{i=1}^n p_i ((i/n) / (V_d f(x)))^{1/d} \biggr|
		\\
		& \lesssim \sqrt{{(k/s)^{2/d}\log n}/{B}}+\log n/B+ (k/s)^{1/d-1/2}(\log n/n)^{1/2}+(k/s)^{(1+\alpha)/d},
	\end{align*}
	which finishes the proof.
\end{proof}

Now, we are in the position of presenting the proof of the convergence rates of the hypothetical density estimation. \\

\begin{proof}[Proof of Proposition \ref{thm::main3}]
	If we choose $k_{D,n} \asymp \log n$, $s_n \asymp  n^{d/(2\alpha+d)} (\log n)^{2\alpha/(2\alpha+d)}$ and $B_n \geq n^{(1+\alpha)/(2\alpha+d)}$ $(\log n)^{(\alpha+d-1)/(2\alpha+d)}$, then we have $(k_{D,n} n/s_n)^{1-d/2} \gtrsim (\log n)^{1+d/2}$.
	Applying Lemma \ref{prop::dist}, we obtain that for all $x \in \mathcal{X}$, there holds
	\begin{align} \label{equ::rkbminus}
		& \biggl| R_k^B(x) - \sum_{i=1}^n p_i ((i/n) / (V_d f(x)))^{1/d} \biggr|
		\nonumber\\ 
		& \lesssim \sqrt{{(k_{D,n}/s_n)^{2/d}\log n}/{B_n}}+\log n/B_n
		+ (k_{D,n}/s_n)^{1/d-1/2}(\log n/n)^{1/2}
		+ (k_{D,n}/s_n)^{(1+\alpha)/d}
		\nonumber\\
		&\lesssim (\log n/n)^{\frac{1+\alpha}{2\alpha+d}}
	\end{align}
	with probability $P_Z^B \otimes P^n$ at least $1 - 3/n^2$.
	Therefore, for all sufficiently large $n$ and $x\in \mathcal{X}$, we have
	\begin{align}\label{equ::rkbxnpi}
		\biggl| R_k^B(x) - \sum_{i=1}^n p_i ((i/n) / (V_d f(x)))^{1/d} \biggr|
		\leq (\log n/n)^{1/(2\alpha+d)}(V_d\underline{c})^{-1/d}/128.
	\end{align}
	Lemma \ref{lem::ribound} together with Assumption \ref{ass::cluster} yields that for all $x\in \mathcal{X}$, there hold
	\begin{align}\label{equ::piivdfxn}
		\sum_{i=1}^n p_i ((i/n) / (V_d f(x)))^{1/d} 
		\geq (V_d \underline{c})^{-1/d} \sum_{i=1}^n p_i (i/n)^{1/d}
		\geq  (\log n/n)^{1/(2\alpha+d)}(V_d \underline{c})^{-1/d}/64
	\end{align}
	and
	\begin{align}\label{equ::popuord11}
		\sum_{i=1}^n p_i ((i/n) / (V_d f(x)))^{1/d} 
		\leq  (V_d \overline{c})^{-1/d}  \sum_{i=1}^n p_i (i/n)^{1/d} 
		\lesssim (\log n/n)^{1/(2\alpha+d)}.
	\end{align}
	Combining \eqref{equ::rkbxnpi}, \eqref{equ::piivdfxn} and \eqref{equ::popuord11}, we find
	\begin{align}\label{equ::rkblower}
		R_k^B(x)
		\geq \sum_{i=1}^n p_i ((i/n) / (V_d f(x)))^{1/d} - \biggl| R_k^B(x) - \sum_{i=1}^n p_i ((i/n) / (V_d f(x)))^{1/d} \biggr|
		\gtrsim (\log n/n)^{\frac{1}{2\alpha+d}}
	\end{align}
	and 
	\begin{align}\label{equ::rkbupper}
		R_k^B(x)
		\leq \sum_{i=1}^n p_i ((i/n) / (V_d f(x)))^{1/d} + \biggl| R_k^B(x) - \sum_{i=1}^n p_i ((i/n) / (V_d f(x)))^{1/d} \biggr|
		\lesssim (\log n/n)^{\frac{1}{2\alpha+d}}.
	\end{align}
	Combining \eqref{equ::popuord11} and \eqref{equ::rkbupper}, we get
	\begin{align*}
		\sum_{j=0}^d \bigl( R_k^B(x) \bigr)^j \biggl( \sum_{i=1}^n p_i ((i/n) / (V_d f(x)))^{1/d} \biggr)^{d-1-j}
		\lesssim (\log n/n)^{\frac{d-1}{2\alpha+d}}.
	\end{align*}
	This together with \eqref{equ::rkbxnpi} yields 
	\begin{align} \label{equ::denominatorlower1}
		&\biggl| \bigl( R_k^B(x) \bigr)^d - \biggl( \sum_{i=1}^n p_i ((i/n) / (V_d f(x)))^{1/d} \biggr)^d  \biggr|
		\nonumber\\
		& \lesssim \biggl| R_k^B(x) - \sum_{i=1}^n p_i ((i/n) / (V_d f(x)))^{1/d} \biggr| \cdot \sum_{j=0}^d \bigl( R_k^B(x) \bigr)^j \biggl( \sum_{i=1}^n p_i ((i/n) / (V_d f(x)))^{1/d} \biggr)^{d-1-j}
		\nonumber\\
		& \lesssim (\log n/n)^{\frac{\alpha+d}{2\alpha+d}}.
	\end{align}
	Combining \eqref{equ::rkblower} and \eqref{equ::denominatorlower1}, we obtain that
	for all $x \in \mathcal{X}$ and all sufficiently large $n$, there holds
	\begin{align}
		\biggl| \frac{\bigl( \sum_{i=1}^n p_i (i/n)^{1/d} \bigr)^d}{V_d \bigl(R_k^B(x) \bigr)^d} - f(x) \biggr| 
		& = \biggl| \frac{\bigl( \sum_{i=1}^n p_i (i/(V_d f(x) n))^{1/d} \bigr)^d -  \bigl(R_k^B(x) \bigr)^d}{\bigl( R_k^B(x) \big)^d} \biggr| \cdot f(x) \nonumber
		\nonumber \\
		& \lesssim (\log n/n)^{\frac{\alpha}{2\alpha+d}}\label{equ::fxbound}
	\end{align}
	with probability $P_Z^B \otimes P^n$ at least $1 - 3 / n^2$.
	This finishes the proof.
\end{proof}

The following lemma, which will be used several times in the sequel, supplies the key to proofs of  mode estimation and mode-based clustering.

\begin{lemma}\label{lem::Eigen}
	Let Assumption \ref{ass::modes} hold. 
	Moreover, let $x \in \mathcal{M}_{r_{\mathcal{M}}}$ and $H(x)$ be the corresponding Hessian matrix. Then there exist two constants $c_1 \geq c_2 > 0$ such that for any $y \in \mathbb{R}^d$, there holds $- c_1 \|y\|^2 \leq y^{\top} H(x) y \leq - c_2 \|y\|^2$.	
	Moreover, for all $1 \leq i \leq \#(\mathcal{M})$ and all $x, y \in B(X_i, r_{\mathcal{M}})$, we have
	\begin{align*}
		f(y) & \leq f(x) + \nabla f(x)^{\top} (y - x) - c_2 \|y - x\|^2/2,
		\\
		f(y) & \geq f(x) + \nabla f(x)^{\top} (y - x) - c_1 \|y - x\|^2/2.
	\end{align*}
\end{lemma}

\begin{proof}[Proof of Lemma \ref{lem::Eigen}]
	For $x \in \mathcal{M}_{r_{\mathcal{M}}}$,
	let $\lambda_i(x)$, $1 \leq i \leq n$ be the eigenvalues of $H(x)$. 
	By Assumption \ref{ass::modes}, $f$ is twice continuously differentiable in $\mathcal{M}_{r_{\mathcal{M}}}$. 
	Consequently, $\lambda_i(x)$ is continuous in $\mathcal{M}_{r_{\mathcal{M}}}$. 
	Applying the extreme value theorem to $\lambda_i(x)$, there exist two constant $c_2'$ and $c_1'$ such that
	\begin{align} \label{equ::lambdax}
		c_1' \leq \lambda_i(x) \leq c_2', 
		\qquad 
		x \in \mathcal{M}_{r_{\mathcal{M}}}.
	\end{align}
	By Assumption \ref{ass::modes}, $H(x)$ is negative definite. 
	Thus, we have $\lambda_i(x) < 0$ for all $x \in B(X_i, r_{\mathcal{M}})$, $1 \leq i \leq \#(\mathcal{M})$. 
	This together with \eqref{equ::lambdax} yields
	\begin{align} \label{equ::lambdax1}
		c_1' \leq \lambda_i(x) \leq c_2' < 0, 
		\qquad 
		x \in \mathcal{M}_{r_{\mathcal{M}}}.
	\end{align}
	Since $H(x)$ is negative definite for all $x\in \mathcal{M}_{r_{\mathcal{M}}}$, there exists an orthogonal matrix $T$ such that $T^{\top} H(x) T = \mathrm{diag} \{ \lambda_1(x), \ldots, \lambda_n(x) \}$. 
	With $\widetilde{y} := T y$ we then have
	\begin{align} \label{equ::ytophy}
		y^{\top} H(x) y 
		= \widetilde{y}^{\top} \mathrm{diag} \{ \lambda_1(x), \ldots, \lambda_n(x) \} \widetilde{y}
		= \sum_{i=1}^n \lambda_i(x) \widetilde{y}_i^2.
	\end{align}
	Combining \eqref{equ::ytophy} and \eqref{equ::lambdax1}, we obtain $c_1' \|\widetilde{y}\|^2 \leq y^{\top} H(x) y \leq c_2' \|\widetilde{y}\|^2$.
	Since $\|\widetilde{y}\|^2 = y^{\top} T^{\top} T y = \|y\|^2$, by choosing $c_1 = - c_1'$ and $c_2 = - c_2'$, we obtain
	\begin{align} \label{equ::-c1y}
		- c_1 \|y\|^2
		\leq y^{\top} H(x) y
		\leq - c_2 \|y\|^2.	
	\end{align}
	By Taylor's expansion, we have $f(y) = f(x) + \nabla f(x)^{\top} (y - x) + (y - x)^{\top} H(\xi) (y - x)/2$ for all $x, y \in B(m_i, r_{\mathcal{M}})$.
	This together with \eqref{equ::-c1y} yields $- c_1 \|y - x\|^2 \leq f(y) - f(x) - \nabla f(x)^{\top} (y - x) \leq - c_2 \|y - x\|^2$, which completes the proof.
\end{proof}

The next proposition, which is need in the proof of Proposition \ref{prop::dist1}, provides a tighter bound for the approximation error due to the higher order of smoothness around the modes. 

\begin{proposition}\label{prop::weightedrho2}
	Let Assumptions \ref{ass::cluster} and \ref{ass::modes} hold. 
	Moreover, let $p_i$ be the probability as in \eqref{equ::def} and $\overline{R}_i(x)$ be the quantile diameter function of $x$ as in \eqref{equ::overlinerx}. 
	Then for any $x \in \mathcal{M}_{r/2}$, we have 
	\begin{align*}
		\biggl| \sum_{i=1}^n p_i \overline{R}_i(x) - \sum_{i=1}^n p_i ((i/n) / (V_d f(x)))^{1/d} \biggr|
		\lesssim (k/s)^{3/d}.
	\end{align*}
\end{proposition}

\begin{proof}[Proof of Proposition \ref{prop::weightedrho2}]
	Let $c_{n,r} := \lfloor (r/2)^d n \underline{c} V_d \rfloor$.
	Using the triangular inequality, we get
	\begin{align*}
		\biggl| \sum_{i=1}^n p_i \overline{R}_i(x) - \sum_{i=1}^n p_i ((i/n) / (V_d f(x)))^{1/d} \biggr|
		& \leq \biggl| \sum_{i=1}^{c_{n,r}} p_i \bigl( \overline{R}_i(x) - ((i/n) / (V_d f(x)))^{1/d} \bigr) \biggr|
		\\
		& \phantom{=}
		+ \biggl| \sum_{i=c_{n,r}}^n p_i \bigl( \overline{R}_i(x) - ((i/n) / (V_d f(x)))^{1/d} \bigr) \biggr|.
	\end{align*}
	The boundedness of $f$ in Assumption \ref{ass::cluster} \textit{(i)} implies that for all $x \in \mathcal{X}$, there holds
	\begin{align*}
		\bigl| \overline{R}_i^d(x) - (i/n) / (V_d f(x)) \bigr|
		= \biggl| \frac{i/n - V_d f(x) \overline{R}_i^d(x)}{V_d f(x)} \biggr|
		\leq \biggl| \frac{i/n - V_d f(x) \overline{R}_i^d(x)}{V_d \underline{c}} \biggr|.
	\end{align*}
	By the definition of $\overline{R}_i(x)$, we have
	\begin{align} \label{equ::invdf}
		\bigl| i/n - V_d f(x) \overline{R}_i^d(x) \bigr|
		& = \biggl| \int_{B(x,\overline{R}_i(x))} f(x') \, dx' - \int_{B(x, \overline{R}_i(x))} f(x) \, dx' \biggr|
		\nonumber\\
		& = \biggl| \int_{B(x,\overline{R}_i(x))} ( f(x') - f(x) ) \, dx' \biggr|.
	\end{align}
	By Assumption \ref{ass::cluster}, we have
	$\underline{c} V_d \overline{R}_i^d(x) / 2^d
	\leq P(B(x,\overline{R}_i(x)))
	= i/n
	\leq V_d \overline{c} \overline{R}^d_i(x)$
	for all $x \in \mathcal{X}$,
	which yields 
	\begin{align} \label{equ::overliner1}
		((i/n) / (V_d \overline{c}))^{1/d}
		\leq \overline{R}_i(x)
		\leq 2 ((i/n) / (\underline{c} V_d))^{1/d}, 
		\qquad 
		\forall x \in \mathcal{X}.
	\end{align}
	If $i \leq c_{n,r}$, then we have $\overline{R}_i(x)\leq r/2$. 
	Consequently, for all $x \in \mathcal{M}_{r/2}$ and $x' \in B(x,\overline{R}_i(x))$, there exists an $m_i \in \mathcal{M}$ such that $\|x' - m_i\| \leq \|x' - x\| + \|x - m_i\| \leq r$. 
	Therefore, we have $x' \in \mathcal{M}_r$.
	Using Taylor's expansion, we get
	\begin{align*}
		f(x') = f(x) + \nabla f(x)^{\top} (x' - x) + (x' - x)^{\top} H(x_{\xi})(x' - x).
	\end{align*}
	Then Lemma \ref{lem::Eigen} implies
	\begin{align}\label{equ::bxrix}
		\biggl| \int_{B(x, \overline{R}_i(x))} & \bigl( f(x') - f(x) \bigr) \, dx' \biggr|
		= \biggl| \int_{B(x, \overline{R}_i(x))} \bigl( \nabla f(x)^{\top} (x' - x) + (x' - x)^{\top} H(x_{\xi})(x' - x) \bigr) \, dx' \biggr| \nonumber
		\\
		& = \biggl| \int_{B(x, \overline{R}_i(x))} (x' - x)^{\top} H(x_{\xi})(x' - x) \bigr) \, dx' \biggr|
		\leq c_1 \int_{B(x, \overline{R}_i(x))} \|x' - x\|^2 \, dx' \nonumber
		\\
		& \lesssim c_d \overline{R}_i^{d+2}(x)
		\lesssim (i/n)^{1+2/d}.
	\end{align}
	This together with \eqref{equ::invdf} yields that
	$\bigl| i/n - V_d f(x) \overline{R}_i^d(x) \bigr|
	\lesssim (i/n)^{1+2/d}$ holds for all
	$i \leq c_{n,r}$ and
	consequently 
	\begin{align} \label{equ::invdileq}
		\bigl| i/n - V_d f(x) \overline{R}_i^d(x) \bigr| / (V_d \underline{c})
		\lesssim (i/n)^{1+2/d},  
		\qquad 
		i \leq c_{n,r}.
	\end{align}
	The first inequality of \eqref{equ::overliner1} implies
	\begin{align} \label{equ::sumrij}
		& \sum_{j=0}^d \overline{R}_i(x)^j ((i/n) / (V_d f(x)))^{(d-i-j)/d}
		\nonumber\\
		& \geq 	\sum_{j=0}^d ((i/n) / (V_d \overline{c}))^{j/d} ((i/n) / (V_d \underline{c}))^{(d-1-j)/d}
		\geq ((i/n) / (V_d \overline{c}))^{(d-1)/d}.
	\end{align}
	This together with \eqref{equ::invdileq} yields 
	\begin{align}\label{equ::term1}
		\bigl| \overline{R}_i(x) - ((i/n) / (V_d f(x)))^{1/d} \bigr|
		\lesssim (i/n)^{3/d}, 
		\qquad  
		i \leq c_{n,r},
	\end{align}
	where we used the equality $x^d - y^d = (x-y)(\sum_{i=0}^{d-1} x^i \cdot y^{d-1-i})$.
	
	On the other hand, the H\"{o}lder continuity in Assumption \ref{ass::cluster} implies
	\begin{align*}
		\bigl| i/n - V_d f(x) \overline{R}_i^d(x) \bigr|
		& \leq \int_{B(x, \overline{R}_i(x))} |f(x') - f(x)| \, dx'
		\nonumber\\
		& \leq c_L \int_{B(x, \overline{R}_i(x))} \|x' - x\|^{\alpha} \, dx'
		\lesssim \overline{R}^{d+\alpha}_i(x)
		\lesssim (i/n)^{(\alpha+d)/d},
	\end{align*}
	where the last inequality follows from \eqref{equ::overliner1}.
	This together with \eqref{equ::sumrij} yields 
	\begin{align}\label{equ::term2}
		\bigl| \overline{R}_i(x) - ((i/n) / (V_d f(x)))^{1/d} \bigr|
		\lesssim (i/n)^{(\alpha+1)/d}, 
		\qquad  
		i > c_{n,r},
	\end{align}
	where we use the equality $x^d - y^d = (x - y) (\sum_{i=0}^{d-1} x^i y^{d-1-i})$. 
	Combining \eqref{equ::term1} and \eqref{equ::term2}, we obtain that for all $x \in \mathcal{M}_{r/2}$, there holds
	\begin{align*}
		& \biggl| \sum_{i=1}^n p_i \overline{R}_i(x) - \sum_{i=1}^n p_i ((i/n) / (V_d f(x)))^{1/d} \biggr|
		\leq \sum_{i=1}^n p_i |\overline{R}_i(x) - (i/(nV_df(x)))^{1/d}|
		\\
		& \lesssim \sum_{i=1}^{c_{n,r}} p_i (i/n)^{3/d} + \sum_{i=c_{n,r}+1}^n p_i (i/n)^{(1+\alpha)/d}
		\lesssim n^{-3/d} \sum_{i=1}^{c_{n,r}} p_i i^{3/d} + n^{-(1+\alpha)/d} \sum_{i=c_{n,r}+1}^n p_i i^{(1+\alpha)/d}
		\\
		& \lesssim n^{-3/d} \sum_{i=1}^n p_i i^{3/d} + n^{-(1+\alpha)/d} c_{n,r}^{(\alpha-2)/d} \sum_{i=1}^n p_i i^{3/d}
		\leq n^{-3/d} (kn/s)^{3/d}
		= (k/s)^{3/d},
	\end{align*}
	which completes the proof.
\end{proof}

The next proposition, which is needed in the proof of Proposition \ref{prop::dist1}, presents the error between the bagged $k$-distance and its infinite version around the modes. This result is in fact an improvement of Proposition \ref{prop::dist}.

\begin{proposition}\label{prop::dist1}
	Let Assumptions \ref{ass::cluster} and \ref{ass::modes} hold.
	Furthermore, let $R_k^B(x)$ and $p_i$ be defined in \eqref{eq::bd}  and \eqref{equ::def}, respectively. 
	Moreover, suppose that $(kn/s)^{1-d/2} \geq (\log n)^{1+d/2}$.
	Then for all $x \in \mathcal{M}_{r/2}$,
	there holds
	\begin{align*}
		& \biggl| R_k^B(x) - \sum_{i=1}^n p_i ((i/n) / (V_d f(x)))^{1/d} \biggr|
		\nonumber\\
		& \lesssim \sqrt{{(k/s)^{2/d}\log n}/{B}} + \log n/B+ (k/s)^{1/d-1/2}(\log n/n)^{1/2}+(k/s)^{3/d}
	\end{align*}
	with probability $P_Z^B \otimes P^n$ at least $1 - 3/n^2$.
\end{proposition}

\begin{proof}[Proof of Proposition \ref{prop::dist1}]
	The proof is similar to that of Proposition \ref{prop::dist} by replacing the approximation error bound with the bound in Proposition \ref{prop::weightedrho2}. 
	Thus we omit the proof.
\end{proof}

With the above results, we are able to present the proof of the convergence rates for the hypothetical density estimation around modes. \\

\begin{proof}[Proof of Proposition \ref{thm::main2}]
	Similar to the proof of Proposition \ref{thm::main3}, we can show the desired assertion by applying Proposition \ref{prop::dist1}. Therefore, we omit the proof.
\end{proof}

\subsection{Proofs Related to Section \ref{sec::ratesbdmbc}}
\label{sec::proofmode}

In this section, we provide proofs related to mode estimation. We give details of proofs for BDMBC, whereas DMBC can be dealt with similarly. To derive the convergence rates of mode estimation for BDMBC, we first show that the hypothetical density estimation around the modes is no less than the supremum of that far away from the modes in Proposition \ref{lem::kernelmode111}, which implies that the local maximum of hypothetical density estimation is close to the modes. 
Then by using Bernstein's inequality in Lemma \ref{lem::bernstein}, we establish concentration inequality for the localized level sets in Lemma \ref{lem::netaxlfx} and derive the distance between the empirical PLLS and the population version of PLLS.
Furthermore, we show that those points which are far away from the modes have a small population PLLSs. 
Hence we can show that the points with lower PLLS are not included in the level sets and thus we can find cluster estimators corresponding to the modes in a subjective manner.

The next proposition, which plays a key role in the proofs related to mode estimation, is needed in the proof of Theorem \ref{thm::modesingle}.

\begin{proposition}\label{lem::kernelmode111}
	Let Assumptions \ref{ass::cluster} and \ref{ass::modes} hold. 
	Moreover, let $f_B(x)$ be the hypothetical density estimator as in \eqref{equ::fbk}. 
	By choosing
	\begin{align*}
		k_{D,n} \asymp \log n,
		\qquad
		s_n \asymp  n^{\frac{d}{4+d}} (\log n)^{\frac{4}{4+d}},
		\qquad
		B_n \geq n^{\frac{3}{4+d}}(\log n)^{\frac{d+1}{4+d}}, 
	\end{align*}
	then with probability $P^n$ at least $1-2/n^2$, there holds
	\begin{align*}
		\inf \bigl\{ f_B(x) : x \in B(m_i, c' r_n) \bigr\} 
		> \sup \bigl\{ f_B(x) : x \in B(m_i, r_{\mathcal{M}}/2) \setminus B(m_i, r_n) \bigr\} 
	\end{align*}
	where $c':=(c_2/(2c_1))^{1/2}$ with the constants $c_1$ and $c_2$ specified as in Lemma \ref{lem::Eigen}.
\end{proposition}

\begin{proof}[Proof of Proposition \ref{lem::kernelmode111}]
	Proposition \ref{thm::main2} yields that there exists a constant $c>0$ such that for all sufficiently large $n$, with probability $P^n$ at least $1-2/n^2$, for all $x\in \mathcal{M}_{r/2}$, there holds
	\begin{align}\label{equ::key1}
		| f_B(x) - f(x) |
		\leq c (\log n/n)^{\frac{2}{4+d}}.
	\end{align}
	
	The following arguments will be made on the good event $E$ in which \eqref{equ::key1} holds.
	
	Let $r_n:=(8c/c_2)^{1/2}(\log n/n)^{1/(4+d)}$. 
	Then we have $r_n\leq r_{\mathcal{M}}/2$ for sufficiently large $n$.
	By Lemma \ref{lem::Eigen}, we have $f(m_i) - c_1 \|x - m_i\|^2/2\leq f(x) \leq f(m_i) - c_2 \|x - m_i\|^2/2$ for all $x \in B(m_i, r_\mathcal{M})$. 
	Consequently, we have $\sup \bigl\{ f(x) : x \in B(m_i, r_{\mathcal{M}}/2) \setminus B(m_i, r_n) \bigl\} \leq f(m_i) - c_2/2 r_n^{2}$.
	This together with \eqref{equ::key1} yields that
	$\sup\bigl\{ f_B(x) : x \in B(m_i, r_{\mathcal{M}}/2) \setminus B(m_i,r_n) \bigl\} 
	\leq f(m_i) - c_2 r_n^2 + c(\log n/n)^{2/(4+d)}$.
	On the other hand, by Lemma \ref{lem::Eigen}, we have 
	$\inf \bigl\{ f(x) : x \in B(m_i, c' r_n) \bigl\} \geq f(m_i) - c_1 (c' r_n)^{2}/2$.
	This together with \eqref{equ::key1} yields that
	$\inf \bigl\{ f_B(x) : x \in B(m_i, c' r_n) \bigl\} \geq f(m_i) - c_1 (c' r_n)^{2}/2 - c (\log n/n)^{2/(4+d)}$.
	Consequently we obtain
	\begin{align*}
		\inf\bigl\{ f_B(x) : x \in B(m_i, c' r_n) \bigl\} 
		& \geq f(m_i) - c_1 (c' r_n)^{2}/2 - c (\log n/n)^{2/(4+d)}
		\\
		& = f(m_i) - c_2 r_n^{2}/2 + c (\log n/n)^{2/(4+d)} 
		\\
		& \geq \sup\bigl\{ f_B(x) : x \in B(m_i, r_{\mathcal{M}}/2) \setminus B(m_i, r_n) \bigl\},
	\end{align*}
	which completes the proof.
\end{proof}

The following Lemma, which is need in the proof of Theorem \ref{thm::mode}, presents the uniform concentration bounds on the empirical mass of balls in $\mathbb{R}^d$.

\begin{lemma}\label{lem::netax}
	Let $P$ be a probability measure on $\mathbb{R}^d$ with a bounded Lebesgue density $f$ and $\eta : \mathbb{R}^d \to (0,\infty)$ be the local radius parameter function. Then for all $x \in \mathbb{R}^d$, $n \geq 1$, and $\tau > 0$, 	with probability $P^n$ at east $1 - 2 e^{-\tau}$, there holds
	\begin{align*}
		\biggl | \frac{1}{n} \sum_{i=1}^n \eins \{ X_i \in B(x, \eta(x)) \} - P(B(x, \eta(x))) \biggr |
		\lesssim  \sqrt{ \|\eta\|_{\infty}^d \log n / n} +  \log n / n.
	\end{align*}
\end{lemma}

\begin{proof}[Proof of Lemma \ref{lem::netax}]
	Let us consider the set $\mathcal{B}_\eta :=  \big\{ B(x,\eta(x))  : x \in \mathbb{R} ^d\bigr\} \subset \mathcal{B}$. Lemma \ref{lem::CoveringNumber} implies that for any probability $\mathrm{Q}$, there holds
	\begin{align} \label{Bk-CoveringNumber1}
		\mathcal{N}(\eins_{\mathcal{B}_{\eta}}, \|\cdot\|_{L_1(\mathrm{Q})}, \varepsilon)
		\leq \mathcal{N}(\eins_{\mathcal{B}}, \|\cdot\|_{L_1(\mathrm{Q})}, \varepsilon)
		\leq C (d+2) (4e)^{d+2} \varepsilon^{-(d+1)}.
	\end{align}
	By the definition of the covering number, there exists an $\varepsilon$-net $\{A_j\}_{j=1}^J \subset \mathcal{B}_{\eta}$ with $J := \lfloor C (d+2) (4e)^{d+2} \varepsilon^{-(d+1)} \rfloor$ and for any $x \in \mathcal{X}$, there exists some $j \in \{ 1, \ldots, J \}$ such that 
	\begin{align}\label{eq::approxAj-}
		\bigl\| \eins \bigl\{ B(x,\eta(x))\}- \eins_{A_j} \bigr\|_{L_1(D)} 
		\leq \varepsilon.
	\end{align}
	For any $i = 1, \ldots, n$, let the random variables $\xi_i$ be defined by $\xi_i = \eins_{A_j}(X_i) - P(A_j)$. Then we have $\mathbb{E}_{P}\xi_i = 0$, $\|\xi_i\|_{\infty} \leq 1$, and $\mathbb{E}_{P}\xi_i^2 \leq P(A_j) \leq \overline{c} V_d \eta(x)^d\leq \overline{c} V_d\|\eta\|_{\infty}^d$. 
	Applying Bernstein's inequality in Lemma \ref{lem::bernstein}, we obtain 
	\begin{align*}
		\frac{1}{n} \sum_{i=1}^n \eins_{A_j}(X_i) - P(A_j)
		\leq \sqrt{2 \overline{c} V_d\|\eta\|_{\infty}^d \tau / n} + 2 \tau / (3n)
	\end{align*}
	with probability $P^n$ at least $1-e^{-\tau}$. 
	Then the union bound together with the covering number estimate \eqref{Bk-CoveringNumber1} implies that for any $A_j$, $j = 1, \cdots, J$, there holds 
	\begin{align*}
		\frac{1}{n} \sum_{i=1}^n \eins_{A_j}(X_i) - P(A_j)
		\leq \sqrt{2 \overline{c} V_d \|\eta\|_{\infty}^d (\tau + \log J) / n} + 2 (\tau + \log J) / (3 n).
	\end{align*}
	This together with \eqref{eq::approxAj-} yields that for all $x \in \mathcal{X}$, there holds
	\begin{align*}
		& \frac{1}{n} \sum_{i=1}^n \eins \{ X_i \in B(x, \eta(x)) \} - P(B(x, \eta(x)))
		\\
		& \leq \sqrt{2 \overline{c} V_d \|\eta\|_{\infty}^d (\tau + \log J) / n}   
		+ 2 (\tau + \log J) / (3 n) + \varepsilon.
	\end{align*}
	Now, if we take $\varepsilon = 1/n$, then for any $n >  (4 e) \vee (d + 2) \vee C$, there holds $\log J =  \log C + \log(d+2) +  (d+2)\log(4e) + (d+1) \log n \leq (2d+5)\log n$. 
	Let $\tau := 2 \log n$. Then we have
	\begin{align}\label{equ::1n1n}
		& \frac{1}{n} \sum_{i=1}^n \eins \{ X_i \in B(x, \eta(x)) \} - P(B(x, \eta(x)))
		\nonumber\\
		& \leq \sqrt{2 (2d+7) \overline{c} V_d \|\eta\|_{\infty}^d \log n / n} + 2 (2d+7) \log n / (3n) + 1/n.
	\end{align}
	On the other hand, let $\xi_i' = - \xi_i$. Then we have $\mathbb{E}_P \xi_i' = 0$ and $\mathbb{E}_P \xi_i'^2 = \mathbb{E}_P \xi_i^2$. Similarly, we can show that 
	\begin{align*}
		& \frac{1}{n} \sum_{i=1}^n \eins \{ X_i \in B(x,\eta(x)) \} - P(B(x, \eta(x)))
		\\
		& \geq - \sqrt{2 (2d+7) \overline{c} V_d \|\eta\|_{\infty}^d \log n / n} - 2 (2d+7) \log n / (3n) - 1/n
	\end{align*}
	holds with probability $P^n$ at least $1-1/n^2$. 
	This together with \eqref{equ::1n1n} yields the assertion.
\end{proof}

The following Lemma, which is needed in the proof of Lemma \ref{lem::netaxlfx}, presents the covering number of the indicator functions of localized level sets.

\begin{lemma}\label{lem::vcballlevelset}
	Let $P$ be a probability measure on $\mathbb{R}^d$ with a bounded Lebesgue density $f$ and $\eta : \mathbb{R}^d \to (0, \infty)$ be the local radius parameter function.
	For $\lambda > 0$, let $\widetilde{L}_f(\lambda) := \{ x \in \mathbb{R}^d : f(x) \leq \lambda \}$ be the lower level set. 
	Moreover, let $\mathcal{B}_{\eta,L} := \{ \eins \{ B(x,\eta(x)) \cap \widetilde{L}_f(\lambda)\}, x \in \mathbb{R}^d \}$ be the collection of sets. 
	Then $\mathcal{B}_{\eta,L}$ is a uniformly bounded VC class satisfying 
	\begin{align*}
		\mathcal{N}(\mathcal{B}_{\eta,L}, L_1(D), \varepsilon)
		\leq W (d+3) (4e)^{d+3} (1/\varepsilon)^{d+1},
	\end{align*}
	where $W > 0$ is a universal constant.
\end{lemma}

\begin{proof}[Proof of Lemma \ref{lem::vcballlevelset}]
	We first show that the collection of sets $\widetilde{\mathcal{L}}_f := \{ \widetilde{L}_f(\lambda), \lambda > 0 \}$ are nested with VC dimension $2$ by  contradiction. 
	Suppose that $\mathrm{VC}(\widetilde{\mathcal{L}}_f) > 2$.
	Then there exists two distinct points $x_1, x_2 \in \mathbb{R}^d$ that can be shattered by $\widetilde{\mathcal{L}}_f$, i.e, $\widetilde{L}_f(\lambda_1) \cap \{ x_1, x_2 \} = x_1$ and $\widetilde{L}_f(\lambda_2) \cap \{ x_1, x_2 \} = x_2$ for some $\lambda_1, \lambda_2 > 0$. 
	Consequently we have $f(x_1) \leq \lambda_1 < f(x_2)$ and $f(x_2) \leq \lambda_2 < f(x_1)$, which leads to a contradiction. 
	Therefore, we have $\mathrm{VC}(\widetilde{\mathcal{L}}_f) = 2$.
	
	On the other hand, for the collection of balls $\mathcal{B}_{\eta} := \{ B(x, \eta(x)) : x \in \mathbb{R}^d \}$, \cite{dudley1979balls} shows that for any set $A \in \mathbb{R}^d$ of $d + 2$ points, not all subsets of $A$ can be formed as a set of the form $B \cap A$ for a $B \in \mathcal{B}_{\eta}$. 
	In other words, $\mathcal{B}_{\eta}$ can not pick out all subsets from $A \in \mathbb{R}^d$ of $d + 2$ points.
	Therefore, the collection $\mathcal{B}_{\eta}$ fails to shatter $A$. 
	Consequently, according to Definition \ref{def::VC dimension}, we have
	$\mathrm{VC}(\mathcal{B}_\eta) = d + 2$.
	By Lemma 9.7 in \cite{Kosorok2008introduction}, we have
	$\mathrm{VC}(\widetilde{\mathcal{L}}_f\cap \mathcal{B}_{\eta})
	\leq \mathrm{VC}(\widetilde{\mathcal{L}}_f)+\mathrm{VC}(\mathcal{B}_{\eta})-1
	\leq d+3$.
	Then our assertion follows directly from Theorem 2.6.4 in \cite{vandervaart1996weak}.
\end{proof}

To prove Proposition \ref{lem::petah11}, we need the following Lemma which presents the uniform concentration bounds on the empirical mass of localized levels sets.

\begin{lemma}\label{lem::netaxlfx}
	Let $P$ be a probability measure on $\mathbb{R}^d$ with a bounded Lebesgue density $f$ and $\eta : \mathbb{R}^d \to (0, \infty)$ be the local radius parameter function.
	Moreover, for $\lambda > 0$, let $\widetilde{L}_f(\lambda) := \{ x \in \mathbb{R}^d : f(x) \leq \lambda \}$ be the lower level sets. 
	Then for all $x \in \mathbb{R}^d$, $n \geq 1$, $\lambda > 0$ and $\tau > 0$, 
	with probability $P^n$ at east $1 - 2 e^{-\tau}$, there holds
	\begin{align*}
		\biggl| \frac{1}{n} \sum_{i=1}^n \eins \{ X_i \in B(x, \eta(x)) \cap \widetilde{L}_f(\lambda) \} - P(B(x, \eta(x)) \cap \widetilde{L}_f(\lambda)) \biggr|
		\lesssim \sqrt{ \|\eta\|_{\infty}^d \log n / n} +  \log n / n.
	\end{align*}
\end{lemma}

\begin{proof}[Proof of Lemma \ref{lem::netaxlfx}]
	The proof is similar to that of Lemma \ref{lem::netax} and hence is omitted.
\end{proof}

The following technical Lemma is needed in the proof of Proposition \ref{lem::petah1}.

\begin{lemma}\label{lem::pabcd}
	Let $P$ be a probability measure on $\mathbb{R}^d$ and $A_i \subset \mathbb{R}^d$, $1 \leq i \leq 4$, be four sets. 
	Then we have
	\begin{align*}
		|P(A \cap B) - P(C \cap D)|
		\leq P(A \triangle C) + P(B \triangle D).
	\end{align*}
\end{lemma}

\begin{proof}[Proof of Lemma \ref{lem::pabcd}]
	We first show that for any $x \in \mathbb{R}^d$, there holds
	\begin{align}\label{equ::abcd}
		\eins \{ A \cap B \} - \eins \{ C \cap D \}
		\leq \eins \{ A \triangle C \} + \eins \{ B \triangle D \}.
	\end{align}
	It is clear to see that \eqref{equ::abcd} holds if $\eins \{ A \cap B \} - \eins \{ C \cap D \} \leq 0$. Therefore, it remains to consider the case $\eins \{ A \cap B \} - \eins \{ C \cap D \} = 1$. In this case, we have $\eins \{ A \cap B \} = 1$ and $\eins \{ C \cap D \} = 0$, which implies that $x \in A$, $x \in B$ and $x \notin C \cap D$. Consequently, if $x \notin C$, we have $\eins \{ A \triangle C \} = 1$. On the other hand, if $x \notin D$, we have $\eins \{ B \triangle D \} = 1$. Therefore, we always have $\eins \{ A \triangle C \} + \eins \{ B \triangle D \} \geq 1$. This shows \eqref{equ::abcd}. Now taking expectation with respect to $P$ on both sides of \eqref{equ::abcd}, we obtain
	\begin{align}\label{equ::pabcd}
		P(A \cap B) - P(C \cap D) 
		\leq P(A \triangle B) + P(C \triangle D).
	\end{align}
	Using the same arguments, we can show that
	\begin{align}\label{equ::pcdab}
		P(C \cap D) - P(A \cap B)
		\leq P(A \triangle B) + P(C \triangle D).
	\end{align}
	Combing \eqref{equ::pabcd} and \eqref{equ::pcdab}, we obtain the assertion.
\end{proof}

The next proposition, which is needed in the proof of Theorem \ref{thm::modesingle}, provides the difference between the empirical PLLS w.r.t.~to the $k$-distance and the population version. 

\begin{proposition}\label{lem::petah11}
	Let Assumptions \ref{ass::cluster}, \ref{ass::modes} and \ref{ass::flatness1} hold and suppose that $2\alpha\gamma \leq 4 + d$. 
	Moreover, let $p_{k_L}^B(x)$ be defined as in \eqref{getah}.
	By choosing 
	\begin{align*}
		k_{D,n} \asymp & \log n,
		\quad
		s_n \asymp  n^{\frac{d}{4+d}} (\log n)^{\frac{4}{4+d}},
		\quad
		B_n \geq n^{\frac{3}{4+d}}(\log n)^{\frac{d+1}{4+d}},
		\quad
		k_{L,n} \gtrsim n^{1-\frac{\alpha\gamma}{4+d}}(\log n)^{1+\frac{\alpha\gamma}{4+d}},
	\end{align*}
	then for all $x \in \mathcal{X}$, with probability $P^n$ at least $1 - 3/n^2$, there holds
	\begin{align*}
		|\widehat{p}_{k_L}^{B}(x) - p_{k_L}(x)|
		\lesssim (\log n)^{-1}.
	\end{align*}
\end{proposition}

\begin{proof}[Proof of Proposition \ref{lem::petah11}]
	Following similar analysis to \eqref{equ::fxbound}, by choosing 
	\begin{align*}
		k_{D,n} \asymp & \log n,
		\quad
		s_n \asymp  n^{\frac{d}{4+d}} (\log n)^{\frac{4}{4+d}},
		\quad
		B_n \geq n^{\frac{3}{4+d}}(\log n)^{\frac{d+1}{4+d}},
	\end{align*}
	we can show that
	$|f_B(x) - f(x)|
	\lesssim (\log n/n)^{\alpha/(4+d)}$
	holds for all $x\in \mathcal{X}$ with probability $P^n\otimes P_Z^B$ at least $1 - 3/n^2$.
	The following arguments will be made on this event.
	
	Let $u_n := (\log n/n)^{\alpha/(4+d)}$. Then from \eqref{getah} we get
	$\widehat{p}_{k_L}^{B}(x)
	= \frac{1}{k_L} \sum_{i=1}^{k_L} \eins \{ f_B(X_i) \leq f_B(x) \}
	\leq \frac{1}{k_L} \sum_{i=1}^{k_L} \eins \{ f(X_i) \leq f(x) + 2 u_n \}$
	and 
	$\widehat{p}_{k_L}^{B}(x)
	\geq \frac{1}{k_L} \sum_{i=1}^{k_L} \eins \{ f(X_i) \leq f(x) - 2 u_n \}$.
	Write $f^+(x) := f(x) + 2 u_n$, $f^-(x) := f(x) - 2 u_n$, and denote
	\begin{align*}
		\widehat{p}^+(x) 
		:= \frac{1}{k_L} \sum_{i=1}^{k_L} \eins \{ f(X_i) \leq f^+(x) \} 
		\quad
		\text{ and } 
		\quad
		\widehat{p}^-(x) 
		:= \frac{1}{k_L} \sum_{i=1}^{k_L} \eins \{ f(X_i) \leq f^-(x) \}.
	\end{align*}
	Then we have
	$\widehat{p}^-(x)
	\leq \widehat{p}_{k_L}^{B}(x)
	\leq \widehat{p}^+(x)$ and
	consequently
	\begin{align}\label{equ::widehatpxbound}
		|\widehat{p}_{k_L}^{B}(x) - p_{k_L}(x)|
		\leq |\widehat{p}^+(x) - p_{k_L}(x)| \vee
		|\widehat{p}^- - p_{k_L}(x)|.
	\end{align}
	Let us consider the first term $|\widehat{p}^+(x) - p_{k_L}(x)|$. By the definition of $p_{k_L}(x)$, for all $x\in \mathcal{X}$, we have
	\begin{align*}
		|\widehat{p}^+(x) - p_{k_L}(x)|
		= \biggl| \sum_{i=1}^{k_L} \frac{\eins \{ f(X_i) \leq f^+(x) \}} {k_L} - \frac{P(y \in \widetilde{L}_f(x) \cap B(x,\overline{R}_{k_L}(x)))}{P(y \in B(x,\overline{R}_{k_L}(x)))} \biggr|.
	\end{align*}
	Since $P(y\in B(x,\overline{R}_{k_L}(x))) = k_L/n$, we have
	\begin{align*}
		& |\widehat{p}^+(x) - p_{k_L}(x)|
		\nonumber\\
		& \leq \frac{n}{k_L} \biggl| \frac{1}{n} \sum_{i=1}^n \eins \bigl\{ X_i \in \widetilde{L}_f(f^+(x)) \cap B(x,R_{k_L}(x)) \bigr\} - P \bigl( y \in \widetilde{L}_f(f(x)) \cap B(x,\overline{R}_{k_L}(x)) \bigr) \biggr|
		\nonumber\\
		& \leq \frac{n}{k_L} \biggl| \frac{1}{n} \sum_{i=1}^n \eins \bigl\{ X_i \in \widetilde{L}_f(f^+(x)) \cap B(x,R_{k_L}(x)) \bigr\} - P \bigl( y \in \widetilde{L}_f(f^+(x)) \cap B(x,R_{k_L}(x)) \bigr) \biggr|
		\nonumber\\
		& \phantom{=}
		+ \frac{n}{k_L} \bigl| P \bigl( y \in \widetilde{L}_f(f^+(x)) \cap B(x,R_{k_L}(x)) \bigr) - P \bigl( y \in \widetilde{L}_f(f(x)) \cap B(x,\overline{R}_{k_L}(x)) \bigr) \bigr|
		\nonumber\\
		& =: (I) + (II).
	\end{align*}
	Lemma \ref{lem::Rrho} yields that for all sufficiently large $n$, there holds
	$(k / (4 V_d \overline{c} n))^{1/d}
	\leq R_k(x) 
	\leq (2 k / (\underline{c} n))^{1/d}$.
	For the first term $(I)$, by applying Lemma \ref{lem::netaxlfx}, we obtain that
	\begin{align}\label{equ::I1}
		(I)  \lesssim (n/k_L) \Bigl( \sqrt{ R_{k_L}(x)^d\log n/n}+\log n/n
		\Bigr)
		\lesssim \sqrt{\log n/k_L}
	\end{align}
	holds with probability $P^n$ at least $1 - 1/n^2$.
	For the second term $(II)$, by applying Lemma \ref{lem::pabcd} and Assumption \ref{ass::flatness1}, we get
	\begin{align*}
		(II)
		& \leq (n/k_L) \bigl| P \bigl( y \in \widetilde{L}_f(f^+(x)) \bigr) - P \bigl( y \in \widetilde{L}_f(f(x)) \bigr) \bigr|
		\nonumber\\
		& \phantom{=} 
		+ (n/k_L) \bigl| P \bigl( B(x,R_{k_L}(x)) \bigr) - P \bigl( B(x,\overline{R}_{k_L}(x)) \bigr) \bigr|
		\\
		& \leq  (n/k_L)\bigl| f^+(x) - f(x) \bigr|^{\gamma}
		+ (\overline{c} n/k_L)|\overline{R}_{k_\mathrm{L}}^d(x) - R_{k_L}^d(x)|.
	\end{align*}
	By applying Lemma \ref{lem::Rrho} and Assumption \ref{ass::cluster}, we have
	\begin{align*}
		(II) \lesssim n u_n^{\gamma} / k_L +\sqrt{\log n/k_L}.
	\end{align*}
	This together with \eqref{equ::I1} yields that
	$|\widehat{p}^+(x) - p_{k_L}(x)|
	\leq (I) + (II)
	\lesssim n u_n^{\gamma} / k_L + \sqrt{\log n/k_L}$. 
	On the other hand, we can show that $|\widehat{p}^-(x) - p_{k_L}|\lesssim  n u_n^{\gamma} / k_L + \sqrt{\log n/k_L}$ in a similar way. Thus from \eqref{equ::widehatpxbound}, we get
	\begin{align*}
		|\widehat{p}_{k_L}^{B}(x)-p_{k_L}(x)|\lesssim n(\log n/n)^{\frac{\alpha\gamma}{4+d}}/k_L+\sqrt{\log n/k_L}.
	\end{align*}
	Since $k_L\leq n$, the assumption $2\alpha\gamma \leq 4+d$ yields $\sqrt{\log n/k_L}\lesssim n(\log n/n)^{\alpha\gamma/(4+d)}/k_L$ and thus the desired assertion.
\end{proof}

The following Lemma, which is needed in the proof of Lemma \ref{lem::petabound}, shows that the instance with PLLS equal to $1$ is a mode of the density function.

\begin{lemma}\label{lem::petamode}
	Let Assumption \ref{ass::cluster} hold and $p_{k_L}(x)$ be defined by \eqref{equ::populationpkl}. If $p_{k_L}(x)=1$ for some $k_L\in \mathbb{N}$, then we have $x\in \mathcal{M}$.
\end{lemma}

\begin{proof}[Proof of Lemma \ref{lem::petamode}]
	Since $p_{k_L}(x)=P(f(y)\leq f(x)|y\in B(x,\overline{R}_{k_L}(x)))=1$, we have 
	\begin{align}\label{equ::fyleqfxqquad}
		f(y)\leq f(x),\qquad y\in B(x,\overline{R}_{k_L}(x))\setminus \mathcal{C}
	\end{align}
	with $\mathcal{C}$ of measure zero. For any $x\in \mathcal{C}$, there exists a sequence $\{x_i\}_{i=1}^n\in B(x,\overline{R}_{k_L}(x))\setminus \mathcal{C}$ such that $\{x_i\}_{i=1}^n\to y$ and 
	\begin{align}\label{equ::fxileqfx}
		f(x_i)\leq f(x), \qquad i\geq 1.
	\end{align}
	By Condition $(i)$ in Assumption \ref{ass::cluster}, $f$ is a continuous function on $B(x,\overline{R}_{k_L}(x))$.  Consequently, \eqref{equ::fxileqfx} yields that $f(y)\leq f(x)$ for $y\in \mathcal{C}$. This together with \eqref{equ::fyleqfxqquad} yields that 
	\begin{align*}
		f(y)\leq f(x), \qquad y\in B(x,\overline{R}_{k_L}(x)).
	\end{align*}
	Therefore, $x$ is a mode of $f$. Hence we complete the proof.
\end{proof}

The following lemma, which is needed in the proof of Theorem \ref{thm::modesingle}, shows that the PLLS can not be too large for the instance far away from the modes.

\begin{lemma}\label{lem::petabound}
	Let assumption \ref{ass::cluster} and \ref{ass::modes} hold. Let $p_{k_L}(x)$ be defined by \eqref{equ::populationpkl}. Then there exists a constant $0< c<1$ such that for all $x\in \mathcal{X}\setminus \mathcal{M}_{ r_{\mathcal{M}}}$, we have $p_{k_L}(x)\leq c$. 
\end{lemma}

\begin{proof}[Proof of Lemma \ref{lem::petabound}]
	Let  $\mathcal{M}_{  r_\mathcal{M}}^{\circ}$ denotes the interior of $\mathcal{M}_{  r_\mathcal{M}}$ and $A:=\mathcal{X}\setminus \mathcal{M}_{  r_\mathcal{M}}^{\circ}$. Then $A$ is a compact set following from the compactness of $\mathcal{X}$. 
	By the condition $(i)$ in Assumption \ref{ass::cluster}, $f$ is a continuous function on $\mathcal{X}$. Thus, $p_{k_L}(x)$ is a continuous function on $\mathcal{X}$.  Therefore, applying extreme value theorem to $p_{k_L}(x)$ on $A$, there exists an $x'\in A$, such that 
	\begin{align}
		c = p_{k_L}(x') = \max_{x \in A} \, p_{k_L}(x).
	\end{align}
	Suppose that $c = 1$, then by Lemma \ref{lem::petamode}, we have $x'\in \mathcal{M}$, which contradicts with $x'\in A$. Therefore, we have $c < 1$ by a contradiction.
	This completes the proof.
\end{proof}

Now, we are in the position of presenting the proof of BDMBC for mode estimation. \\

\begin{proof}[Proof of Theorem \ref{thm::mode}]
	By Lemma \ref{lem::petabound}, there exists a constant $c > 0$ such that $p_{k_L}(x) \leq c$ for all $x \in \mathcal{X} \setminus \mathcal{M}_{r_{\mathcal{M}}}$. 
	The following proof will be made in the case $\lambda>(c+1)/2$.
	Let $c'' := \sqrt{c_2 / (2 c_1)} \wedge 1$ with the constants $c_1$ and $c_2$ specified as in Lemma \ref{lem::Eigen} and $r_n$ specified as in Proposition \ref{lem::kernelmode111}. 
	Lemma \ref{lem::netax} with $\eta(x) := c'' r_n$ and $\tau := 2 \log n$ yields that for all $1 \leq i \leq \#(\mathcal{M})$, there holds
	\begin{align}\label{equ::mode11}
		\biggl| \frac{1}{n} \sum_{i=1}^n \eins \{ X_i \in B(m_i, c''r_n) \} - P(B(m_i, r_n)) \biggr|
		\leq   \sqrt{ r_n^d \log n / n} +  \log n / n + 1 / n 
	\end{align}
	with probability at least $1 - 2/n^2$.
	Since $r_n \asymp (\log n/n)^{1/(d+4)}$, we have
	\begin{align*}
		P(B(m_i, c'' r_n))
		\geq \underline{c} V_d (c'' r_n)^d / 2^d
		\gtrsim \sqrt{r_n^d \log n / n} + \log n / n + 1 / n,
	\end{align*}
	where the first inequality follows from Assumption \ref{ass::cluster} \textit{(i)}. 
	This together with \eqref{equ::mode11} yields that $\sum_{i=1}^n \eins \{ X_i \in B(m_i, c'' r_n) \} > 0$. 
	Consequently, $D \cap B(m_i, c'' r_n)$ is a non-empty set. 
	In other words, there exists an $\widetilde{m}_i \in B(m_i, c'' r_n)$.
	Since $c'' \leq 1$, we have $\widetilde{m}_i \in B(m_i, r_n)$, which implies that $D \cap B(m_i, r_n) \neq \emptyset$. 
	Therefore, we can pick $\widehat{m}_i$ with maximal $f_B$ out of the finite sample $D \cap B(m_i, r_n)$, i.e.,
	\begin{align}\label{equ::defmi11}
		\widehat{m}_i 
		:= \argmax_{X_i \in  B(m_i, r_n)} f_B(X_i).
	\end{align}
	
	Next, we show that $\widehat{p}_{k_L}^{k_D}(\widehat{m}_i) = 1$. 
	Proposition \ref{lem::kernelmode111} implies that with probability $P^n$ at least $1-2/n^2$, there holds $\inf \bigl\{ f_B(x) : x \in B(m_i, c' r_n) \bigr\} 
	> \sup \bigl\{ f_B(x) : x \in B(m_i, r_{\mathcal{M}}) \setminus B(m_i, r_n) \bigr\}$ with $c'$ specified in Proposition \ref{lem::kernelmode111}, which implies that
	$f_B(\widetilde{m}_i)
	\geq \inf \bigl\{ f_B(x) : x \in B(m_i, c' r_n) \bigr\} 
	\geq \sup \bigl\{ f_B(x) : x \in B(m_i, r_{\mathcal{M}}) \setminus B(m_i, r_n) \bigr\} $.
	Consequently, by the definition of $\widehat{m}_i$ in \eqref{equ::defmi11}, we have
	\begin{align}\label{equ::fbmi}
		f_B(\widehat{m}_i)
		\geq f_B(\widetilde{m}_i)
		> \sup \bigl\{ f_B(x) : x \in B(m_i, r_{\mathcal{M}}) \setminus B(m_i, r_n) \bigr\}.
	\end{align}
	This together with \eqref{equ::defmi11} yields
	\begin{align}\label{equ::widehatm}
		\widehat{m}_i
		= \argmax_{X_i \in B(m_i, r_{\mathcal{M}})} f_B(X_i).
	\end{align}
	For any $X_j \in B(\widehat{m}_i, R_{k_L}(x))$, we have
	$\|X_j - m_i\|_2
	\leq \|X_j - \widehat{m}_i\|_2 + \|\widehat{m}_i - m_i\|_2
	\leq R_{k_L}(x) + r_n$,
	where $R_{k_L}(x)$ denotes the $k_L$-distance.
	By Lemma \ref{lem::Rrho}, for all sufficiently large $n$, we have $R_{k_L}(x)\lesssim (k_L/n)^{1/d}\leq r_{\mathcal{M}/2}$. Consequently, we get $\|X_j-{m}_i\|_2 \leq r_{\mathcal{M}}$, This together with \eqref{equ::widehatm} implies that
	$f_B(X_j)\leq f_B(\widehat{m}_i)$, $X_j\in B(\widehat{m}_i,R_{k_L}(x))$.
	Therefore, we get $\widehat{p}_{k_L}^{B}(\widehat{m}_i)=1$. This implies that $\widehat{m}_i\in\widehat{\mathcal{M}}$. 
	Moreover, we have $\|\widehat{m}_i-m_i\|_2 \leq r_n\lesssim (\log n/n)^{1/(4+d)}$.
	
	Note that $\widehat{p}_{k_L}^{B}(\widehat{m}_i)=1$
	implies that $\widehat{m}_i\in \widehat{D}_B(\lambda)$, where $\widehat{D}_B(\lambda)$ is defined by \eqref{equ::detah}. Therefore, we can pick a cluster estimator $\widehat{C}_i$ out of $\mathcal{C}_B(\lambda)$ such that $\widehat{m}_i\in \widehat{C}_i$ for $1\leq i\leq \#(\mathcal{M})$. 
	Next, we will show that for any $1 \leq i < j \leq \#(\mathcal{M})$, there holds $\widehat{C}_i \neq \widehat{C}_j$ by contradiction. Suppose that there exists $1 \leq i < j \leq \#(\mathcal{M})$ such that $\widehat{C}_i = \widehat{C}_j$. Then the distinct mode estimations $\widehat{m}_i$ and $\widehat{m}_j$ with $\widehat{m}_i\in B(m_i,c'r_n)$ and $\widehat{m}_j\in B(m_j,c'r_n)$ are contained in the same connected components of the subgraph $G_B(\lambda)$. 
	Consequently, there exists a sequence $X_1' \ldots, X_{\ell}' \in \widehat{D}_B(\lambda)$ such that $X_i'$ and $X_{i+1}'$ are connected in the subgraph $G_B(\lambda)$, $1 \leq i \leq \ell$, where we set $x'_0 := \widehat{m}_i$ and $x'_{\ell+1} := \widehat{m}_j$.
	This together with Lemma \ref{lem::petah11} yields that 
	\begin{align*}
		\widehat{p}_{k_L}^B(x)\leq  (c+1)/2< \lambda, \quad x\in \mathcal{X} \setminus \mathcal{M}_{r_{\mathcal{M}}}.
	\end{align*}
	for all sufficiently large $n$,
	where the last inequality follows from the choice of $\lambda$.
	Since $X_i'\in \widehat{D}_B(\lambda)$ for $1\leq i\leq \ell$, we have 
	\begin{align}\label{equ::ximathcalm}
		X_i'\in \mathcal{M}_{r_{\mathcal{M}}}, \quad 1\leq i\leq \ell.
	\end{align}
	Let $v:=\sup_{0\leq i\leq \ell+1} \{i:X_i'\in B(m_i,r_{\mathcal{M}})\}$. 
	Since $X_0'=\widehat{m}_i\in B(m_i,r_{\mathcal{M}})$ and $X_{\ell+1}'=\widehat{m}_j\notin B(m_i,r_{\mathcal{M}})$,
	we have  $0\leq v\leq \ell$. 
	From the definition of the supremum and \eqref{equ::ximathcalm}, there exists $i'\neq i$ such that $X_{v+1}'\in B(m_{i'},r_{\mathcal{M}})$.
	This together with $X_v'\in B(m_i,r_{\mathcal{M}})$ yields that
	\begin{align}\label{equ::xvxv1}
		\|X_v' - X_{v+1}'\|
		\geq \|m_i - m_{i'}\| - 2 r_{\mathcal{M}}
		\geq \min_{1 \leq i < j \leq \#(\mathcal{M})} \|m_i - m_j\| - 2 r_{\mathcal{M}}.
	\end{align}
	On the other hand, since $X_v'$ and $X_{v+1}'$ are in the connected components of the subgraph $G_B(\lambda)$, we have $\|X_v' - X_{v+1}'\| \leq R_{k_G}(X_v')\wedge R_{k_G}(X_{v+1}')$, where $R_{k}(x)$ represents the $k$-distance of $x$. 
	By Lemma \ref{lem::Rrho}, for all sufficiently large $n$, we have $R_{k_G}(X_v') \wedge R_{k_G}(X_{v+1}') \lesssim (k_G/n)^{1/d} \lesssim (\log n/n)^{1/d}$. 
	Therefore, we get 
	\begin{align*}
		\|X_v'-X_{v+1}'\|< \min_{1\leq i<j\leq \#(\mathcal{M})}\|m_i-m_j\|-2r_{\mathcal{M}}
	\end{align*} 
	for all sufficiently large $n$, which leads contradiction to \eqref{equ::xvxv1}. 
	Consequently we have $\widehat{C}_i \neq \widehat{C}_j$ for $1 \leq i < j \leq \#(\mathcal{M})$. This completes the proof.
\end{proof}

Next, we present the proof of DMBC for mode estimation. \\

\begin{proof}[Proof of Theorem \ref{thm::modesingle}]
	By the triangle inequality, we have
	\begin{align}\label{equ::knminus}
		\big|k/n-f(x)V_dR_k^d(x)\big|\leq 	\big|k/n-P(B(x,R_k(x)))\big|+\big|P(B(x,R_k(x)))-f(x)V_dR_k^d(x)\big|.
	\end{align}
	By \eqref{equ::rkxoverrkx} in Lemma \ref{lem::Rrho}, we get 
	$\bigl| k/n - P(B(x, R_k(x))) \bigr| \lesssim \sqrt{k \log n} / n$. 
	Similar to the analysis of \eqref{equ::bxrix}, we can show that $\bigl| P(B(x, R_k(x))) - V_d f(x) R_k^d(x) \bigr| \lesssim R^{d+2}_k(x)$ from Lemma \ref{lem::Eigen}. This together with  \eqref{equ::infrksuprk} in Lemma \ref{lem::Rrho} and \eqref{equ::knminus} yields 
	\begin{align*}
		\bigl| k/n - f(x) V_d R_k^d(x) \bigr|
		\leq \sqrt{k \log n/ n} + R_k^{d+2}(x).
	\end{align*}
	Then using \eqref{equ::infrksuprk} in Lemma \ref{lem::Rrho} and choosing $k_{D,n} \asymp n^{\frac{4}{4+d}}(\log n)^{\frac{d}{4+d}}$, we get
	\begin{align*}
		\biggl| \frac{k}{n V_d R^d_k(x)} - f(x) \biggr|
		= \biggl| \frac{k/n - f(x) V_d R_k^d(x)}{V_d R_k^d(x)} \biggr|
		\leq \sqrt{\log n / k} + (k/n)^{2/d}.
	\end{align*}
	Similar analysis to that in the proof of Theorem \ref{thm::modesingle} yields the desired assertion. Thus we omit the proof of Theorem \ref{thm::modesingle} here.
\end{proof}

\subsection{Proofs Related to Section \ref{equ::ratelevelset}}
\label{sec::prooflevelset}

The following Lemma is needed in the proof of Theorem \ref{thm::levelset}.

\begin{lemma}\label{lem::petahholder}
	Let Assumption \ref{ass::cluster} and  \ref{ass::flatness1} hold. 
	Moreover, let $p_{k_L}(x)$ be as in \eqref{equ::populationpkl}. 
	Then for any $x,y\in \mathbb{R}^d$, we have $|p_{k_L}(x)-p_{k_L}(y)|\leq c_\gamma n\|x-y\|^{\alpha\gamma}/k_L$.
\end{lemma}

\begin{proof}[Proof of Lemma \ref{lem::petahholder}]
	For any $x, y \in \mathbb{R}^d$, there holds
	\begin{align} \label{equ::petaholder}
		& \bigl| p_{k_L}(x) - p_{k_L}(y) \bigr|
		\nonumber\\
		& = \biggl| \frac{P(f(z) \leq f(x), z \in B(x,\overline{R}_{k_L}(x)))}{P(B(x, \overline{R}_{k_L}(x)))} - \frac{P(f(z) \leq f(y),z \in  B(y, \overline{R}_{k_L}(y)))}{P(B(y, \overline{R}_{k_L}(y)))} \biggr|
		\nonumber\\
		& = (k_L / n) \bigl| P(f(z) \leq f(x), z \in B(x, \overline{R}_{k_L}(x))) - P(f(z) \leq f(y), z \in B(y, \overline{R}_{k_L}(y))) \bigr|,
	\end{align}
	where we use $P(x,\overline{R}_i(x))=i/n$ when the density function is continuous by Assumption \ref{ass::cluster}.
	By Lemma \ref{lem::pabcd}, we have
	\begin{align*}
		& \bigl| P(f(z) \leq f(x), z \in B(x, \overline{R}_{k_L}(x))) - P(f(z) \leq f(y), z \in B(y, \overline{R}_{k_L}(y))) \bigr|
		\\
		& \leq |P(\{ z : f(z) \leq f(x) \} \triangle \{ z : f(z) \leq f(y) \})| + |P(B(x, \overline{R}_{k_L}(x))) - P(B(y, \overline{R}_{k_L}(y)))|
		\\  
		& = |P(\{ z : f(z) \leq f(x) \} \triangle \{ z : f(z) \leq f(y) \})|,
	\end{align*}
	where we use $P(x,\overline{R}_i(x)) = i/n$.
	By Assumption \ref{ass::cluster} \textit{(ii)} and  \ref{ass::flatness1}, we have
	\begin{align*}
		|P ( \{ z : f(z) \leq f(x) \} \triangle \{ z : f(z) \leq f(y) \} ) |
		\leq c_{\gamma} |f(y) - f(x)|^{\gamma}
		\leq c_{\gamma} \|x - y\|^{\alpha\gamma}.
	\end{align*}
	This together with \eqref{equ::petaholder} yields 
	$|p_{k_L}(x) - p_{k_L}(y)|
	\leq c_{\gamma} n \|x - y\|^{\alpha\gamma} / k_L$,
	completing the proof.
\end{proof}

The next proposition, which provides the difference between the empirical PLLS and the population version w.r.t.~the bagged $k$-distance, supplies the key to the proof of Theorem \ref{thm::levelset}.

\begin{proposition}\label{lem::petah1}
	Let Assumptions \ref{ass::cluster} and \ref{ass::flatness1} hold and suppose that $2 \alpha \gamma \leq 2 \alpha + d$. 
	Choosing 
	\begin{align*}
		k_{D,n} \asymp \log n,
		\quad
		s_n \asymp n^{\frac{d}{2\alpha+d}}(\log n)^{\frac{2\alpha}{2\alpha+d}},
		\quad 
		B_n \geq n^{\frac{1+\alpha}{2\alpha+d}}(\log n)^{\frac{\alpha+d-1}{2\alpha+d}}, 
	\end{align*}
	then with probability $P^n\otimes P_Z^B$ at least $1 - 3/n^2$, for all $x\in \mathcal{X}$, there holds
	\begin{align*}
		|\widehat{p}_{k_L}^B(x) - p_{k_L}(x)|
		\lesssim n(\log n/n)^{\frac{\alpha\gamma}{2\alpha+d}}/k_L.
	\end{align*}
\end{proposition}

\begin{proof}[Proof of Proposition \ref{lem::petah1}]
	The proof is similar to that of Proposition \ref{lem::petah11} and hence we omit it here.
\end{proof}

Next, we present the proof of the level set estimation of BDMBC. \\

\begin{proof}[Proof of Theorem \ref{thm::levelset}]
	The desired assertion involves two directions to show from the Hausdorff metric: 
	\begin{align*}
		(I) := \max \bigl\{ d(x, L_{k_L}(\lambda)) :
		x \in \widehat{L}_{k_L}(\lambda) \bigr\},
		\qquad 
		(II) := \sup \bigl\{ d(x,\widehat{L}_{k_L}(\lambda)) : x \in L_{k_L}(\lambda) \bigr\} .
	\end{align*}
	Proposition \ref{lem::petah1} yields that with probability $P^n$ at least $1-3/n^2$, for all $x\in \mathcal{X}$, there holds
	\begin{align}\label{equ::petahapply}
		|\widehat{p}_{k_L}^B(x) - p_{k_L}(x)|
		\lesssim n (\log n/n)^{\frac{\alpha\gamma}{2\alpha+d}} / k_L
		:= \delta_n.
	\end{align}
	
	The following arguments will be made on the event that \eqref{equ::petahapply} holds.
	
	For any $x\in \widehat{L}_{k_L}(\lambda)$, we have $\widehat{p}_{k_L}^B(x)\geq \lambda$. This together with \eqref{equ::petahapply} yields 
	\begin{align}\label{equ::lambdac5}
		p_{k_L}(x)
		\geq \lambda - \delta_n.
	\end{align}
	If $p_{k_L}(x) \geq \lambda$, i.e., $x \in L_{k_L}(\lambda)$, then we have $d(x, L_{k_L}(\lambda)) = 0$.
	Otherwise if $p_{k_L}(x) < \lambda$, then \eqref{equ::lambdac5} yields $\lambda - \delta_n \leq  p_{k_L}(x) < \lambda$.
	By Assumption \ref{ass::flatness}, we then have $(I) \leq 
	d(x, L_{k_L}(\lambda)) \leq (\delta_n/c_\beta)^{1/\beta}$.
	
	Next, let us consider $(II)$. We first show that for any $x \in L_{k_L}(\lambda)$, there exists some $y \in L_{k_L}(\lambda+2\delta_n)$ such that $\|x - y\| \leq (2\delta_n)^{1/\beta}$.
	Indeed, if $x \in L_{k_L}(\lambda+2\delta_n)$, then we can choose $y = x$ and  we have $\|y - x\|=0$. Otherwise if $x \notin L_{k_L}(\lambda+2\delta_n)$, then we have $\lambda \leq p_{k_L}(x) \leq \lambda+2\delta_n$.
	By Assumption \ref{ass::flatness}, we have $d(x,L_{k_L}(\lambda+2\delta_n)) \leq (2\delta_n)^{1/\beta}$. 
	Therefore, we can choose $y \in L_{k_L}(\lambda+2\delta_n)$ such that $\|x - y\| \leq (2\delta_n)^{1/\beta}$.
	
	Lemma \ref{lem::netax} with $r_n = \bigl( (\delta_n k_L) / (c_{\gamma} n) \bigr)^{1/\alpha\gamma}$ and $\tau := 2 \log n$ implies that for all $y \in \mathbb{R}^d$, there holds
	\begin{align}\label{equ::mode1}
		\biggl| \frac{1}{n} \sum_{i=1}^n \eins \{ X_i \in B(y, r_n) \} - P(B(y, r_n)) \biggr|
		\lesssim \sqrt{r_n^d \log n / n} +  \log n / n
	\end{align}
	with probability at least $1 - 1/n^2$. 
	Assumption \ref{ass::cluster} \textit{(ii)} together with
	the definition of $r_n$ and the condition $\gamma > d/(2\alpha+d)$ yields that for all sufficiently large $n$, we have $r_n \leq 1$ and
	\begin{align*}
		P(B(y, h))
		\geq \underline{c} V_d r_n^d
		\geq  \sqrt{r_n^d \log n / n} +  \log n / n.
	\end{align*}
	This together with \eqref{equ::mode1} yields 
	$\sum_{i=1}^n \eins \{ X_i \in B(y,r_n) \} > 0$.
	Therefore, we can pick an $X_i \in D$ such that $\|y - X_i\| \leq r_n$.
	By Lemma \ref{lem::petahholder}, we have
	$p_{k_L}(X_i)
	\geq p_{k_L}(y) - c_{\gamma} n \|X_i - y\|^{\alpha\gamma} / k_L
	\geq \lambda + \delta_n$.
	This together with \eqref{equ::petahapply} yields that $\widehat{p}_{k_L}^B(X_i)\geq \lambda$, which implies that $X_i\in \widehat{L}_{k_L}(\lambda)$. By the triangular inequality, we have
	$\|x - X_i\|
	\leq \|x - y\| + \|y - X_i\|
	\leq (2 \delta_n)^{1/\beta} + r_n$,
	which yields 
	\begin{align*}
		(II) := \sup \bigl\{ d(x, \widehat{L}_{k_L}(\lambda)) : x \in L_{k_L}(\lambda) \bigr\} 
		\leq (2\delta_n)^{1/\beta} + r_n.
	\end{align*}
	Therefore, we have 
	\begin{align*}
		d_{\mathrm{Haus}}(\widehat{L}_{k_L}(\lambda), L_{k_L}(\lambda))
		\leq (I) \vee (II)
		\lesssim (\log n/n)^{\frac{1}{2\alpha+d}} +(n/k_L)^{1/\beta} ( \log n/n)^{\frac{\alpha\gamma}{(2\alpha+d)\beta}}.
	\end{align*}
	By the condition $\alpha\gamma \geq \beta$ and the selection of $k_L \gtrsim n^{1+(\beta-\alpha \gamma)/(2\alpha+d)}(\log n)^{(\alpha\gamma-\beta)/(2\alpha+d)}$, there holds $d_{\mathrm{Haus}}(\widehat{L}_{k_L}(\lambda), L_{k_L}(\lambda)) \lesssim (\log n/n)^{1/(2\alpha+d)}$.
	Thus, we obtain the desired assertion.
\end{proof}

\section{Conclusions}\label{sec::conclusions}

In this paper, we propose an ensemble algorithm called \textit{bagged $k$-distance for mode-based clustering} (\textit{BDMBC}) by putting forward a new measurement called the \textit{probability of localized level sets} (\textit{PLLS}), which transforms the multi-level density clustering to the single-level setting. 
To deal with the curse of dimensionality in high-dimensional density estimation, we employ the $k$-distance that can be directly calculated from the data. 
We further introduce the bagging technique to improve computational efficiency in large-scale situations.
To establish solid theoretical guarantees of the proposed algorithm, we first derive optimal convergence rates for mode estimation with properly chosen parameters.
It turns out that with a relatively small $B$, the sub-sample size $s$ can be much smaller than the number of training data $n$ at each bagging round, and the number of nearest neighbors $k_D$ can be reduced simultaneously.
Moreover, by establishing optimal convergence results for the level set estimation of the PLLS in terms of Hausdorff distance, we show that BDMBC can find localized level sets for varying densities and thus enjoys local adaptivity.
Finally, we also conducted persuasive experiments on both synthetic and real-world datasets, showing the promising experimental performances of our BDMBC, demonstrating how bagging narrows the searching grid of parameters, and offering advice on how to choose parameters in applications. 

It's worth pointing out that compared to other clustering algorithms, our algorithm BDMBC enjoys various advantages.
On the one hand, BDMBC is more computationally efficient than hierarchical density-based clustering algorithms. On the other hand, compared with other mode-based clustering algorithms, BDMBC has stronger resistance to the curse of dimensionality and an easy procedure in the parameter-searching procedure.

\bibliographystyle{plain}
\small{\bibliography{BDMBC}}
\end{document}